\documentclass[conference]{IEEEtran}
\IEEEoverridecommandlockouts
\usepackage{cite}
\usepackage{amsmath,amssymb,amsfonts}
\usepackage{algorithmic}
\usepackage{graphicx}
\usepackage{textcomp}
\usepackage{xcolor}
\def\BibTeX{{\rm B\kern-.05em{\sc i\kern-.025em b}\kern-.08em
    T\kern-.1667em\lower.7ex\hbox{E}\kern-.125emX}}

\usepackage{color, soul}
\usepackage{multicol}
\usepackage{amsmath,mathtools,amssymb,amsfonts,amsthm, xparse} 
\usepackage{breqn}
\usepackage{cancel}
\usepackage{bbm}
\usepackage{tikz}
\usetikzlibrary{shapes.geometric}
\usepackage{enumitem}
\usepackage[ruled,linesnumbered]{algorithm2e}
\usepackage{soul}
\usepackage{caption}
\usepackage{subcaption}
\usepackage{mathrsfs}

\usepackage{scalerel,nicefrac}

\usepackage{array,multirow,graphicx}
\usepackage{float}

\makeatletter\usepackage{apptools}
\AtAppendix{\counterwithin{lemma}{section}}
\newcommand*{\rom}[1]{\expandafter\@slowromancap\romannumeral #1@}
\makeatother    
    
\newcommand*{\rv}[1]{\MakeUppercase{#1}}

\let\quant\relax
\DeclareMathOperator{\quant}{\mathcal{Q}}
\newcommand{\round}{n}

\DeclarePairedDelimiter\floor{\lfloor}{\rfloor}

\DeclareMathOperator{\EXP}{\mathbb{E}} 
\DeclareMathOperator{\R}{\mathbb{R}}

\let\Z\relax
\DeclareMathOperator{\Z}{\boldsymbol{Z}}
\let\x\relax
\DeclareMathOperator{\x}{\boldsymbol{x}}
\let\y\relax
\DeclareMathOperator{\y}{\boldsymbol{y}}
\let\g\relax
\DeclareMathOperator{\g}{\boldsymbol{g}}

\let\b\relax
\DeclareMathOperator{\b}{\boldsymbol{b}}
\let\c\relax
\DeclareMathOperator{\c}{c} 
\let\q\relax
\DeclareMathOperator{\q}{\boldsymbol{q}}
\let\r\relax
\DeclareMathOperator{\r}{\boldsymbol{r}}
\let\w\relax
\DeclareMathOperator{\w}{\boldsymbol{w}}
\let\z\relax
\DeclareMathOperator{\z}{\boldsymbol{z}}
\let\h\relax
\DeclareMathOperator{\h}{\boldsymbol{h}}

\let\pifunc\relax
\DeclareMathOperator{\pifunc}{\boldsymbol{\pi}}
\let\Pifunc\relax
\DeclareMathOperator{\Pifunc}{\boldsymbol{\Pi}}
\let\pivec\relax
\DeclareMathOperator{\pivec}{\underline{\boldsymbol{\pi}}} 
\let\Pivec\relax
\DeclareMathOperator{\Pivec}{\underline{\boldsymbol{\Pi}}} 
\let\conv\relax
\DeclareMathOperator{\conv}{\textnormal{conv}} 

\let\X\relax
\DeclareMathOperator{\X}{\boldsymbol{X}}

\let\C\relax
\DeclareMathOperator{\C}{C} 
\let\Q\relax
\DeclareMathOperator{\Q}{\boldsymbol{Q}}

\let\F\relax
\DeclareMathOperator{\F}{\mathcal{F}}

\newcommand{\norm}[1]{\left\lVert#1\right\rVert}

\newtheorem{prop}{Proposition}
\newtheorem{lemma}{Lemma}
\newtheorem{claim}{Claim}
\newtheorem{theorem}{Theorem}

\newtheorem{remark}{Remark}
\newtheorem{definition}{Definition}
\newtheorem*{notation}{Notation}
\newtheorem{assumption}{Assumption}

\begin{document}

\title{Network Adaptive Federated Learning: \\
Congestion and Lossy Compression
\thanks{
This material is based upon work of Hegde and de Veciana supported by the National Science Foundation (NSF) under grant No. 2148224 and is supported in part by funds from OUSD R\&E, NIST,  and industry partners as specified in the Resilient \& Intelligent NextG Systems (RINGS) program and the WNCG/6G@UT industrial affiliates. The work of Mokhtari is supported in part by the NSF AI Institute for Future Edge Networks and Distributed Intelligence (AI-EDGE) via NSF grant 2112471, the Machine Learning Lab (MLL) at UT Austin, and the Wireless Networking and Communications Group (WNCG) Industrial Affiliates Program.
}
}

\author{\IEEEauthorblockN{Parikshit Hegde}
\IEEEauthorblockA{\textit{Electrical and Computer Engineering} \\
\textit{The University of Texas at Austin}\\
Austin, Texas, USA \\
hegde@utexas.edu}
\and
\IEEEauthorblockN{Gustavo de Veciana}
\IEEEauthorblockA{\textit{Electrical and Computer Engineering} \\
\textit{The University of Texas at Austin}\\
Austin, Texas, USA \\
gustavo@ece.utexas.edu}
\and
\IEEEauthorblockN{Aryan Mokhtari}
\IEEEauthorblockA{\textit{Electrical and Computer Engineering} \\
\textit{The University of Texas at Austin}\\
Austin, Texas, USA \\
mokhtari@austin.utexas.edu}
}

\maketitle
\thispagestyle{plain}
\pagestyle{plain}

\begin{abstract}
In order to achieve the dual goals of privacy and  learning across distributed data, Federated Learning (FL) systems rely on frequent exchanges of large files (model updates) between a set of clients and the server. As such FL systems are exposed to, or indeed the cause of, congestion across a wide set of network resources. 
Lossy compression can be used to reduce the size of exchanged files and associated delays, at the cost of adding noise to model updates. By judiciously adapting clients' compression to varying network congestion, an FL application can reduce wall clock training time. To that end, we propose a Network Adaptive Compression (NAC-FL) policy, which dynamically varies the client's lossy compression choices to network congestion variations. We prove, under appropriate assumptions, that NAC-FL is asymptotically optimal in terms of directly minimizing the expected wall clock training time. Further, we show via simulation that NAC-FL achieves robust performance improvements with higher gains in settings with positively correlated delays across time. 
\end{abstract}


\section{Introduction}
\label{sec:introduction}
Communication costs and delays of sending model 
updates from clients to the server are a known bottleneck in training Federated Learning (FL) systems \cite{mcmahan2017communication, li2020federated, kairouz2021advances, konevcny2016federated}. Two common techniques used to alleviate this issue are: 1) \emph{local computations} where clients perform several local steps before communicating with the server, and 2) \emph{(lossy) compression} where clients communicate quantized/compressed updates to the server. The eventual end goal of these approaches is to minimize the \emph{wall clock time} for convergence of the training algorithm (hereon referred to as FL algorithm) by reducing the amount of data communicated from clients to the server.

To this end, several works have analyzed the relationship between compression, local computations and the number of rounds needed by FL algorithms to converge \cite{alistarh2017qsgd, basu2019qsparse, stich2018local, stich2018sparsified, ghadimi2013stochastic, reisizadeh2020fedpaq, haddadpour2020federated, sattler2019robust}. However, these works ignore the impact of changing network congestion, \emph{both across clients and across time}, on the wall clock time to converge. For instance, a client may choose a high degree of compression when it sees high network congestion, while a client seeing lower congestion may  \emph{opportunistically} choose not to  compress as much.  In this work, we ask the following question: ``Can we design a policy that adapts the amount of compression across clients and time according to changing network conditions in order to optimize the wall clock time?'' To answer this question, we first characterize the impact that changing network congestion and an adaptive compression policy have on the wall clock time. Second, we propose the Network Adaptive Compression for Federated Learning (NAC-FL) policy that judiciously chooses compression levels based on network congestion to minimize the wall clock time. Crucially, NAC-FL does not rely on the prior knowledge of the distribution of network congestion. Instead, it learns to optimize its compression decisions on-the-fly based on the congestion seen by clients.

NAC-FL works in an opportunistic manner by adaptively choosing high or low amounts of compression across clients and across time based on low or high network congestion.  It further considers two effects that compression has on the wall clock time. First, with increasing amount of compression, the FL algorithm would require more communication rounds to converge, as the server receives ``noisier'', and hence inaccurate, model updates. Second, with higher degrees of compression, the duration of each round would decrease as a smaller model update is communicated. Since the wall clock time is affected by both the number of rounds and the duration of each round (it is effectively the product of the two quantities), a policy for choosing compression levels should consider these jointly. Fig.~\ref{fig:motivation} provides an illustrative visualization. Hence, NAC-FL aims to find the ``sweet-spot'' compression levels over time varying network congestion.




\begin{figure}[t!]
	\centering
	\includegraphics[width=.45\textwidth]{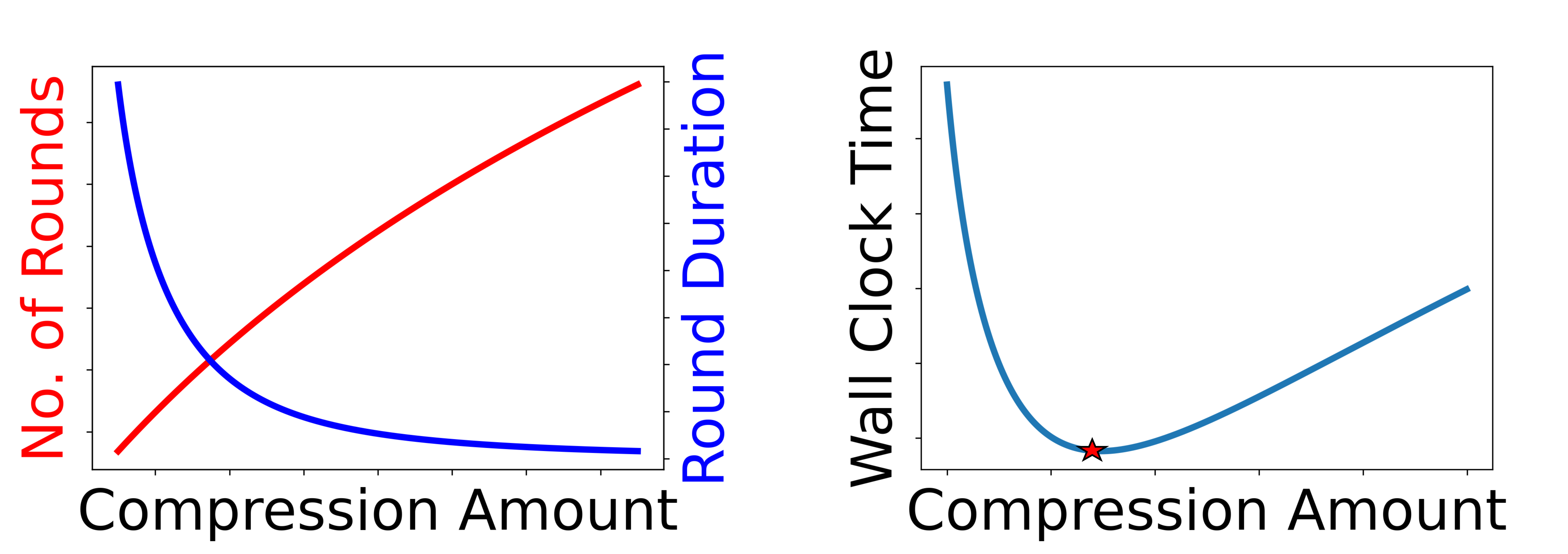}
	\caption{Illustration of how compression level affects round duration, number of rounds and wall clock time.}
	\label{fig:motivation}
	\vspace{-2mm}
\end{figure}


\vspace{2mm}
\noindent \textbf{Contributions.} 
 We propose a general framework to study how to best adapt compression of client model updates. Assuming a stationary Markov model for the underlying network congestion state, we show that optimal policies are state dependent and characterize the expected stopping time for  convergence to a predefined model accuracy.

This characterization provides the underlying insight for our proposed NAC-FL policy. To our knowledge this is the first policy for compression that adapts to the stochastic variations of the underlying network congestion process. 
Under appropriate assumptions on the FL algorithm and underlying network congestion and delays, we provide a proof of the asymptotic optimality of NAC-FL in terms of minimizing the mean time until the convergence criterion is met. To our knowledge this is the first  theoretical result of this type.

 Finally we demonstrate via simulation the performance gains and robustness of NAC-FL vs alternative fixed compression and/or fixed error per round policies. We explore a variety of models for network congestion, finding that in particular NAC-FL excels in the practically relevant setting where the network sees  positive correlations in the network congestion accross time.



\subsection{Related Work}
\label{ssec:related-work}
%
%
Perhaps the most related papers to our work are \cite{bouzinis2022wireless, zhang2020federated, sun2020adaptive,jhunjhunwala2021adaptive, honig2022dadaquant} which explored adaptive compression schemes for FL settings. In \cite{bouzinis2022wireless, zhang2020federated, sun2020adaptive} the authors propose adapting compression to network congestion. In these works, the algorithm to select compression has a per round budget, e.g., a budget on delay (or compression error) per round, and possibly heterogeneous compression levels are chosen across the clients based on the current network congestion to minimize the compression error (or delay) for the round. These works exploit the diversity of network congestion across the clients, but \emph{not across time}. Meanwhile \cite{jhunjhunwala2021adaptive, honig2022dadaquant} have observed that using a higher amount of compression at the start and gradually reducing compression through time may improve the wall clock time. Our proposed policy is novel in that it learns how to best exploit congestion variation across clients and across time to optimize the wall clock time.

Another line of work that aims to reduce the overall communication cost is client sampling \cite{chen2020optimal, chen2020joint, ribero2020communication, perazzone2022communication}, where at each round, only a subset of the clients are chosen to participate. The authors of \cite{perazzone2022communication} propose a client sampling and power control policy that adapts to time varying channels of clients sharing a single base station and optimizes a proxy for wall clock time. Overall we veiw  lossy compression and client sampling as alternative approaches geared at addressing communication bottlenecks. A study of how to  jointly adapt lossy compression and client sampling to changing network congestion is left for future work.

\subsection{Paper Organization}

In Section \ref{sec:model_setup}, we introduce our system model. In Section \ref{sec:adapt_quant}, we propose our NAC-FL algorithm for lossy compression and under appropriate assumptions prove it is asymptotically optimal. Section \ref{sec:simulations} is devoted to exploring the method for several problem instances and in particular for various models for the underlying network congestion in terms of correlation across clients and time. In Section \ref{sec:flac_practice}, we comment on the practical aspects of estimating the file transfer delay of clients when deploying NAC-FL. Finally,  in Section \ref{sec:conclusion}, we close the paper with some concluding remarks.

\begin{notation}
Throughout this document, unless otherwise mentioned, quantities denoted with lowercase letters  correspond to constants, and uppercase letters correspond to random variables. Bold symbols correspond to vectors, and regular symbols indicate scalars. For example, $\x$ is a constant vector, $\X$ is a random vector, $x$ is a constant scalar, and $X$ is a random scalar/variable. Lowercase and uppercase forms of the same letter correspond to constant and random variable notions of the same quantity. A sequence indexed by $\round$ will be denoted as $(x^{\round})_{\round}$.
\end{notation}

\section{Model Setup}
\label{sec:model_setup}

In this paper, we focus on a federated architecture, where a server aims to find a model that performs well with respect to the data of a group of $m$ clients, and in which nodes exchange updates based on their local information with only  the server. 
More precisely, suppose the loss function associated with client $j$ is denoted by $f_j(\w)$, where $\w$ represents the weights of the model, e.g., the weights of a neural network. The goal is to find the model that minimizes the average loss across clients
\begin{equation*}
    f(\w) = \frac{1}{m} \sum_{j=1}^m f_j(\w).
\end{equation*}

The FL algorithm proceeds in rounds. Each round consists of two stages: (i) a local stage in which each client updates the most recent model received from the server  via gradient-based updates based on its local data and (ii), an  aggregation stage in which the server updates the global model by aggregating the local updates received from clients. 
We shall let  $\w^{\round}$ denote the global model at the server at round $\round$. Further, we let  $\tau^{\round}$ denote the total number of local steps (such as gradient steps)  that each client performs at round $\round$, and let $\w_j^{\tau^{\round},\round}$ denote the resulting local model at node $j$. 

In this paper, we are interested in the setting where 
each client sends a compressed version $\tilde{\g}_{Qj}^{\round}$
of its local model $\w_j^{\tau^{\round}, \round}$ 
 to the server using a lossy compression algorithm (or, \emph{compressor}) $\quant(\cdot, \cdot)$. The compressor accepts a vector $\x$ and a parameter $q \in [0, q_{\max}]$ indicating the amount of compression with the maximum value being $q_{\max}$, and outputs $\hat{\X} = \quant(\x, q)$ which is an approximation of $\x$, but has a decreased file size as compared to $\x$. $\hat{\X}$ is capitalized to highlight that the compressor $\quant(\cdot, \cdot)$ may use randomness in its compression.  We shall denote by $q_j^{\round}$ the compression parameter used by client $j$ for round $\round$, and denote by $\q^{\round} \triangleq (q_j^{\round})_{j=1}^m$ the vector of parameters used by the clients in  round $\round$. After receiving updates from all the clients, the server aggregates the compressed local models and produces the next global model $\w^{\round+1}$.

Given a target tolerance $\varepsilon>0$, the goal of FL is to generate a sequence of global models until on some round $r_\varepsilon$ 
a prespecified {\em stopping criterion} is first met, e.g., the norm of the global loss function gradient is at most $\epsilon $, i.e.,  $\norm{\nabla f(\w^{r_\varepsilon})}  \leq \varepsilon$.
Our goal is to find an adaptive compression policy that dynamically adapts to the possibly time varying 
network states such that the target accuracy is achieved with a minimum overall wall clock time. 


We  formalize the \emph{overall wall clock time}, denoted $t_{\varepsilon}$, required  to achieve the target accuracy as follows. 
The duration $d( \tau^{\round}, \q^{\round}, \rv{\c}^{\round} )$ of a round $\round$ depends on:
\begin{itemize}
\item $\tau^{\round}$, the number of local computations performed by clients which we will assume to be the same across clients;
\item $\q^{\round}$, an $m$ dimensional vector of clients' compression parameters ;
\item $\c^{\round}$, the \emph{network state} which models network congestion and is assumed to be an element of a finite set ${\cal C}.$  
\end{itemize}
This allows some flexibility, e.g., the round's duration may depend on the max delay to deliver the model update from clients to server, or the sum of the delays if clients share a single resource in TDMA (Time Division Multiple Access) fashion. 
The total wall clock time is then given by 
\begin{equation}
t_{\varepsilon} = \sum_{\round=1}^{r_\varepsilon} d\left( \tau^{\round}, \q^{\round}, \c^{\round} \right).
\label{eq:wall_clock_time}
\end{equation}


In our system model, the sequence of network states, $( \c^{\round} )_\round$, is assumed to be exogenous, i.e.,  not be controlled by the server or the clients nor their choices 
of  $\tau^{\round}$ and $\q^{\round}.$
The delays associated with the server multicasting global models to clients are assumed to be exogeneous i.e., can not be controlled by the FL server/clients and are not compressed, whence are not part of the model.
Still, in this work, based on observing  the network state we will devise an approach to select the clients compression parameters so as to minimize the wall clock time. As discussed in Section \ref{sec:flac_practice},
in practice observation of the network state may involve light weight in band estimation by probing delays of message bits as they are delivered in a given round.

A policy for choosing compression parameters is called a \emph{state dependent stationary policy} if it can be expressed as a function $\pifunc$ of the current network state, i.e., $\q^{\round} = \pifunc(\c^{\round})$ for all rounds $\round \in \mathbb{N}$. Such a policy will be referred to simply as policy $\pifunc$. Given a random sequence of network states, $(\C^{\round})_{\round}$, let $R_{\varepsilon}^{\pifunc}$ be the random variable denoting the minimum number of rounds needed to converge to error tolerance $\varepsilon$ under policy $\pifunc$. Then, the corresponding wall clock time, denoted by $T_{\varepsilon}^{\pifunc}$, is expressed as, 
\[
T_{\varepsilon}^{\pifunc} = \sum_{\round=1}^{R_{\varepsilon}^{\pifunc}} d\left(\tau^{\round}, \pifunc\left(\C^{\round}\right), \C^{\round}\right).
\]



\section{Network Adaptive Compression for Federated Learning (NAC-FL)}
\label{sec:adapt_quant}

Our approach to designing a policy to adapt 
clients' compression parameters centers on recognizing that
 the expected wall clock time can be broken up into a product 
 of the \emph{expected number of rounds}  $r_\varepsilon$ needed to converge to an error tolerance $\varepsilon$ and the \emph{average duration of each round} $\hat{d}$.  We start by characterizing the relationship between $r_\varepsilon$, $\hat{d}$, and the sequence of selected quantization parameters $(\q^{\round})_{\round}$ and network states $(\c^{\round})_{\round}$ for a given FL algorithm.

Below we state an assumption relating $r_\varepsilon$ to $(\q^{\round})_{\round}.$ 
To that end we introduce a strictly increasing, continuous and bounded  scalar function 
$h_{\varepsilon}: [0, q_{\max}] \rightarrow \mathbb{R}^+$ of compression parameter $q$ and an associated vector function
$\boldsymbol{h}_{\varepsilon}: [0, q_{\max}]^{\times m} \rightarrow \mathbb{R}_+^m$
of a compression vector $\boldsymbol{q}$ where 
$\boldsymbol{h}_{\varepsilon,j}(\q) = h_{\varepsilon}(q_j).$ 
We let  $\boldsymbol{h}_{\varepsilon}^{-1}$ denote the inverse of this vector function. 
\begin{assumption}
For a given FL algorithm there exists a strictly increasing, continuous and bounded function $h_{\varepsilon}(q)$ and norm $\norm{\cdot}$ such that
given a sequence of compression parameters $\left(\q^{\round}\right)_{\round}$, the FL algorithm has reached the desired error tolerance $\varepsilon$ by round $r$ if and only if,
$$
r > \frac{1}{r} \sum_{\round=1}^{r} \norm{ \boldsymbol{h}_\varepsilon \left( \q^{\round} \right)}
$$
for some norm. 
\label{ass:q_suff_stat}
\end{assumption}


The above assumption implies that the expected number of rounds can be written as the average of an increasing function of the sequence of selected quantization parameters. Roughly speaking, given a lossy compression policy that generates a stationary parameter sequence  $( \Q^{\round})_{\round}$ whose marginal distribution is the same as the random vector $\Q$, the above criterion means that the expected number of rounds to converge to the desired error tolerance is approximately $\EXP[
 \norm{ \boldsymbol{h}_\varepsilon \left( \Q \right)}].$

This is a general condition that is motivated by convergence bounds of several FL algorithms with compression, including, \cite{alistarh2017qsgd, stich2018sparsified, haddadpour2020federated}. In particular in Appendix \ref{sec:flaq}, we illustrate this motivation for an extension of the FedCOM algorithm  \cite{haddadpour2020federated}, when $q$ indicates the normalized-variance introduced by the compressor, the scalar function is $h_\varepsilon(q) = O(\sqrt{q+1}/\varepsilon)$ and the norm is the $L_2$ norm.  


\begin{assumption}
For any sequence of compression parameters $\left(\q^{\round}\right)_{\round}$ the minimum number of rounds $r_\varepsilon$ needed to converge to an error tolerance $\varepsilon$ is such that $r_\varepsilon = \Theta(1/poly(\varepsilon))$, where $poly(\varepsilon)$ denotes a polynomial of $\varepsilon$.
\label{ass:asymptotic_convergence}
\end{assumption}
Assumption \ref{ass:asymptotic_convergence} is a natural assumption for gradient based optimization algorithms. It requires the convergence guarantees for the FL algorithm to be such that when we require a more accurate solution, the number of required communication rounds grows.  This argument indeed holds even for the settings that we do not exchange compressed signals.

We also make the following additional assumption about the round duration function. 

\begin{assumption}
  Given a network state $\c$, number of local computations $\tau$, and compression parameters $\q = \boldsymbol{h}_{\varepsilon}^{-1}(\r)$, the round duration $d\left(\tau, \q, \c \right) = d\left(\tau, \boldsymbol{h}_{\varepsilon}^{-1}(\r), \c \right)$ is bounded, convex in $\r$ and decreasing in every coordinate of $\r$. 
  \label{ass:comm_delay}
\end{assumption}
In Assumption \ref{ass:comm_delay}, the round duration being decreasing in $\r$ is reasonable, since we expect more rounds as well as smaller file sizes with higher compression. The convexity is motivated by the notion that we use a ``good compressor'' as illustrated next. Consulting Fig. \ref{fig:delay_convexity}, for any two parameters $q_1,q_2$ and $0<\alpha<1$, a new time-sharing compressor $\mathcal{Q}'$ may be derived which outputs $\mathcal{Q}(\x,q_1)$ with probability $\alpha$ and outputs $\mathcal{Q}(\x, q_2)$ with probability  $(1-\alpha)$. This compressor has expected round duration $\alpha d(\tau,q_1,\c) + (1-\alpha)d(\tau, q_2, \c)$. And, in certain cases, its compression parameter is $q_\alpha = \alpha q_1 + (1-\alpha)q_2$ (such as when the stochastic quantizer parameterized by its normalized variance \cite{alistarh2017qsgd} is used). If $\mathcal{Q}$ is a ``good compressor'', then its round duration, $d(\tau, q_{\alpha}, \c)$, should be lower compared to that of the simple time-shared compressor, $\alpha d(\tau, q_1, \c)+(1-\alpha)d(\tau, q_2, \c)$. Therefore, the convexity of the round duration function is a reasonable assumption for ``good compressors''  (considering $h_{\varepsilon}(q) \propto q$ for simplicity). 
%
%
\begin{figure}[t!]
	\centering
	\includegraphics[width=.35\textwidth]{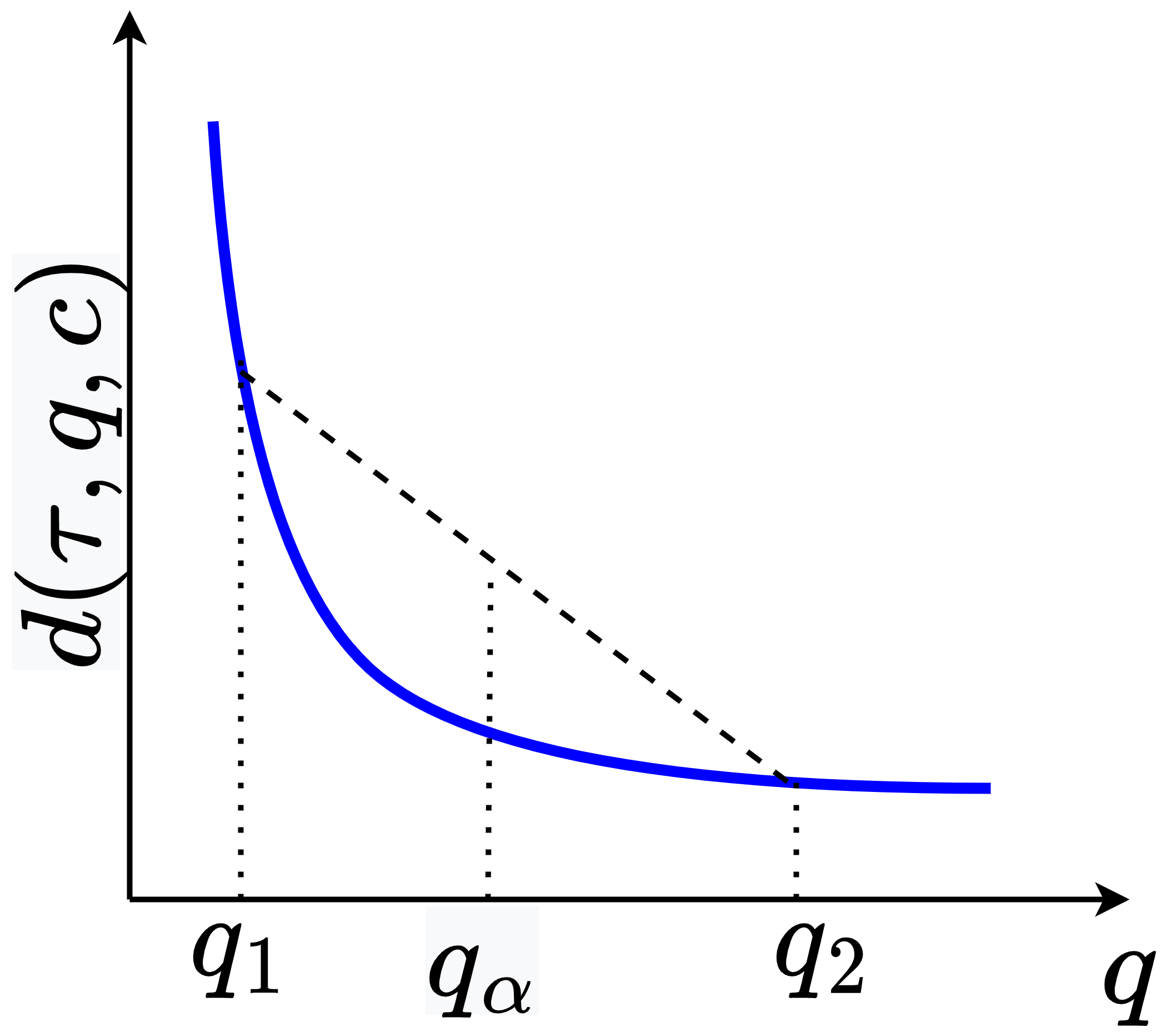}
	\caption{Illustration of a round duration as a function of compression parameter $q$ for a fixed local computation $\tau$ and network state $\c$. }
	\label{fig:delay_convexity}
	\vspace{-2mm}
\end{figure}

\begin{assumption}
The sequence of network states $(\C^{\round})_{\round}$ forms an irreducible aperiodic stationary Markov Chain on a finite state space $\mathcal{C}$ with invariant distribution $\mu$.
\label{ass:markov}
\end{assumption}
Assumption \ref{ass:markov} is a natural assumption made to facilitate the analysis of algorithms (see e.g., \cite{stolyar2005asymptotic}). 

\subsection{Expected Wall Clock Time Formulation}
\label{ssec:wall-clock}
Given the above mentioned assumptions, we are now ready to introduce the proposed framework. We begin by showing that we need only consider state dependent stationary policies for choosing compression parameters when optimizing the overall wall clock time.  

\begin{lemma}
Under Assumptions \ref{ass:q_suff_stat}-\ref{ass:markov}
there exists a state dependent stationary policy to select compression parameters which is asymptotically optimal in terms of minimizing the wall clock time to reach a desired error tolerance of $\varepsilon$ as $\varepsilon \to 0$.
\label{lm:stationary_q}
\end{lemma}
The proof of Lemma \ref{lm:stationary_q} depends on two critical observations. First, since by Assumption \ref{ass:asymptotic_convergence} the number of rounds needed to converge grows large as $\varepsilon \to 0$, one can expect the empirical distribution of the network states modelled by the finite state Markov Chain to concentrate around the invariant prior to the stopping time.  
Second, due to the convexity of the round duration function in Assumption \ref{ass:comm_delay},  given a sequence of network states there exists a state dependent stationary policy that is near optimal and depends solely on the  empirical distribution of the sequence. The proof is in Appendix \ref{app:proof_stationary_q}.

Here, we will focus on the setting where $\varepsilon$ is small, hence by Lemma \ref{lm:stationary_q}, we only need to consider state dependent stationary policies, $\q^{\round}=\pifunc(\c^{\round})$. 

\begin{lemma}
Under Assumptions \ref{ass:q_suff_stat}-\ref{ass:markov} and a fixed number of local computations  per round $\tau$, for every $\delta>0$, there exists an $\varepsilon_{th} > 0$ such that, for all $\varepsilon < \varepsilon_{th}$ and any state-dependent stationary policy $\pifunc$, the expected wall clock time is bounded as,
\begin{equation}
  1-\delta \leq \frac{\EXP\left[T_{\varepsilon}^{\pifunc} \right]}{
    \EXP[\norm{ \boldsymbol{h}_\varepsilon \left( \pifunc(\C) \right)}]
      \EXP[d\left(\tau, \pifunc(\C), \C \right)]} \leq 1+\delta,
    \label{eq:exp_wall_clock_time}
\end{equation}
where, $\C$ denotes a random variable whose distributions is $\mu$ (see Assumption \ref{ass:comm_delay}).
\label{lm:exp_time}
\end{lemma}
Lemma \ref{lm:exp_time} is proved in Appendix \ref{app:lemma_2_proof}. Define,
\begin{equation}
    \hat{t}_{\varepsilon}^{\pifunc} \triangleq 
        \EXP[\norm{ \boldsymbol{h}_\varepsilon \left( \pifunc(\C) \right)}]
        \EXP[d\left(\tau, \pifunc(\C), \C \right)].
        \label{eq:wct_proxy}
\end{equation}
Due to Lemma \ref{lm:exp_time}, for small enough $\varepsilon$, $\hat{t}_{\varepsilon}^{\pifunc}$ provides an accurate approximation for $\EXP[T_{\varepsilon}^{\pifunc}]$. Therefore, from here onwards, we shall assume implicitly that that a small $\varepsilon$ is considered and focus on finding a policy to optimize $\hat{t}_{\varepsilon}^{\pifunc}$.

Suppose the distribution of $\C$ is known. Then, one could compute expected wall clock time as given in \eqref{eq:wct_proxy} for any state dependent stationary policy $\pifunc$. In this case, we could determine an optimal policy $\pifunc^*$ by solving the optimization problem,
\begin{equation}
\min_{\pifunc \in \mathcal{Q}_{m\lvert \mathcal{C} \rvert}} \quad
    \hat{t}_{\varepsilon}^{\pifunc} \:=\:\EXP[\norm{ \boldsymbol{h}_\varepsilon \left( \pifunc(\C) \right)}]
\EXP\left[{d}\left(\tau, \pifunc(\C), \C \right)\right],
\label{eq:offline_opt0}
\end{equation}
where $\mathcal{Q}_{m\lvert \mathcal{C} \rvert}$ is the set of all state-dependent stationary policies.

Alas, in practice, we often cannot directly solve the above problem, as the distribution of $\C$ is unknown. Hence, below, we propose a stochastic approximation like algorithm that achieves the optimal wall clock time of $\pifunc^*$ asymptotically.


\subsection{NAC-FL: Informal Description}
The idea underlying NAC-FL is to keep running estimates for $\EXP\left[\norm{ \boldsymbol{h}_\varepsilon \left( \Q \right)}\right]$ and $\EXP\left[ d(\tau, \Q, \C) \right]$ i.e.,
\begin{equation*}
\begin{split}
    \hat{r}_\varepsilon^{\round} = \frac{1}{\round}\sum_{k=1}^\round 
    \norm{ \boldsymbol{h}_\varepsilon \left( \q^{(k)}\right)}, \quad
    \hat{d}^{\round} = \frac{1}{\round}\sum_{k=1}^\round d\left(\tau, \q^{(k)}, \c^{(k)} \right).
\end{split}
\end{equation*}
Given a network state of $\c^{\round+1}$ at round $\round + 1$, and, a possible choice for compression parameters $\q$, the running averages would be updated as follows,
\begin{equation}
    \begin{split}
        \hat{r}_\varepsilon^{\round+1} &= \frac{\round}{\round+1} \hat{r}_\varepsilon^{\round} + \frac{1}{\round + 1} \norm{ \boldsymbol{h}_\varepsilon \left( \q\right)}, \\
        \hat{d}^{\round+1} &= \frac{\round}{\round+1} \hat{d}^{\round} + \frac{1}{\round + 1} d\left(\tau, \q, \c^{\round+1} \right).
    \end{split}
\end{equation}
As seen in \eqref{eq:wct_proxy}, to minimize the wall clock time one should minimize $\hat{r}_\varepsilon^{\round+1}\hat{d}^{\round+1}$, which can be expanded as,
\begin{align*}
\hat{r}_\varepsilon^{\round+1}\hat{d}^{\round+1}\! = &\frac{\round}{(\round+1)^2} \Big[r_\varepsilon^{\round} d\left(\tau, \q, \c^{\round+1}\right)\! +\! \hat{d}^{\round} \norm{ \boldsymbol{h}_\varepsilon \!\left( \q\right)} \Big] \\
&+ \frac{\round^2}{(\round+1)^2}r_\varepsilon^{\round}\hat{d}^{\round} + O\left(\frac{1}{(\round+1)^2}\right).
\end{align*}
Given the fact  that $\hat{r}_\varepsilon^{\round}$ and $\hat{d}^{\round}$ are constants, and neglecting the term $O\left(1/(n+1)^2 \right)$, an optimal choice for $\q^{\round+1}$ is
\begin{equation}
\q^{\round+1} = \underset{\q}{\text{argmin} }\quad  \hat{r}_\varepsilon^{\round} d\left(\tau, \q, \c^{\round+1} \right) + \hat{d}^{\round} \norm{ \boldsymbol{h}_\varepsilon \left( \q\right)}.
\label{eq:nac-fl-1}
\end{equation}

The NAC-FL policy is summarized in Algorithm \ref{algo:gradient-based}.
To retrieve policy informally described above the tunable parameters 
$\left( \beta_n \right)_{\round}$ and $\alpha$ should be set to 
$\beta_n = \frac{1}{n}$ and $\alpha=1$.

Consider two possible network states $\c$ and $\c'$ at a round $\round$. If the delay under state $\c$ is higher compared to $\c'$ for any compression parameters, then NAC-FL would choose a higher compression amount $\q$ for state $\c$ compared to compression amount $\q'$ for state $\c'$, i.e., $\q > \q'$ elementwise. This may be concluded from the selection policy of \eqref{eq:nac-fl-1}, and noting that $r_{\varepsilon}(\q)$ is increasing in $\q$ (Assumption \ref{ass:q_suff_stat}), and $d(\tau, \q, \c)$ is decreasing in $\q$ (Assumption \ref{ass:comm_delay}). 

Observe that since the estimates $\hat{r}_\varepsilon^{\round}$ and $\hat{d}^{\round}$ will initially change across rounds, NAC-FL may choose different compression parameters in two rounds for which the network was in the same state, i.e., NAC-FL is not a state-dependent stationary policy. Still, we will show NAC-FL is asymptotically near optimal. To develop this result we shall next present NAC-FL in a more formal manner. 

\begin{algorithm}[h]
\caption{NAC-FL}
\label{algo:gradient-based}
\SetKwInOut{Input}{Input}
\SetKwInOut{Output}{Output}
\SetKwInOut{Parameter}{Parameter}
\SetKwComment{Comment}{/* }{ */}
\Input{Initialization: $\hat{r}_\varepsilon^{(0)}, \hat{d}^{(0)}$ ; step size schedule $\{\beta_{\round}\}_{\round=1}^\infty$; parameter $\alpha$.}

\For{$\round = 1, \dots, $ \textbf{until termination}}{
    Server observes network state $\c^{\round}$ \;
    $\q^{\round} = \underset{\q}{\text{argmin} }\quad  \alpha \hat{r}_\varepsilon^{(\round-1)} d\left(\tau, \q, \c^{\round} \right) + \hat{d}^{(\round-1)} \norm{ \boldsymbol{h}_\varepsilon \left( \q\right)}$\;
    $\hat{r}_\varepsilon^{\round} = (1-\beta_{\round}) \hat{r}_\varepsilon^{(\round-1)} + \beta_{\round} \norm{ \boldsymbol{h}_\varepsilon \left( \q^{\round}\right)}$ \;
    $\hat{d}^{\round} = (1-\beta_{\round}) \hat{d}^{(\round-1)} + \beta_{\round} d(\tau, \q^{\round}, \c^{\round})$\;
}
\end{algorithm}

\subsection{NAC-FL: Formal Description}
Our NAC-FL approach is also inspired by the Frank-Wolfe Algorithm \cite{frank1956algorithm}.
We start by reformulating the optimization program in \eqref{eq:offline_opt0}. Denote by set $V_{\varepsilon}$ all possible pairs of expectations $(\hat{r}_\varepsilon, \hat{d})$,
\begin{equation}
\begin{aligned}
V_{\varepsilon} = \Big\{ (\hat{r}_\varepsilon, \hat{d}): \exists \pifunc \in \mathcal{Q}_{m\lvert \mathcal{C} \rvert} ~\text{s.t.}~ & \hat{r}_\varepsilon = \EXP\left[\norm{ \boldsymbol{h}_\varepsilon \left( \pifunc(\C)\right)}\right], \\
 &\hat{d} = \EXP\left[d\left(\tau, \pifunc(\C), \C \right)\right] \Big\}.
\end{aligned} 
\label{eq:feasible_set}
\end{equation}

Using the set $V_{\varepsilon}$, and denoting $H(r,d) \triangleq rd$, we may write the optimization \eqref{eq:offline_opt0}
characterizing the optimal policy $\pifunc^*$ as
\begin{equation}
\min_{\hat{r}_\varepsilon, \hat{d}} \{  H(\hat{r}_\varepsilon, \hat{d}) : 
 (\hat{r}_\varepsilon,\hat{d}) \in V_{\varepsilon} \}.
\label{eq:offline_opt}
\end{equation}
In this case, from a point $(\hat{r}_\varepsilon^{\round}, \hat{d}^{\round})$, the Frank-Wolfe update would be given as,
\begin{align}
(\hat{r}_\varepsilon, \hat{d}) &= \underset{(r,d) \in V_{\varepsilon}}{\text{argmin}} \quad \nabla H\left(\hat{r}_\varepsilon^{\round}, \hat{d}^{\round}\right)^\top \begin{pmatrix}
r \\
d
\end{pmatrix}, \label{eq:FW_update}\\
\hat{r}_\varepsilon^{\round+1} &= (1-\beta)\hat{r}_\varepsilon^{\round} + \beta \hat{r}_\varepsilon, \nonumber\\
\hat{d}^{\round+1} &= (1-\beta)\hat{d}^{\round} + \beta \hat{d}. \nonumber
\end{align}
The gradient  $\nabla H(\hat{r}_\varepsilon, \hat{d})$ is, $
\nabla H(\hat{r}_\varepsilon, \hat{d}) = \left(
\hat{d} \quad
\hat{r}_\varepsilon \right)^\top$. $V_{\varepsilon}$ is a set of feasible averages of $\hat{r}_{\varepsilon}$ and $\hat{d}$. Therefore, at round $(\round+1)$, not all the pairs $(r,d) \in V_{\varepsilon}$ may be achievable.
Hence, NAC-FL approximates equation \eqref{eq:FW_update} as,
\begin{equation*}
\q^{\round+1} = \underset{\q}{\text{argmin} }\quad  \hat{r}_\varepsilon^{\round} d\left(\tau, \q, \c^{\round+1} \right) + \hat{d}^{\round} \norm{ \boldsymbol{h}_\varepsilon \left( \q\right)}.
\end{equation*}

We have thus retrieved our proposed NAC-FL algorithm based on the Frank-Wolfe update, with 
one difference. The above derivation suggests the use of a fixed 
step-size $\beta$ at all rounds while the previously derived algorithm used
a decaying the step-size  $\beta_n = 1/n$. In our simulations, we will embrace the latter.

The following assumption is required to show the asymptotic optimality of NAC-FL. A state dependent stationary policy $\pifunc$ maps from a domain of finite size $\left\lvert \mathcal{C} \right\rvert$, to a range positive-real vectors of dimension $m$. Therefore, the policy may be represented by a positive-real vector, $\pivec$, of dimension $m\left\lvert \mathcal{C} \right\rvert$. Further, a vector $\r^{\pivec}$ may be obtained by applying $h_{\varepsilon}(\cdot)$ elementwise to the policy vector $\pivec$, $\r^{\pivec} \triangleq \boldsymbol{h}_{\varepsilon}(\pivec)$.  This representation is used in the following assumption.

\begin{assumption}
 The objective function $\hat{t}_{\varepsilon}^{\pifunc}$ of the optimization problem in \eqref{eq:offline_opt0} is a strictly quasiconvex function in $\pifunc$ in the following sense,
 \begin{equation}
 {\r^\pivec}^\top \left( \nabla_{\r^{\pivec}} \hat{t}_{\varepsilon}^{\pifunc}  \right) = 0 \; \implies \; {\r^{\pivec}}^\top \left(\nabla^2_{\r^{\pivec}} \hat{t}_{\varepsilon}^{\pifunc} \right) \r^{\pifunc} > 0.
 \label{eq:assumpt_unique_stationary}
 \end{equation}
\label{ass:unique_stationary}
\end{assumption}
Assumption \ref{ass:unique_stationary} ensures that there is a unique state dependent stationary policy $\pivec^*$ which optimizes \eqref{eq:offline_opt0}. 
  We have observed that the considered network model, compression model and the $\norm{ \boldsymbol{h}_\varepsilon \left( \q\right)}$ function associated with the FedCOM algorithm indeed satisfy this assumption.

Next we shall establish an optimality property for NAC-FL. To that end we shall consider executing NAC-FL  without termination with $\beta_\round =\beta$ for all $\round$ and let $\left(\Q_{\beta}^{\round}\right)_{\round}$,
$\hat{R}_{\varepsilon,\beta}^{\round}$ and $\hat{D}_\beta^{\round}$ 
be the corresponding sequence of compression parameters and
the associated estimates.  

\begin{theorem}
Let $\pifunc^*$ be the solution and
$\hat{t}_{\varepsilon}^{\pifunc^*}$
the minimum of the optimization problem in \eqref{eq:offline_opt0}. If Assumptions \ref{ass:q_suff_stat}-\ref{ass:unique_stationary} hold, then there exists a positive sequence $(\beta_i)_{i=1}^\infty$ with $\beta_i \to 0$ as $i\to\infty$, such that for every $\rho > 0$, there exists a thereshold $n_{th}(\rho)$ such that,
\begin{align*}
\lim_{i \to \infty} \sup_{\round \geq n_{th}(\rho)/\beta_i} &P\left( \norm{ \begin{pmatrix}
    \hat{R}_{\varepsilon, \beta_i}^{\round} - \EXP[\norm{ \boldsymbol{h}_\varepsilon \left( \pifunc^*(\C)\right)}] \\
    \hat{D}_{\beta_i}^{\round} - \EXP[d \left( \tau, \pifunc^*(\C), \C \right)]
\end{pmatrix}}
%
> \rho \right)
 = 0,
\end{align*}
\label{th:asympt_opt}
\end{theorem}
The proof of Theorem \ref{th:asympt_opt} is included in Appendix \ref{app:proof-sketch}.
\begin{remark}
Theorem \ref{th:asympt_opt} should be interpreted with some subtlety.
Say the desired error-tolerance $\varepsilon$ is very small such that
the number of rounds needed to converge under any compression policy is such that $r_{\varepsilon} \gg n_{th}(\rho)/\beta$. 
Then, based on Theorem \ref{th:asympt_opt}, one can show that NAC-FL compression choices will be near optimal after $n_{th}(\rho)/\beta$ rounds. Thereafter, since $r_\varepsilon$ is large, NAC-FL will make near optimal choices for long enough leading to a near optimal expected wall clock time.
\end{remark}

We further remark on the meaning of the asymptotic result in the context of minimizing the wall clock time. In applications that require a very low error-tolerance $\varepsilon$, one needs to have a large  number (i.e., in the asymptotic region) of communication rounds $r_\varepsilon$ for convergence. Therefore, even though the wall clock time obtained by using NAC-FL may be large in this setting, it is near-optimal compared to other methods of choosing compression parameters.

\section{Simulation}
\label{sec:simulations}

In this section, we present our simulation results.  We begin by describing additional model details used in our simulations. 

\subsection{Additional Model Details}

\subsubsection{Compression Model}
\label{ssec:compression_model}
We shall use the \emph{stochastic quantizer} in \cite{alistarh2017qsgd} which we will denote as $\quant_q(\cdot, b)$. The quantizer has a parameter $b \in \{1,\ldots,32\}$ corresponding to the number of bits used to represent each co-ordinate, in addition to the bit used to denote signs. When input a vector $\x$, it outputs,
\begin{equation}
\quant_q(\x, b) = \norm{\x}_\infty \mathsf{sign}(\x) \zeta(\x, b)
\label{eq:stochastic_quantization}
\end{equation}
where $\mathsf{sign}(\x)$ is the element-wise sign operator and where the function $\zeta(\x, b)$ uniformly quantizes each co-ordinate amongst $2^b-1$ levels between 0 and 1. That is, if $x_i/\norm{\x}_{\infty} \in \left[\frac{l}{2^b-1}, \frac{l+1}{2^b-1} \right)$, then it is quantized as,
\[
\zeta_i(\x, b) = \begin{cases}
\frac{l+1}{2^b-1}, & \text{with prob. } \frac{|x_i|}{\norm{x}_\infty}(2^b-1) -l, \\
\frac{l}{2^b-1}, & \text{otherwise.}
\end{cases}
\]
When $\x$ is quantized to $b$-bits per co-ordinate, its file size is given by the function, $s(b) = \norm{\x}_0 (b+1)+32$ bits. Here, the zero-norm, $\norm{\x}_0$, gives the length of the vector, the $1$ indicates the bit used to denote the sign, and the 32 bits are for a floating point number denoting the norm, $\norm{\x}_\infty$. Finally, if client $j$ uses the parameter $b_j$, then the vector of parameters used by the clients is denoted as, $\b = \left( b_j \right)_{j=1}^m$.
\subsubsection{Network Congestion Model} 

For purposes of evaluating the performance of various algorithms over different types of network congestion we propose the following general, albeit idealized, model. 
We let $\boldsymbol{C}^{\round}$ be a $m$ dimensional random vector denoting the \emph{Bit Transmission Delay} (BTD) for
clients during round $\round$. 
We further let $\boldsymbol{C}^{\round} = \text{exp}\left(\boldsymbol{Z}^{\round}\right)$ 
 i.e., coordinate-wise exponentiation of an $m$ dimensional first order autoregressive process given by
 $\left(\Z_i\right)_{i=0}^{\infty}$  where $\Z_0 = \boldsymbol{0}$, where 
\begin{equation}
    \Z^{\round} = A\Z^{(\round-1)} + \boldsymbol{E}^{\round}, \quad \forall \round \geq 1,
    \label{eq:autoregressive_evolution}
\end{equation}
where $A$ is an $m \times m$ deterministic matrix, and $\boldsymbol{E}^{\round} \sim \mathcal{N}(\boldsymbol{\mu}, \Sigma)$ are i.i.d.,  $m$ dimensional normal random vectors. 
Different correlations across time and clients may be modelled by varying $A$, $\boldsymbol{\mu}$ and $\Sigma$.
The marginal distributions of $\boldsymbol{C}^{\round}$
are thus log-normal but can be correlated in different ways based on the underlying autoregressive process. In particular:
\begin{description}
    \item [Homogeneous Independent:] the parameters are set to $A=0$, $
\boldsymbol{\mu}=\boldsymbol{1}$, and $\Sigma = \sigma^2 I.$ This results in a process which is independent and identically distributed across clients and time. 
\item [Heterogeneous Independent:]the parameters are set to $A=0$, $\mu_i = 0$ for $i \in \{1, \dots, 5\}$ and $\mu_i = 2$ for $i \in \{6, \dots, 10\}$, and $\Sigma = I.$ This results in a process which is independent across clients and time, with the BTD being lower for the first 5 clients compared to the rest. 
    \item[Perfectly correlated:] the parameters are set to $A$ such that $A_{i,j}=\frac{a}{m}$ where $a\in (0,1)$,  $
\boldsymbol{\mu}=\boldsymbol{0}$, and $\Sigma$ such that $\Sigma_{i,j} = \sigma^2=1.$ This results in a process where all clients see the same positively correlated time-varying delays. 
\item[Partially correlated:] the parameters are set to $A$ such that $A_{i,j}=\frac{a}{m}$, $
\boldsymbol{\mu}=\boldsymbol{0}$, and 
$\Sigma$ such that 
$\Sigma_{i,i} = 1$ and $\Sigma_{i,j} = 1/2$ for $i \neq j.$ 
This results in a process where delays are positively correlated accross clients and time. 
\end{description}


\subsubsection{Model for Round Durations}  
We will model the duration of a round  as the maximum across clients' delays, i.e.,  
\[
d(\tau, \b, \c) = \max_j[ \theta\tau +  c_j s(b_j)], 
\]
where $\theta$ represents the compute time per local computation, and $c_j s(b_j)$ the BTD of client $j$ times the size of the client $j$'s file capturing the time taken to communicate its update.  For simplicity we will set $\theta=0$. 

\subsubsection{Compression Level Choice Policies}
We compare NAC-FL to the following policies,

\paragraph{Fixed Bit} Here, a number $b$ is fixed, and all the clients use the stochastic quantizer $\quant_q(\x,b)$ from \eqref{eq:stochastic_quantization} with the parameter $b$. We present results for $b \in \{1,2,3\}$, as we didn't notice a performance improvement for larger parameters in our experiments.

\paragraph{Fixed Error} This method was suggested in \cite{bouzinis2022wireless} and is parameterized by a number $q$. At each round $\round$, the parameters $\b^{\round}$ of the stochastic quantizers are such that the average normalized-variance $\bar{q}^{\round}$ (see equation \eqref{eq:avg_norm_variance}) is smaller than $q$, and the duration of the round $d(\tau, \q^{\round}, \c^{\round})$ is minimized. We fix $q=5.25$ in all our experiments after finding it to be performing well across different settings.


\subsubsection{Machine Learning Model}
 We consider $m=10$ clients. We consider the MNIST dataset \cite{lecun1998gradient} which may be distributed homogeneously or heterogeneously amongst the clients. Since data is heterogeneous across clients in most FL applications, we consider the heterogenous data case. That is, each client has data corresponding to 1 unique label. The MNIST dataset has 60,000 training samples, 10,000 test samples and 10 labels. The clients and the server aim to train a fully connected neural network with the architecture $(784, 250, 10)$ with the sigmoid activation for the hidden layer. The learning rate is initialized to $\eta_{0} = 0.07$, and is decayed by a factor $0.9$ every 10 rounds. The aggregation rate and local computations per round are fixed throughout the training to $\gamma = 1$ and $\tau = 2$ respectively. As for the parameters of the NAC-FL policy, we set $\beta_\round = \frac{1}{\round}$, and $\alpha = 2$.  
 
 We measure the performance of the global model using the following,
 \paragraph{Training Loss} The training loss of the global model is the empirical cross entropy loss across the entire set of training samples. 
 \paragraph{Test Accuracy} The test accuracy is measured over all the test samples. Here, in some experiments, we run 20 simulations with different random seeds, and report the mean, 90th percentile and 10th percentile times to reach a test accuracy of 90\%. The 90th and 10th percentile scores are reported to capture the variation in performance across the 20 simulations. We also report a \emph{gain} metric, which is sample mean of the time gained to reach 90\% accuracy by NAC-Fl compared to a another policy reported in percentage. For instance, let $x_i$, $y_i$ be the times under NAC-FL and another policy for a random seed $i$, then the gain is $100*\left(\sum_{i=1}^{20} y_i/x_i - 1\right)/20$.  



\subsection{Simulation Results}
\subsubsection{Homogeneous Independent BTD} We simulated over $\sigma^2 \in \{1,2,3\}$ in order to study the change in performance over increasing variance. We observe that in all the cases, NAC-FL and the Fixed Error policy have very similar performance across all the considered statistics. This is because the Fixed Error policy was designed to operate well in the i.i.d., network delay case. However, both NAC-FL and Fixed Error policy perform better than all the Fixed Bit policies according to all the statistics across all the considered parameters. Moreover, we observed that the gap in the performance to Fixed Bit policies increased with increasing variance. For instance, the gain of the best Fixed Bit policy increased from 145\% to 250\% when the variance was increased from 1 to 3, while the gain of the worst fixed bit policy increased from 314\% to 881\%. This is as expected because both NAC-FL and Fixed Error policy adapt to the heterogenous delay of clients at any given time. Surprisingly, NAC-FL lagged behind Fixed Error policy in some metrics, but it performed better in terms of the gain metric in all the 3 cases, with the gain over Fixed Error policy ranging from 1\% to 8\%. 

\begin{table}[h]
\centering
\begin{tabular}{|l|l|l|l|l|l|l|}
\hline
                                                        $\sigma^2$            &         & 1 bit & 2 bits & 3 bits & Fixed Error   & NAC-FL        \\ \hline
\multirow{4}{*}{$1 $} & Mean    & 6.31  & 3.82   & 4.15   & \textbf{1.58} & 1.60          \\ \cline{2-7} 
                                                                    & 90th & 6.95  & 4.72   & 5.00   & \textbf{1.86} & 2.05          \\ \cline{2-7} 
                                                                    & 10th & 5.63  & 3.20   & 3.38   & 1.20          & \textbf{1.14} \\ \cline{2-7} 
                                                                    & Gain & 314\%  & 145\%   & 168\%   & 3\%          & -             \\ \hline
\multirow{4}{*}{$2 $} & Mean    & 54.8  & 32.5   & 34.9   & 12.5          & \textbf{12.2}          \\ \cline{2-7} 
                                                                    & 90th & 70.6  & 44.7   & 43.1   & \textbf{19.0}          & 20.8          \\ \cline{2-7} 
                                                                    & 10th & 42.5  & 19.2   & 21.0   & 6.26          & \textbf{5.82}          \\ \cline{2-7} 
                                                                    & Gain & 522\%  & 216\%   & 240\%   & 8\%          & -             \\ \hline
\multirow{4}{*}{$3 $} & Mean    & 799   & 430    & 458    & \textbf{165}           & 168           \\ \cline{2-7} 
                                                                    & 90th & 1430  & 752    & 665    & \textbf{318}           & 320           \\ \cline{2-7} 
                                                                    & 10th & 418   & 157    & 148    & \textbf{46.2}          & 57.9          \\ \cline{2-7} 
                                                                    & Gain & 881\%  & 270\%    & 250\%    & 1\%          & -             \\ \hline
\end{tabular}
\caption{Performance comparison of policies with homogeneous independent  BTD in terms of the mean, 90th percentile and 10th percentile times to reach 90\% test accuracy under the different policies, and their average sample-path gain compared to NAC-FL. All the numbers represented are in $10^7$ seconds. }
\label{table:iid_delay}
\end{table}

\subsubsection{Heterogeneous Independent BTD}
We considered this case since the first 5 clients would have consistently worse delay, NAC-FL and the Fixed Error policy would consistently compress the updates of those clients heavily. Since the data distribution is heterogeneous, it may be possible heavy compression of updates from specific clients throughout the training may hurt the performance. On the other hand, the Fixed Bit policies use the same amount of compression across all clients equally irrespective of their delays. Still, we observed that NAC-FL and the Fixed Error policy perform better than the Fixed Bit policies as can be seen in Table \ref{table:hetero}. In fact, performance in terms of the gain metric is very comparable to the i.i.d., network delay case with $\sigma^2=1$ in Table \ref{table:iid_delay}.

\begin{table}[h]
\centering
\begin{tabular}{|l|l|l|l|l|l|}
\hline
        & 1 bit & 2 bits & 3 bits & Fixed Error   & NAC-FL        \\ \hline
Mean    & 9.49  & 5.85   & 6.46   & 2.49          & \textbf{2.48} \\ \hline
90th & 11.5  & 7.16   & 8.09   & \textbf{3.48} & 3.54          \\ \hline
10th & 8.30  & 4.37   & 4.98   & 1.74          & \textbf{1.54} \\ \hline
Gain & 319\%  & 146\%   & 173\%   & 4\%          &  - \\ \hline
\end{tabular}
\caption{Performance comparison of policies with heterogenous independent BTD. The numbers shown are the mean, 90th percentile and 10th percentile times to reach 90\% test accuracy under the different policies, and their average sample-path gain compared to NAC-FL. All the numbers represented are in $10^8$ seconds. }
\label{table:hetero}
\end{table}

\subsubsection{Perfectly Correlated BTD}
 We will demonstrate that NAC-FL performs better than Fixed Error and Fixed Bit policies under increasing correlated delay across time since they are not designed to optimize the wall clock time under this case.

To study the variation of network delay across rounds, consider the marginal auto-regressive process of 1 client which may be represented by the following scalar autoregressive process,
\begin{equation}
Z^{\round} = a'Z^{(\round-1)} + E^{\round},
\label{eq:scalar_ar}
\end{equation}
where $E^{\round} \sim \mathcal{N}(0,1) $. We define metric called \emph{asymptotic variance}, denoted $\sigma_\infty^2$, which is designed to capture the variance, and long and short term correlations of a random process,
\begin{equation}
        \sigma_\infty^2 \triangleq \lim_{n \to \infty} \frac{\EXP\left[\left(Z^{(1)} + \dots + Z^{\round}\right)^2 \right]}{n}.
\end{equation}
For the autoregressive process in \eqref{eq:scalar_ar}, it may be computed to be, $\sigma_\infty^2 = 1/(1-a')^2$.

Table \ref{table:var_time_eq_client} shows the performance of the different policies under varying asymptotic variance of the marginals.  We observe that in addition to beating the baseline fixed bit policies on all the metrics, the NAC-FL performs better than the Fixed Error policy in most metrics as well. Considering the gain metric, we observe gain of 13\% over the Fixed Error policy for low asymptotic variance of $\sigma_{\infty}^2=1.56$, and is as large as 27\% for higher asymptotic variance of $\sigma_{\infty}^2=4$. Notably, in terms of the 10th percentile time to reach 90\% accuracy, the Fixed Error policy required 40\%, 23\% and 32\% more time compared to NAC-FL in the $\sigma_{\infty}^2$=1.56, 4 and 16 cases respectively.

\begin{table}[h]
\centering
\begin{tabular}{|l|l|l|l|l|l|l|}
\hline
                          $\sigma_\infty^2$     &   & 1 bit & 2 bits & 3 bits & Fixed Error & NAC-FL \\ \hline
\multirow{4}{*}{1.56}  & Mean    & 5.14  &  3.04   & 3.47   & 2.21 & {\bf 2.11}  \\ \cline{2-7} 
                          & 90th & 5.94  & 3.65   & 4.43   & {\bf 2.66} & 3.32   \\ \cline{2-7} 
                          & 10th & 3.88  & 2.38   & 2.18   & 1.43 & {\bf 1.02}   \\ \cline{2-7}
                          & Gain & 191\%  & 58\%   & 75\%   & 13\% & -   \\ \hline
\multirow{4}{*}{4}  & Mean    & 5.82  & 3.49   & 4.03   & 2.47 & {\bf 2.23} \\ \cline{2-7} 
                          & 90th & 7.43  & 4.77   & 6.28   & {\bf 3.94}  & 4.00 \\ \cline{2-7} 
                          & 10th & 3.88  & 2.22   & 1.98   & 1.21  & {\bf 0.981} \\ \cline{2-7}
                          & Gain & 252\%  & 82\%   & 107\%   & 27\%  & \\ \hline
\multirow{4}{*}{16} & Mean    & 8.42  & 5.19   & 6.15   & 3.75  & {\bf 3.36} \\ \cline{2-7} 
                          & 90th & 12.8  & 10.3   & 13.4   & 7.94  & {\bf 7.2} \\ \cline{2-7} 
                          & 10th & 4.34  & 1.40   & 1.67   & 1.15  & {\bf 0.87} \\ \cline{2-7}
                          & Gain & 316\%  & 72\%   & 98\%   & 21\%  & - \\ \hline
\end{tabular}
\caption{Performance comparison of policies with perfectly correlated BTD in terms of the mean, 90th percentile and 10th percentile times to reach 90\% test accuracy under the different policies, and their average sample-path gain compared to NAC-FL. All the numbers represented are in 10\textsuperscript{7} seconds. }
\label{table:var_time_eq_client}
\end{table}

\subsubsection{Partially Correlated BTD}
In Table \ref{table:var_time_var_client}, we show results for the partially correlated BTD case with  asymptotic variance $\sigma_{\infty}^2=4$. We consider this case to demonstrate that NAC-FL is effective with positive (but, not 100\%) correlation across clients as well. Indeed, we observe NAC-FL performing better compared to all the other policies across all the considered metrics, with a gain of 10\% over the Fixed Error policy, and 129\% over the best fixed bit policy. Notably, in terms of the 10th percentile and 90th percentile metrics, NAC-FL outperformed Fixed Error policy by 30\% and 15\% respectively. 
\begin{table}[h]
\centering
\begin{tabular}{|l|l|l|l|l|l|}
\hline
        & 1 bit & 2 bits & 3 bits & Fixed Error & NAC-FL        \\ \hline
Mean    & 13.6  & 8.33   & 9.51   & 4.22        & \textbf{3.83} \\ \hline
90th & 15.9  & 10.5   & 13.9   & 6.24        & \textbf{5.46} \\ \hline
10th & 9.51  & 5.47   & 5.80   & 2.64        & \textbf{2.02} \\ \hline
Gain & 307\%  & 129\%   & 159\%   & 10\%          &  - \\ \hline
\end{tabular}
\caption{Performance comparison of policies with partially correlated BTD in terms of the mean, 90th percentile and 10th percentile times to reach 90\% test accuracy under the different policies, and their average sample-path gain compared to NAC-FL. All the numbers represented are in $10^7$ seconds.}
\label{table:var_time_var_client}
\end{table}

Figure \ref{figabc} contains sample path plots of Training Loss and Accuracy vs Wall Clock Time for the independent homogeneous ($\sigma^2=2$), heterogeneous and perfectly correlated ($\sigma_\infty^2=4)$ BTD cases. Both accuracy and loss plots for NAC-FL and Fixed Error are overlapping in the independent homogeneous and heterogeneous BTD cases, as expected. However, in the perfectly correlated BTD case, NAC-FL dominates the performance of Fixed Error policy.

In summary, we observe that NAC-FL's performance is robust under a range of network models considered. NAC-FL vastly outperformed the baseline Fixed Bit policies in all the network models. The performance of NAC-FL was observed to be similar to that of Fixed Error policy in the independent BTD setting, albeit, it outperformed Fixed Error policy in terms of the gain metric under all the network models. Notably, the gap between NAC-FL and Fixed Error policy was observed to be noticeably high in the perfectly and paritally correlated BTD settings, where NAC-FL was able to adapt to positive correlations of BTD across time, whereas Fixed Error could not.

\begin{figure*}[h]%
\centering
\begin{subfigure}{.65\columnwidth}
\includegraphics[width=\columnwidth]{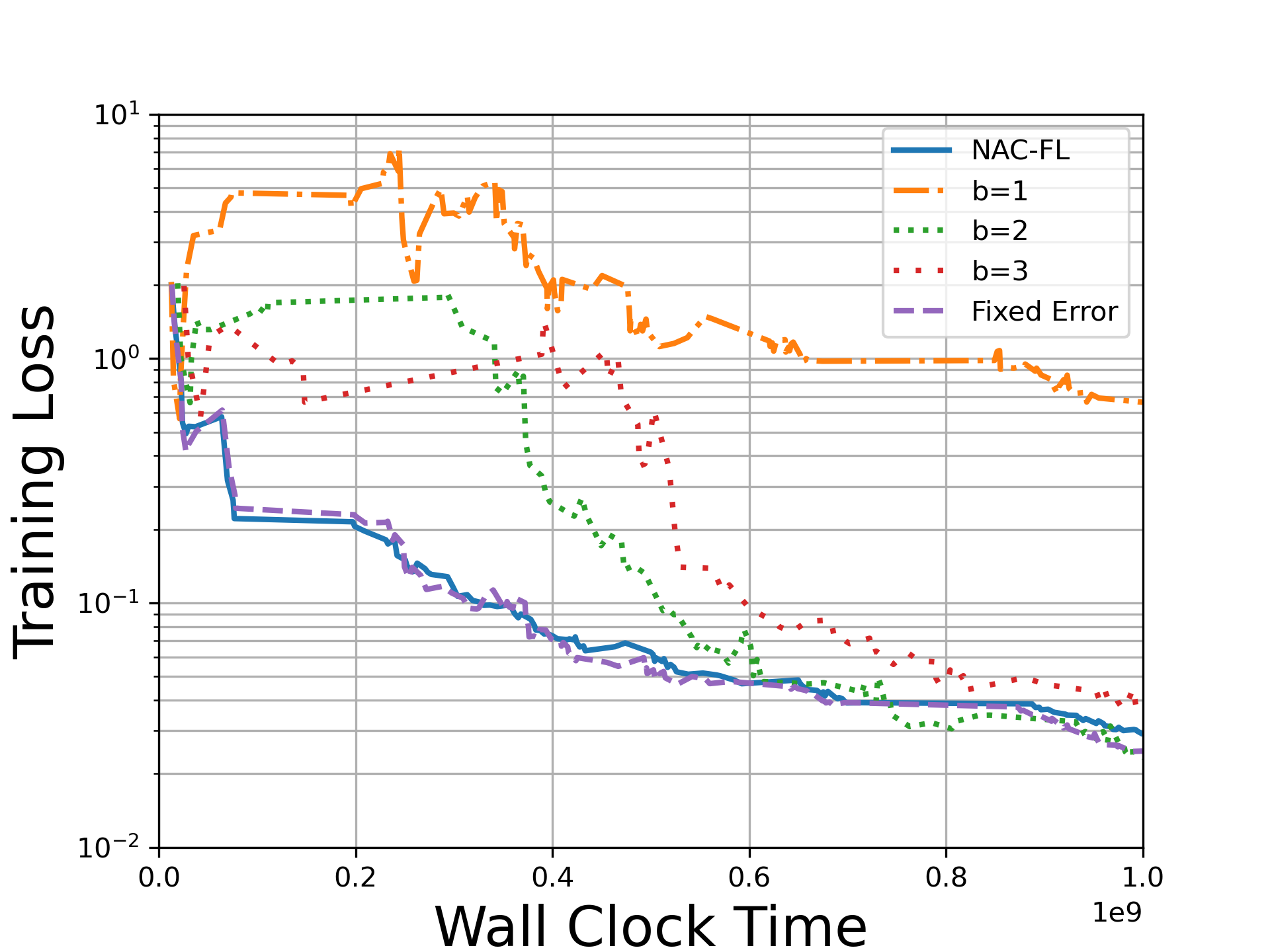}%
\caption{}%
\label{subfig:a}%
\end{subfigure}\hfill%
\begin{subfigure}{.65\columnwidth}
\includegraphics[width=\columnwidth]{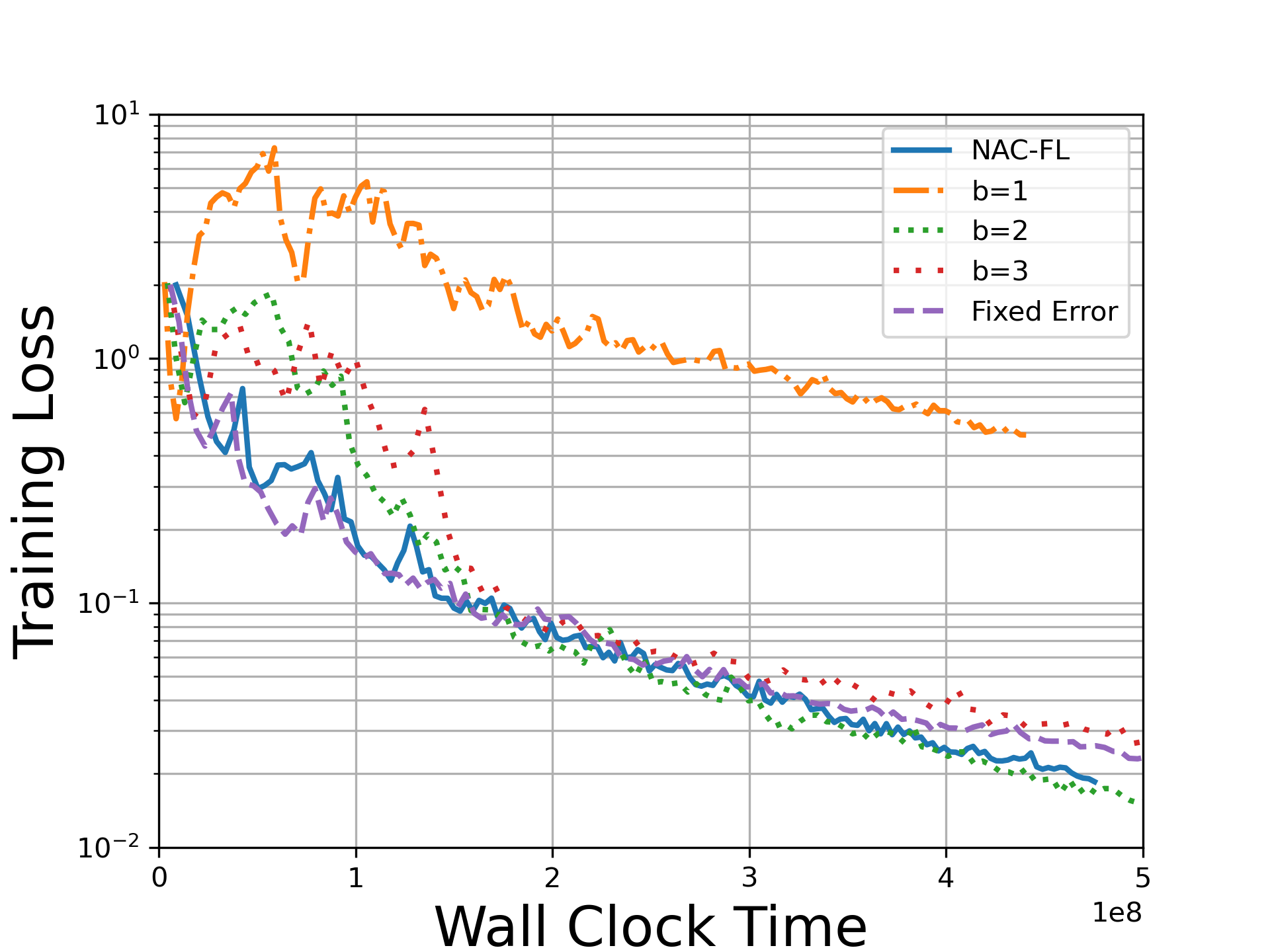}%
\caption{}%
\label{subfig:b}%
\end{subfigure}\hfill%
\begin{subfigure}{.65\columnwidth}
\includegraphics[width=\columnwidth]{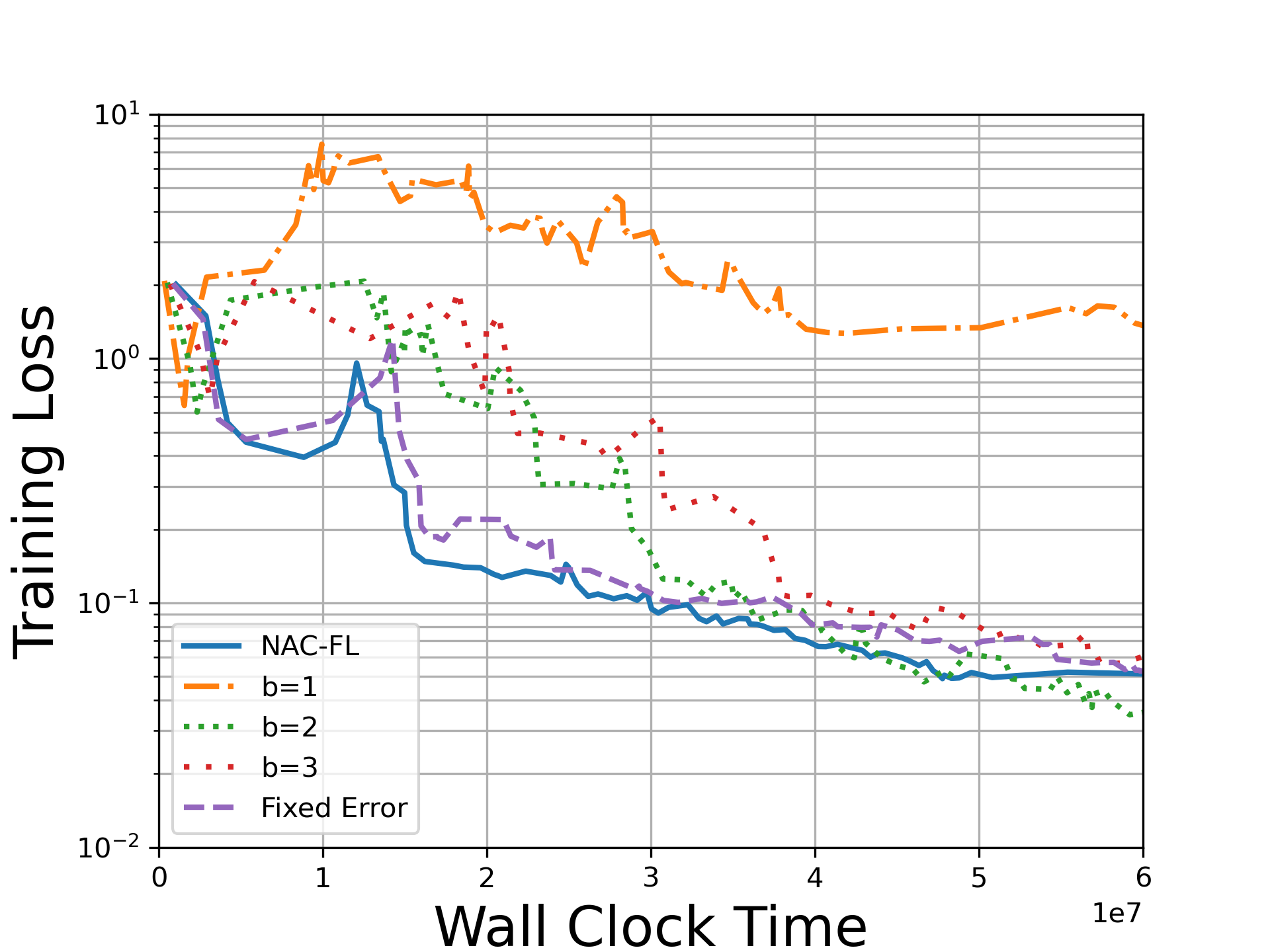}%
\caption{}%
\label{subfig:c}%
\end{subfigure}%

\bigskip
\begin{subfigure}{.65\columnwidth}
\includegraphics[width=\columnwidth]{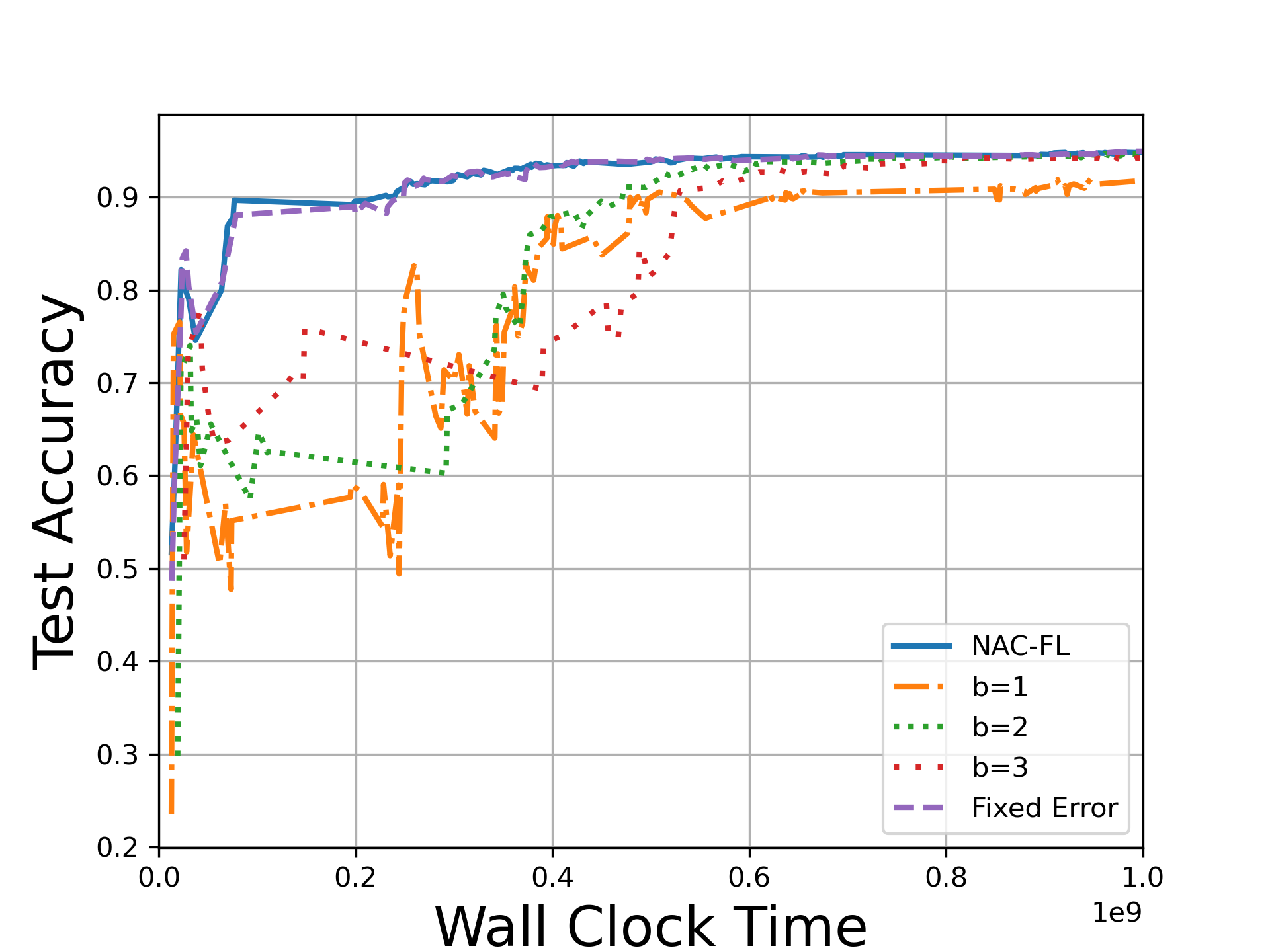}%
\caption{}%
\label{subfig:d}%
\end{subfigure}\hfill%
\begin{subfigure}{.65\columnwidth}
\includegraphics[width=\columnwidth]{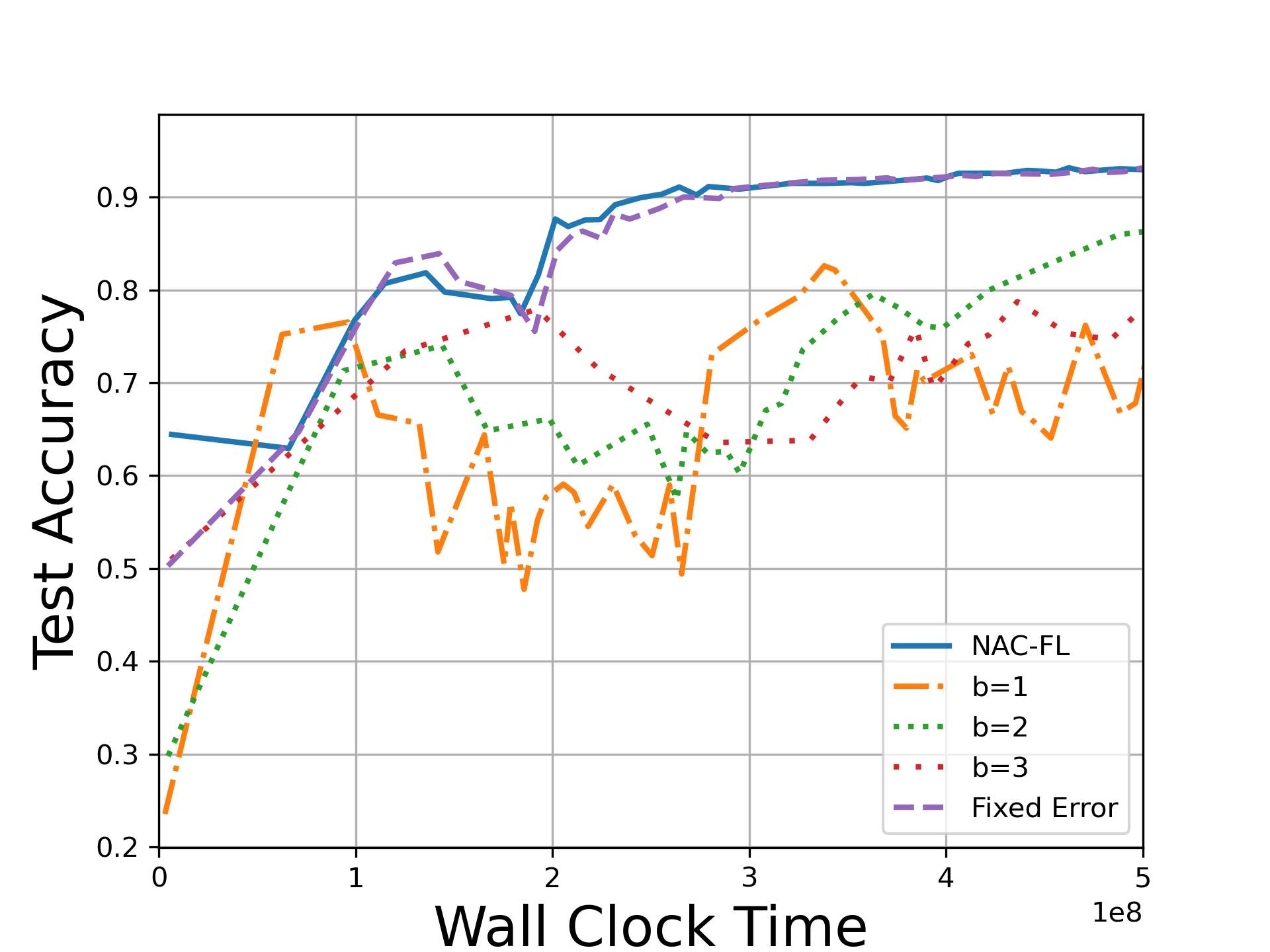}%
\caption{}%
\label{subfig:e}%
\end{subfigure}\hfill%
\begin{subfigure}{.65\columnwidth}
\includegraphics[width=\columnwidth]{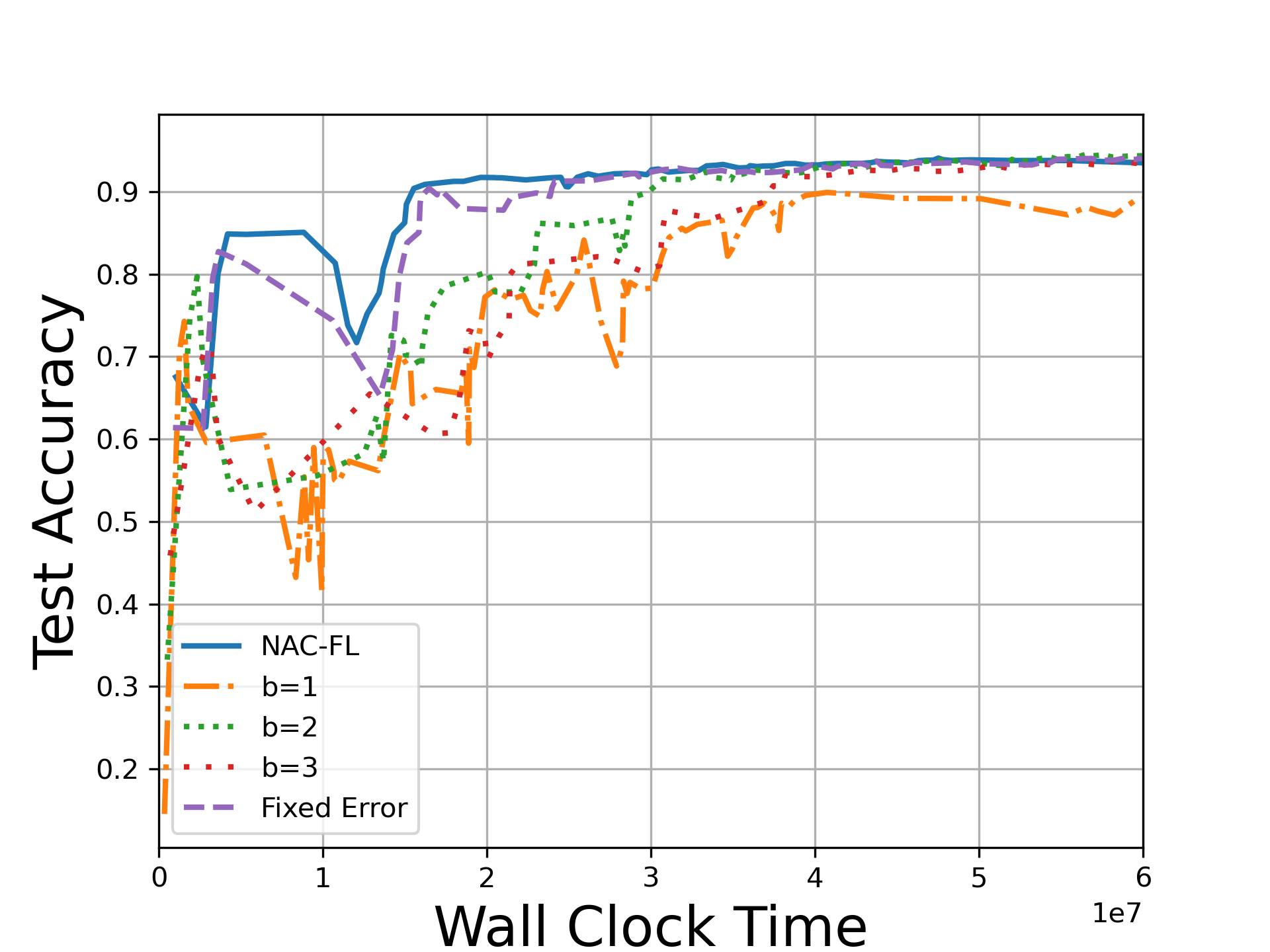}%
\caption{}%
\label{subfig:f}%
\end{subfigure}%
\caption{Plots of Training Loss and Test Accuracy vs Wall Clock time on different network models. Figures (a) and (d) correspond to homogeneous independent BTD case ($\sigma^2=2$), Figures (b) and (e) correspond to the heterogeneous independent BTD case, and Figures (c) and (f) correspond to the perfectly correlated BTD case ($\sigma_\infty^2=4$).}
\label{figabc}
\end{figure*}

\section{NAC-FL in Practice}
\label{sec:flac_practice}
In this section we briefly comment on some practical aspects underlying estimating model update delays. This involves estimating the network's current average BTD to each client. A simple approach to doing so is to observe that for the stochastic quantizer described in Section \ref{ssec:compression_model}, clients always send the vector of signs of their updates, no matter what are the bits per coordinate that will be chosen. So, as the clients send their signs, the server may probe the delay characteristics to estimate the BTD of clients without having to request vacuous (non update related) bits to do so. It may then use these estimates to perform the optimization in \eqref{eq:nac-fl-1} for the round.

\section{Conclusion}
\label{sec:conclusion}

Due to their distributed character FL algorithms are exposed to congestion across a potentially large number of network resources, whence one might say they are exposed to network congestion and variability at scale. Building adaptive algorithms that minimize the impact of time varying congestion across clients presents a significant challenge, particularly when the aim is to directly optimize the expected wall clock time. NAC-FL exemplifies a new class of robust algorithms to optimally adapt clients' lossy compression. This paper further provides the technical roadmap to formalizing and showing asymptotic optimality
for such algorithms.

\bibliographystyle{IEEEtran}
\bibliography{refs}

\onecolumn
\renewcommand{\thesection}{\Roman{section}}
\begin{appendices}

\setcounter{lemma}{0}
\renewcommand{\thelemma}{\Alph{section}.\arabic{lemma}}
\setcounter{prop}{0}
\renewcommand{\theprop}{\Alph{section}.\arabic{prop}}

\section{Federated Learning with Adaptive Compression (FLAC)}
\label{sec:flaq}
 In this section, we consider a variant of the FedCOM algorithm\cite{haddadpour2020federated}, which we will call FedCOM-V. FedCOM is based on fixing a quantization parameter throughout run of the FL algorithm. On the other hand, FedCOM-V allows for an arbitrary sequence of quantization parameters $\left(\q^{\round}\right)_{\round}$, in order to account for adaptive compression policies such as NAC-FL. FedCOM-V is presented in Algorithm \ref{algo:fedcom}. 

\begin{algorithm}[h]
\caption{FedCOM-V}
\label{algo:fedcom}
\SetKwInOut{Input}{Input}
\SetKwInOut{Output}{Output}
\SetKwInOut{Parameter}{Parameter}
\SetKwComment{Comment}{/* }{ */}
\Input{number of local computations schedule $\left(\tau_{\round}\right)_{\round=1}^\infty$, local learning rate schedule $\left(\eta_{\round} \right)_{\round=1}^\infty$, adaptively chosen global learning rate schedule $\left(\gamma_{\round} \right)_{\round=1}^\infty$, adaptively chosen number of rounds $r$, initial global model $\w^{1}$.}
\KwResult{$\w^{r+1}$: Final model}

\For{$\round = 1, \dots, {r}$}{
    \For{each client $j\in[m]$}{
        Set $\w_j^{1,\round} = \w^{\round}$ \;
        \For{$a = 1, \dots, \tau_{\round}$}{
            Sample a minibatch $\mathcal{Z}_j^{a,\round}$ and compute $\tilde{\g}_j^{a,\round} \triangleq \nabla f(\w_j^{a,\round}; \mathcal{Z}_j^{a,\round})$ \;
            $\w_j^{a+1, \round} = \w_j^{a,\round} - \eta_{\round} \tilde{\g}_j^{a,\round}$\;
        }
        Device sends $\tilde{\g}_{Qj}^{\round} = \quant((\w^{\round}-\w_j^{\tau_{\round}+1, \round})/\eta_{\round}, q_j^{\round})$ back to the server\;
    }
    Server computes, $\tilde{\g}_{Q}^{\round} = \frac{1}{m} \sum_{j=1}^m \tilde{\g}_{Qj}^{\round}$ \;
    Server computes $\w^{\round+1} = \w^{\round} - \eta_{\round}\gamma_{\round} \tilde{\g}_{Q}^{\round}$ and broadcasts to all devices\;
}
\end{algorithm}

In order to study the convergence properties of FedCOM-V, we make the following standard assumptions.
\begin{assumption}[Smoothness and Lower Boundedness]
The objective function $f(\cdot)$ is differentiable and $L$-smooth. That is, $\norm{\nabla f(\x) - \nabla f(\y)} \leq L \norm{\x -\y}$, for every $\x, \y \in \R^d$. Moreover, the optimal value of $f$ is lower bounded, $f^* = \min_\w f(\w) > -\infty$.
\label{ass:L-smooth}
\end{assumption}
\begin{assumption}[Bounded Variance]
For all clients $j \in [m]$ and rounds $n$ and local step $a$, we can sample an independent mini-batch $\mathcal{Z}_{j}^{a,n}$ of size $|\mathcal{Z}_{j}^{a,n}| = b$ and compute an unbiased stochastic gradient $\tilde{\g}_{j}^{a,n} = \nabla f(\w ; \mathcal{Z}_{j}^{a,n})$, i.e., $\EXP_{\mathcal{Z}_{j}^{a,n}}[\tilde{\g}_j] = \nabla f(\w_j^{a,n})$. Moreover, the variance is bounded by a constant $\sigma^2$, i.e., $\EXP_{\mathcal{Z}_{j}^{a,n}} \left[\norm{\tilde{\g}_{j}^{a,n} -\nabla f \left(\w_j^{a,n} \right)}^2 \right] \leq \sigma^2$.
\label{ass:sg_variance}
\end{assumption}
\begin{assumption}[Compression Model]
The output of the compressor $\quant(\x, q)$ is an unbiased estimator of $\x$, i.e., $\EXP[\quant(\x, q)|\x] = \x$, and, its variance is bounded as, $\EXP[\norm{\quant(\x, q)-\x}^2|x] \leq q\norm{\x}^2$. 
\label{ass:quantization}
\end{assumption}

We denote the \emph{maximum normalized-variance} by $q_\max$  and  the \emph{average normalized-variance} used at round $\round$ by
\begin{equation}
    \bar{q}^{\round} = \frac{1}{m} \sum_{j=1}^m q_j^{\round}.
    \label{eq:avg_norm_variance}
\end{equation}

The following Theorem states the relationship between $\left(\q^{\round}\right)_n$, $\varepsilon$ and $r_{\varepsilon}$ and is proved in Appendix \ref{app:fedcom_proof}.
\begin{theorem}
Let Algorithm \ref{algo:fedcom} be run with a sequence of compressors such that the average normalized-variance at round $\round$ is $\bar{Q}^{\round}$. Further, assume that the sequence $\left(\bar{Q}^{\round}\right)_n$ forms a stationary process with the stationary distribution represented by a random variable $Q$. 
To obtain $\EXP[\norm{\nabla f(\w)}^2] \leq \varepsilon$, we can choose,
\begin{equation*}
    r_{\varepsilon} = O\left(\log(1/\varepsilon)\frac{\EXP\left[\sqrt{Q+1}\right]}{\varepsilon}  \right), \quad \tau^{\round} = O\left( \round \right).
\end{equation*}

\label{th:fedcom_informal}
\end{theorem}

 The upper bound on $r_\varepsilon$ in Theorem \ref{th:fedcom_informal} provides a justification for Assumption \ref{ass:q_suff_stat} with $h_{\varepsilon}(q)=O(\sqrt{q+1})$. Here, $\tau^{\round}$ is a function of $\round$, but for the purposes of NAC-FL we may use the average of $\tau^{(1)}$ to $\tau^{(r_\varepsilon)}$ in the expression of the duration function. One may obtain a similar expression for other popular FL algorithms \cite{stich2018sparsified, alistarh2017qsgd}.

\newpage
\section{Proof of Theorem \ref{th:asympt_opt}}
\label{app:proof-sketch}

In this section we show that NAC-FL converges to the optimal solution asymptotically as $\beta \downarrow 0$. 

In order to consider the effect of $\beta \downarrow 0$ on NAC-FL estimates $\hat{R}_{\varepsilon}^{n}$ and $\hat{D}^{n}$ in \eqref{eq:FW_update}, we shall denote these as $\hat{R}_{\varepsilon, \beta}^{n}$ and $\hat{D}_{\beta}^{n}$ respectively. Let $\conv(V_{\varepsilon})$ be the convex hull of the set $V_{\varepsilon}$ defined in \eqref{eq:feasible_set}. Recall the positive sequence $(\beta_i)_i$ with $\beta_i \to 0$ from the statement of Theorem \ref{th:asympt_opt}. Letting $\X_{\beta}^{\round} \triangleq ( \hat{R}_{\varepsilon, \beta}^{\round} \: \hat{D}_{\beta}^{\round} )^\top$, and $H(\x) \triangleq x_1x_2$ over the domain $\mathbb{R}_{+}^{2}$, we have the following result.
\begin{prop}
Let the initialization $\X_{\beta}^{0}$ be equal to $\x^{0} \in \mathbb{R}_{+}^{2}$ almost surely for any $0<\beta<1$, then, for any $s > 0$, $\lim_{i \to \infty} \X_{\beta_i}^{\floor*{s/\beta_i}}$ exists, is almost surely deterministic and denoted as $\x(s) \triangleq \lim_{i \to \infty} \; \X_{\beta_i}^{(\floor*{s/\beta_i})}$. Further, for any $\x^{0} \in \mathbb{R}_{+}^{2}$, $\x(s)$ obeys the following differential equation,
\begin{equation}
    \begin{split}
        \x(0) &= \x^{0}, \\
        \dot{\x}(s) &= \boldsymbol{v}(s) - \x(s), \quad s > 0,\\
        \boldsymbol{v}(s) &= \underset{\boldsymbol{v} \in \conv(V_{\varepsilon})}{\textnormal{argmin}} \quad \nabla H\left(\x(s)\right)^\top \boldsymbol{v}, \quad s>0.
    \end{split}
    \label{eq:differential_equation}
\end{equation}
\label{prop:fluid_limit}
\end{prop}
The proof of Proposition \ref{prop:fluid_limit} is very similar to that of the main result of \cite{stolyar2005asymptotic}. For the sake of completeness, we briefly show the proof at the end of this section. From hereon, \eqref{eq:differential_equation} will be referred to as the Fluid-Frank-Wolfe (FFW) process.

\begin{prop}
Under Assumption \ref{ass:unique_stationary}, the FFW process in \eqref{eq:differential_equation} has a unique fixed point $\x^* \in \conv(V_{\varepsilon})$ such that,
\[
\x^* = \underset{\x \in \conv(V_{\varepsilon})}{\textnormal{argmin}} \quad \nabla H\left(\x^*\right)^\top\x.
\]
Moreover, $\x^* \in V_{\varepsilon}$, and $\x^*$ is the minimizer of $H$ over the set $V_{\varepsilon}$.
\label{prop:unique_stationary}
\end{prop}
Proposition \ref{prop:unique_stationary} is proved in Appendix \ref{app:theorem_1_proof}.

Denote, $G(\x) = \min_{\boldsymbol{v} \in \conv(V_{\varepsilon})} \nabla H(\x)^\top (\boldsymbol{v}-\x).$
Since, $\nabla H$ is a continuous function of $\x$, $G(\x)$ is a continuous function of $\x$ as well.

As a consequence of Proposition \ref{prop:unique_stationary}, there exists a unique point $\x^{*} \in \conv(V)_{\varepsilon}$ such that $G(\x^{*})=0$. For all other $\x \in \conv(V_{\varepsilon})$, $G(\x) < 0$. We will, in fact, prove a stronger result that $G(\x)$ is bounded away from 0 for points that are a distance away from $\x^*$.

\noindent\textbf{Claim 1:} for any $\omega > 0$, there exists a $\xi > 0$ such that if $\x \in \conv(V_{\varepsilon})$ and  $\norm{\x - \x^*} \geq \omega$, then $G(\x) < -\xi$.
\begin{proof}
 We prove this claim by contradiction. Suppose there exists an $\omega > 0$ such that for all $\xi > 0$, the set, \begin{align*}
 \mathcal{X}^{\xi} &\triangleq \left\{\x^\xi: \x^{\xi} \in \conv(V_{\varepsilon}), \norm{\x^{\xi}-\x^*}\geq \omega \text{ and } G(\x^{\xi}) \geq -\xi \right\}, \\
    &= \conv(V_{\varepsilon}) \bigcap \left\{\x^\xi: \norm{\x^{\xi}-\x^*}\geq \omega \right\} \bigcap \left\{ \x^\xi: G(\x^{\xi}) \geq -\xi \right\},
 \end{align*}
 is non-empty. 
 
 $\conv(V_{\varepsilon})$ is a compact set because it is the convex hull of a compact set, $V_{\varepsilon}$. Further, the sets $\left\{\x^\xi: \norm{\x^{\xi}-\x^*}\geq \omega \right\}$ and $\left\{ \x^\xi: G(\x^{\xi}) \geq -\xi \right\}$ are also closed because they are the pre-image of continuous functions over closed sets. Therefore, $\mathcal{X}^{\xi}$ is a closed set since it is the intersection of three closed sets. Further, it is also bounded because $\conv(V_{\varepsilon})$ is bounded. Therefore, $\mathcal{X}^{\xi}$ is a compact set.
 
 Consider $\xi_1 > \xi_2 > 0$. Since, $G(\x) \geq -\xi_2$ implies that $G(\x) \geq -\xi_1$, we have that $\mathcal{X}^{\xi_1} \supset \mathcal{X}^{\xi_2}$. Consider a decreasing sequence $(\xi_i)_{i\in \mathbb{N}}$ with $\lim_{i \to \infty} \xi_i = 0$. Then, $(\mathcal{X}^{\xi_i})_{i \in \mathbb{N}}$ is a decreasing sequence of compact and non-empty sets. We know that a decreasing sequence of non empty compact sets has a limit, and the limit is non-empty \cite{sparling}. Therefore,
 \[
 \mathcal{X}^0 \triangleq  \bigcap_{i=1}^\infty \mathcal{X}^{\xi_i},
 \]
 exists and is non-empty. Since $\xi_i \downarrow 0$, this means that any $\x \in \mathcal{X}^0$ satisfies $G(\x)=0$. Since $\norm{\x - \x^*} \geq \omega$ and $\x \in \conv(V_{\varepsilon})$ for any $\x \in \mathcal{X}^0$, this is a contradiction to the fact that $\x^*$ is a unique point in $\conv(V_{\varepsilon})$ with $G(\x)=0$. Therefore, there must exist some $\xi>0$ for which $\mathcal{X}^{\xi}$ is empty.
\end{proof}

Next we proceed to study the asymptotic convergence of the process $\x(\cdot)$. Note that since $V_{\varepsilon}$ is apriori unknown, the initialization $\x^{0}$ may not be in the set $V_{\varepsilon}$. Nevertheless, the FFW process $\x(\cdot)$ eventually reaches the set $\conv(V_{\varepsilon})$. In order to formalize this, let $\conv^{\zeta}(V_{\varepsilon})$ denote the $\zeta$-thickening of the set $\conv(V_{\varepsilon})$,
\[
\conv^{\zeta}(V_{\varepsilon}) = \{\y: \exists \x \in \conv(V_{\varepsilon}) \text{ such that } \norm{\y-\x}_2 \leq \zeta \}.
\]
\begin{prop}
Consider the FFW process defined in \eqref{eq:differential_equation}. For every $\zeta > 0$, there exists an $s_{\zeta}>0$ such that, $\x(s) \in \conv^{\zeta}(V_{\varepsilon})$ for all $s > s_{\zeta}$.
\label{prop:conv_zeta}
\end{prop}
The proof is the same as that of Corollary 2 in \cite{stolyar2005asymptotic}.

\noindent\textbf{Claim 2:} $\x(s) \to \x^*$ as $s \to \infty$. 
\begin{proof}
First we prove that $\underset{s \to \infty}{\text{lim inf}} \: \norm{\x(s) -\x^* } = 0 $ by contradiction. As a contradiction assume that there exists an $\omega > 0$ and $s^{\omega}>0$ such that $\norm{\x(s)-\x^*}>\omega$ for all $s>s^{\omega}$. 

Let $\xi > 0$ be the constant according to Claim 1 which ensures that $G(\x)<-\xi$ for all $\x$ in $\conv(V_{\varepsilon})$ satisfying $\norm{\x - \x^*}>\omega$. Moreover, due to continuity of $G(\cdot)$, there exists a $\xi'>0$ and a small enough $\zeta>0$ such that $G(\x)<-\xi'$ for all $\x$ in $\conv^{\zeta}(V_{\varepsilon})$ that satisfy $\norm{\x-\x^*} \geq \omega$. Due to Proposition \ref{prop:conv_zeta}, there exists a constant $s_{\zeta}>0$ such that $\x(s)\in \conv^{\zeta}(V_{\varepsilon})$ for all $s>s_{\zeta}$.

Define, $s^{\omega}_* = s^{\omega} + s_{\zeta} + (H(\x(s_{\zeta}+s^{\omega}))+1)/\xi'$. Then,
\begin{align*}
    H\left(\x\left(s^{\omega}_{*}\right)\right) 
    &= H(\x(s_{\zeta}+s^{\omega})) + \int_{s_{\zeta}+s^{\omega}}^{s^{\omega}_*} dH(\x(s)), \\
    &= H(\x(s_{\zeta}+s^{\omega})) +\int_{s_{\zeta}+s^{\omega}}^{s^{\omega}_*}\nabla H(\x(s))^\top \dot{\x}(s) ds, \\
    &= H(\x(s_{\zeta}+s^{\omega})) +\int_{s_{\zeta}+s^{\omega}}^{s^{\omega}_*}\nabla H(\x(s))^\top (\boldsymbol{v}(s)-\x(s))ds, \\
    &= H(\x(s_{\zeta}+s^{\omega})) +\int_{s_{\zeta}+s^{\omega}}^{s^{\omega}_*} G(\x(s))ds, \\
    &< H(\x(s_{\zeta}+s^{\omega})) + \int_{s_{\zeta}+s^{\omega}}^{s^{\omega}_*} -\xi' ds, \\
    &= H(\x(s_{\zeta}+s^{\omega})) - H(\x(s_{\zeta}+s^{\omega})) - 1 <0.
\end{align*}
Since, $H$ is a positive function, this is a contradiction. Therefore, there exists a time $s> s^{\omega}+s_{\zeta}$ such that $\norm{\x(s)-\x^*} < \omega$. Since this is true for every $\omega>0$ and $s^{\omega}>0$, we have proved that $ \underset{s \to \infty}{\text{lim inf}} \: \norm{\x(s)-\x^*} = 0$.

Next we prove that $\lim_{s \to \infty} \x(s) = \x^*$. Define,
\[
H^{\omega} = \max_{\substack{\x \in \conv^{\zeta}(V_{\varepsilon}) 
                        \\ \norm{\x-\x^*}\leq \omega}
                        } H(\x).
\]

Since $\underset{s \to \infty}{\text{lim inf}} \: \norm{\x(s)-\x^*} = 0$, there exists a constant $s^{\omega}_{th} > s_{\zeta}$ such that $\norm{\x(s^{\omega}_{th}) - \x^*} \leq \omega$.  Due to Proposition \ref{prop:conv_zeta}, for all $s>s^{\omega}_{th}$, we have $\x(s) \in \conv^{\zeta}(V_{\varepsilon})$. Therefore, if for any $s > s^{\omega}_{th}$, $H(\x(s)) > H^{\omega}$ is true, then $\x(s)$ satisfies $\x(s)\in \conv^{\zeta}(V_{\varepsilon})$ and $\norm{\x(s) - \x^*} > \omega$. Therefore, due to Claim 1 at all such points, the gradient satisfies, $dH(\x(s))/ds = G(\x(s)) < 0$. This implies that $H(\x(s)) \leq H^{\omega}$ for all $s > s^{\omega}_{th}$.

Moreover, by the continuity of $H(\cdot)$, $H^{\omega} \to H(\x^*)$ as $\omega \downarrow 0$. And, by definition of the minimum $\x^*$, $H(\x(s))\geq H(\x^*)$ for any $s > 0$. Therefore, by the Sandwich Theorem, $\lim_{s \to \infty} H(\x(s)) = H(\x^*)$.

Further, by the continuity of $H(\cdot)$ and the uniqueness of the minimum $\x^*$, $\lim_{s \to \infty} H(\x(s)) = H(\x^*)$ implies that $\lim_{s \to \infty} \x(s) = \x^*$.
\end{proof}

Claim 2 proves that the Fluid-Frank-Wolfe process converges to the optimal solution $\x^*$ asymptotically. In particular, for any $\rho>0$, there exists an $n_{th}(\rho)>0$ such that,
\[
\sup_{s>n_{th}(\rho)} \norm{\x(s)-\x^*} \leq \rho.
\]
Denote, $\x_{\beta}(s) = \X_\beta^{\floor*{s/\beta}}$. Then, since the functions converge as follows, $(\x_{\beta_i}) \to \x$ as $i\to \infty$, from the Continous Mapping Theorem \cite[Theorem 2.7]{billingsley2013convergence}, we have,
\[
\lim_{i \to \infty} \sup_{s > n_{th}(\rho)} P\left(\norm{\x_{\beta_i}(s)-\x^*}>\rho\right) = 0.
\]
The above implies the Theorem statement,
\[
\lim_{i \to \infty} \sup_{n > n_{th}(\rho)/\beta_i} P\left(\norm{\X_{\beta_i}^{n}-\x^*}>\rho\right) = 0.
\]

\subsection{Proof of Proposition \ref{prop:fluid_limit}}

Define the ``scaled process'' as, $\x_\beta(s) \triangleq \X_{\beta}^{\floor*{s/\beta}}$. Denote $D_{\mathbb{R}^2}[0,\infty)$ as the set of functions with domain $[0,\infty)$, range $\mathbb{R}^2$, and which are right continuous with left limits. Observe that $\x_\beta$ has sample paths in $D_{\mathbb{R}^2}[0,\infty)$ for any $0 < \beta < 1$.

Denote, $\boldsymbol{V}_\beta^n \triangleq \begin{bmatrix}
    \norm{\boldsymbol{h}_\varepsilon(\q^{n})} \\
    d(\tau, \q^{n}, \c)
\end{bmatrix}$, which is the \emph{action} taken by the NAC-FL algorithm (Algorithm \ref{algo:gradient-based}) at round $n$, and $\boldsymbol{v}_\beta(s) \triangleq \boldsymbol{V}_\beta^{\floor*{s/\beta}}$. 
Defining, 
\[
K = \max\left( 
\x^0, \max_{q\in [0,q_\max],\C\in \mathcal{C}} \norm{
\begin{bmatrix}
    \norm{\boldsymbol{h}_\varepsilon(\q)} \\
    d(\tau, \q, \C)
\end{bmatrix}
}
\right),
\]
by the update rule of NAC-FL, $
\x_\beta(s) = (1-\beta) \x_\beta(s-\beta) + \beta \boldsymbol{v}_\beta(s-\beta)$,  we have $\x_\beta(s) \leq K$ for any $0<\beta<1$ and $s>0$. Further, rearranging the NAC-FL update rule as,
\[
\x_\beta(s) - \x_\beta(s-\beta) = \beta \left( \boldsymbol{v}_\beta(s-\beta) - \x_\beta(s-\beta)  \right)
\]
we obtain, $\norm{\x_\beta(s) - \x_\beta(s-\beta)} \leq 2 \beta K$. More generally, for any $s_1, s_2 >0$, we have,
\[
\norm{\x_\beta(s_1)-\x_\beta(s_2)} \leq 2K \max(\beta, \lvert s_1 - s_2 \rvert).
\]
This implies the ``asymptotic Lipschitz'' property,
\[
\lim_{\beta \to 0} \norm{\x_\beta(s_1)-\x_\beta(s_2)} \leq 2K \lvert s_1 - s_2\rvert, \quad \forall s_1, s_2 >0.
\]
Then, by Corollary 7.4 in Chapter 3 of \cite{ethier2009markov}, the set of stochastic processes $\left\{\x_\beta(\cdot)\right\}_{0<\beta<1}$ is \emph{relatively compact}. Therefore, there exists a sequence $\left(\beta_i\right)_{i}$ with $\beta_i \to 0$ as $i \to \infty$ such that $\x_{\beta_i}(\cdot) \to \x(\cdot)$ as $i \to \infty$ for some stochastic process $\x(\cdot)$ with sample paths in $D_{\mathbb{R}^2}[0,\infty)$.

Next, we need to prove that $\x(\cdot)$ behaves according to \eqref{eq:differential_equation}. To do so, observe that due to the ``continuity property'' (i.e., $\norm{\X_\beta^n - \X_\beta^{n-1}}\leq 2K\beta$), for any $\delta > 0$, there exists a small enough $\beta>0$ and $\Delta>0$ such that, for any integer $n$ in the range $[s/\beta, (s+\Delta)/\beta]$, we have,
\[
\lvert \left(\nabla H(\X_\beta^n) \right)^\top \boldsymbol{V}_\beta^n - Y^*_{\C^n} \rvert \leq \delta,
\]
where,
\[
Y^*_{\C} \triangleq \min_{q \in [0,q_\max]} \left(\nabla H(\x_\beta(s)) \right)^\top \begin{bmatrix}
    \norm{\boldsymbol{h}_{\varepsilon}(\q)} \\
    d(\tau, \q, \C)
\end{bmatrix}, \quad \C \in \mathcal{C}.
\]
The above equations say that the optimal action at any round in the considered range is very close to the optimal action at the start of the range, for an appropriate selection of parameters. Summing across $n$ in the range $[s/\beta, (s+\Delta)/\beta]$ we obtain,
\[
\left\lvert \sum_{s/\beta \leq n \leq (s+\Delta)/\beta} \left(\nabla H(\X_\beta^n) \right)^\top \boldsymbol{V}_\beta^n - \sum_{s/\beta \leq n \leq (s+\Delta)/\beta} Y^*_{\C^n} \right\rvert \leq \delta \Delta/\beta.
\]
Multiplying by $\beta$ on both sides, from the definition of the scaled process, we have,
\[
\left\lvert \int_{s}^{s+\Delta} \left(\nabla H(\x_\beta(\xi)) \right)^\top \boldsymbol{v}_\beta(\xi) d\xi - \sum_{s/\beta \leq n \leq (s+\Delta)/\beta} \beta Y^*_{\C^n} \right\rvert \leq \delta \Delta.
\]
From the Law of Large Numbers for Markov Chains, we have $\lim_{\beta \to 0} \sum_{s/\beta \leq n \leq (s+\Delta)/\beta} \beta Y^*_{\C^n} = \Delta \sum_{\C \in \mathcal{C}} \mu(\C) Y^*_{\C}$. Similar to the convergence of $\x_\beta$ shown above, one can prove convergence of $\boldsymbol{v}_\beta$ to a process $\boldsymbol{v}$. Therefore, taking limit $i\to \infty$ along the sequence $(\beta_i)_i$, we get,
\[
\left\lvert \int_{s}^{s+\Delta} \left(\nabla H(\x(\xi)) \right)^\top \boldsymbol{v}(\xi) d\xi - \Delta \sum_{\C \in \mathcal{C}} \mu(\C) Y^*_{\C} \right\rvert \leq \delta \Delta.
\]
Observe that $\sum_{\C \in \mathcal{C}} \mu(\C) Y^*_{\C} = \min_{\boldsymbol{v} \in \conv(V_\varepsilon)} \left(\nabla H(\x(s)) \right)^\top \boldsymbol{v}$. Therefore, by choosing a $\Delta$ small enough, we get,
\[
\left\lvert \left(\nabla H(\x(s)) \right)^\top \boldsymbol{v}(s) - \min_{\boldsymbol{v} \in \conv(V_\varepsilon)} \left(\nabla H(\x(s)) \right)^\top \boldsymbol{v}  \right\rvert \leq \delta.
\]
Since $\delta$ can also be chosen arbitrarily small, we have,
\[
\left(\nabla H(\x(s)) \right)^\top \boldsymbol{v}(s) = \min_{\boldsymbol{v} \in \conv(V_\varepsilon)} \left(\nabla H(\x(s)) \right)^\top \boldsymbol{v}.
\]
\setcounter{lemma}{0}
\renewcommand{\thelemma}{\Alph{section}.\arabic{lemma}}
\setcounter{prop}{0}
\renewcommand{\theprop}{\Alph{section}.\arabic{prop}}
\setcounter{claim}{0}
\renewcommand{\theclaim}{\Alph{section}.\arabic{claim}}
\newpage
\section{Proof of Lemma \ref{lm:stationary_q}}
\label{app:proof_stationary_q}

In this section we show that a state-dependent stationary policy asymptotically optimizes the wall clock time. To do so, we first define the notion of a \emph{type} for sequences of network states and compression parameters. Then, we show that for a given network state sequence, a policy for choosing compression parameters which depends on the sequence type optimizes the wall clock type. Finally, because the type asymptotically concentrates for markov processes, we show that a state-dependent stationary policy asymptotically optimizes the wall clock time.

We start by defining the notion of an empirical distribution, called type, and its associated expectation and conditional expectations.
\begin{definition}[Type]
The type of a finite sequence, $x^{[r]} \triangleq \left(x^{\round}\right)_{\round=1}^r$ with elements in domain $\mathcal{X}$, is a function, $\hat{p}\left( \cdot \; ; x^{[r]} \right): \mathcal{X} \to [0,1]$, defined as,
\[
\hat{p}\left( x \; ; \; x^{[r]}\right) = \frac{\sum_{\round=1}^r \mathbbm{1}\left(x^{\round} = x \right)}{r}, \qquad \forall x \in \mathcal{X},
\]
where $\mathbbm{1}(x = y)=1$ if $x=y$, and $0$ otherwise.
\label{def:type}
\end{definition}
Similarly, the \emph{conditional type} and the \emph{joint type} are defined as follows.
\begin{definition}[Joint Type and Conditional Type]
The joint type of two finite sequences, $x^{[r]}$ and $y^{[r]}$ with domains $\mathcal{X}$ and $\mathcal{Y}$ respectively, is a function, $\hat{p}\left( \cdot \; ; x^{[r]}, y^{[r]} \right): \mathcal{X} \times \mathcal{Y} \to [0,1]$, defined as,
\[
\hat{p}\left( x, y \; ; \; x^{[r]}, y^{[r]}\right) = \frac{\sum_{\round=1}^r \mathbbm{1}\left(x^{\round} = x \; , \; y^{\round} = y \right)}{r}, \qquad \forall x \in \mathcal{X}, \; y \in \mathcal{Y}.
\]
The conditional type $\hat{p}\left( \cdot \vert \cdot \; ; x^{[r]}, y^{[r]} \right): \mathcal{X} \times \mathcal{Y} \to [0,1]$ is defined as,
\begin{align*}
  \hat{p}\left( x \vert y \; ; \; x^{[r]}, y^{[r]}\right) &= \frac{\sum_{\round=1}^r \mathbbm{1}\left(x^{\round} = x \; , \; y^{\round} = y \right)}{\sum_{\round=1}^r \mathbbm{1}\left(y^{\round} = y\right)}, \qquad \forall x \in \mathcal{X}, \; y \in \mathcal{Y} \text{ such that } \hat{p}(y;y^{[r]}) > 0, \\
    &= \frac{\hat{p}\left( x, y ; x^{[r]}, y^{[r]} \right)}{\hat{p}\left( y ; y^{[r]} \right)}.
\end{align*}
\label{def:joint-type}
\end{definition}
Then, the expectation and conditional expectation with respect to the type may be defined as follows.
\begin{definition}[Expectation and Conditional Expectation]
The expectation of a non-negative function $g: \mathcal{X} \to \mathbb{R}^{+}$ with respect to type $\hat{p}(\cdot; x^{[r]})$ is defined as\footnote{If $\mathcal{X}$ is uncountably infinite, then, $\sum_{x \in \mathcal{X}} g(x) \triangleq \sup\left\{\sum_{x \in \mathcal{F}} g(x): \mathcal{F}\subset \mathcal{X}, \: \mathcal{F} \text{ is finite}\right\}$.},
\[
\hat{\EXP}\left[ g(X) ; x^{[r]} \right] \triangleq \sum_{x \in \mathcal{X}} g(x)\hat{p}(x; x^{[r]}),
\]
where $X$ denotes a random variable with distribution $\hat{p}(x; x^{[r]})$.
Similarly, the conditional expectation of a non-negative function $l: \mathcal{X} \to \mathbb{R}^{+}$ with respect to the type $\hat{p}(\cdot|\cdot; x^{[r]}, y^{[r]})$ is defined as,
\[
\hat{\EXP}\left[ l(X)|Y=y ; x^{[r]}, y^{[r]} \right] \triangleq \sum_{x \in \mathcal{X}} l(x)\hat{p}(x|y; x^{[r]}, y^{[r]}), \quad \forall y \in \mathcal{Y} \text{ such that } \hat{p}(y;y^{[r]}) > 0,
\]
where the random variable pair $(X,Y)$ has joint distribution $\hat{p}(x, y; x^{[r]}, y^{[r]})$.
\end{definition}

\begin{prop}
Suppose Assumptions~\ref{ass:q_suff_stat} and \ref{ass:comm_delay} hold, and let $(\c^{\round})_{\round}$ denote an observed sequence of network states and $(\q^{\round})$ denote a sequence of compression parameters. Then, for any positive integer $r$ and positive $\varepsilon$ with the associated function $\boldsymbol{h}_{\varepsilon}(\cdot)$ defined in Assumption \ref{ass:q_suff_stat}, there exists a sequence dependent, state dependent stationary policy $\pifunc$ such that,
    \begin{equation}
        \sum_{\round = 1}^{r} \norm{\boldsymbol{h}_\varepsilon\left(\q^{\round}\right)} 
            \geq  
            \sum_{\round = 1}^{r} \norm{\boldsymbol{h}_\varepsilon\left(\pifunc\left(\c^{\round}\right)\right)},
    \label{eq:delay_convexity_1}
    \end{equation}
    and,
    \begin{equation}
        \sum_{\round = 1}^{r} d\left(\tau, \q^{\round}, \c^{\round}\right) 
            \geq 
            \sum_{\round = 1}^{r} d\left(\tau, \pifunc\left(\c^{\round}\right), \c^{\round}\right).
    \label{eq:delay_convexity_2}
    \end{equation}
    \label{prop:delay_convexity}
\end{prop}
\begin{proof}[Proof of Proposition \ref{prop:delay_convexity}]
Given the sequence $(\q^{\round}, \c^{\round})_{\round=1}^{r}$, we obtain the joint type $p(\cdot, \cdot ; \q^{[r]}, \c^{[r]})$. Thus, one may interpret the sequence as given an observed network state $\c$, the policy plays the compression parameters $\q$ with probability  $\hat{p}(\q|\c \; ; \; \q^{[r]}, \c^{[r]})$. Define the state-dependent stationary policy $\pifunc$ as playing the conditional mean (w.r.t., the function $\boldsymbol{h}_{\varepsilon}$) given any network state $\c$. That is,
\begin{equation}
    \pifunc(\c) 
        = \boldsymbol{h}_{\varepsilon}^{-1}\left(\hat{\EXP}\left[\boldsymbol{h}_{\varepsilon}(\Q) \big\vert \C=\c; \q^{[r]}, \c^{[r]} \right]\right), \quad \forall \c \in \mathcal{C} \text{ such that } \hat{p}(\c; \c^{[r]}) > 0.
\label{eq:pi_opt_seq_depdt}
\end{equation}
 Such a choice for $\pifunc(\c)$ always exists because $\boldsymbol{h}_{\varepsilon}(\cdot)$ is continuous, bounded and strictly increasing coordinate-wise applied function which implies that the inverse operator of $\boldsymbol{h}_{\varepsilon}(\cdot)$ is well-defined. Note that $\boldsymbol{h}_{\varepsilon}(\pifunc(\c)) 
        = \hat{\EXP}\left[\boldsymbol{h}_{\varepsilon}(\Q) \big\vert \C=\c; \q^{[r]}, \c^{[r]} \right]$. So, due to the convexity of $\norm{\cdot}$,
\begin{equation}
    \norm{\boldsymbol{h}_{\varepsilon}\left(\pifunc(\c)\right)} 
    \leq 
    \hat{\EXP}\left[\norm{\boldsymbol{h}_{\varepsilon}\left(\Q\right)}|\C=\c \; ; \; \q^{[r]}, \c^{[r]}\right], \quad \forall \c \in \mathcal{C}.
\label{eq:norm_convexity}
\end{equation}
Then,
\begin{equation}
    \begin{split}
        \sum_{\round=1}^{r} \norm{\boldsymbol{h}_{\varepsilon}\left(\pifunc(\c^{\round})\right)} 
    &= 
    r \hat{\EXP}\left[ \norm{\boldsymbol{h}_{\varepsilon}(\pifunc(\C))} \; ; \; \c^{[r]} \right], \\
    &\stackrel{(a)}{\leq} 
    r \hat{\EXP}\left[ \hat{\EXP}\left[ \norm{\boldsymbol{h}_{\varepsilon}(\Q)} \vert \C=\c \; ; \; \q^{[r]}, \c^{[r]} \right] \; ; \; \c^{[r]} \right], \\ 
    &\stackrel{(b)}{=}
    r \hat{\EXP}\left[ \norm{\boldsymbol{h}_{\varepsilon}(\Q)} \; ; \; \q^{[r]}  \right], \\
    &= 
    \sum_{\round=1}^{r} \norm{\boldsymbol{h}_{\varepsilon}\left(\q^{\round})\right)},
    \end{split}
    \label{eq:norm_dominance}
\end{equation}
where (a) follows from \eqref{eq:norm_convexity}, and (b) follows from the Tower-rule of expectations.

Next we bound $d(\tau,\pifunc(\c), \c)$. By the definition of $\pifunc$ in \eqref{eq:pi_opt_seq_depdt}, for all $\c \in \mathcal{C}$,
\begin{align}
    d\left(\tau, \pifunc(\c), \c \right) 
    &= 
    d\left(\tau, \boldsymbol{h}_{\varepsilon}^{-1}\left(\boldsymbol{h}_{\varepsilon}(\pifunc(\c))\right), \c \right), \nonumber \\
    &\stackrel{(a)}{=} 
    d\left(\tau, \boldsymbol{h}_{\varepsilon}^{-1}
    \left( \hat{\EXP}\left[\boldsymbol{h}_{\varepsilon}(\Q) \big\vert \C=\c; \q^{[r]}, \c^{[r]} \right] \right)
    , \c \right), \nonumber \\
    &\stackrel{(b)}{\leq} 
    \hat{\EXP}\left[d\left( \tau, \boldsymbol{h}_{\varepsilon}^{-1}\left(\boldsymbol{h}_{\varepsilon}(\Q) \right), \c \right)\vert \C=\c \; ; \; \q^{[r]}, \c^{[r]} \right], \nonumber \\
    &= 
    \hat{\EXP}\left[d\left( \tau, \Q, \c \right)\vert \C=\c \; ; \; \q^{[r]}, \c^{[r]} \right], \label{eq:lm_1_delay_convexity}
\end{align}
where (a) follows from the definition of policy $\pifunc$, and (b) follows from the convexity of $d(\tau, \boldsymbol{h}_{\varepsilon}^{-1}(\cdot), \c)$ (Assumption~\ref{ass:comm_delay}). \eqref{eq:lm_1_delay_convexity} is analogous to \eqref{eq:norm_convexity}. Therefore, we may repeat the same calculation in \eqref{eq:norm_dominance} for the delay,
\begin{align*}
    \sum_{\round=1}^{r} d(\tau, \pifunc(\c^{\round}), \c^{\round}) 
    &= 
    r \hat{\EXP}\left[ d(\tau, \pifunc(\C), \C) \; ; \; \c^{[r]} \right], \\
    &\leq
    r \hat{\EXP}\left[ \hat{\EXP}\left[ d(\tau, \Q, \C) \vert \C=\c \; ; \; \q^{[r]}, \c^{[r]} \right] \; ; \; \c^{[r]} \right], \\ 
    &=
    r \hat{\EXP}\left[ d(\tau, \Q, \C) \; ; \; \q^{[r]}, \c^{[r]}  \right], \\
    &= 
    \sum_{\round=1}^{r} d(\tau, \q^{\round}, \c^{\round}),
\end{align*}
\end{proof}

Equation \eqref{eq:delay_convexity_1} in Proposition \ref{prop:delay_convexity} states that if the FL algorithm with a sequence of compression parameters $(\q^{\round})_{\round}$ has reached an error tolerance of $\varepsilon$ by round $r$, then, under Assumption \ref{ass:q_suff_stat}, it has also reached error tolerance $\varepsilon$ under  sequence of compression parameters $(\pifunc(\c^{\round}))_{\round}$ by around $r$. Moreover, \eqref{eq:delay_convexity_2} states that $(\pifunc(\c^{\round}))_{\round}$ takes lesser amount of time up to round $r$ compared to $(\q^{\round})_{\round}$. However, this construction of $\pifunc$ is sequence dependent. More specifically, it is dependent on the type $\hat{p}(\cdot \; ; \; \c^{[r]})$  of the network state sequence observed. In order to prove Lemma \ref{lm:stationary_q}, we need to construct a state-dependent but sequence-independent stationary policy that is near-optimal in minimizing the expected wall clock time. Therefore, in the following, we first define the notion of a \emph{typical set} and show in Proposition \ref{prop:ergodicity} that the type of an observed network state sequence concentrates around its mean with high probability.

\begin{definition}[Typical Set]
For a distribution $p$ on network sets, a typical set with parameters $(r, \nu)$, is defined as,
\[
\mathcal{T}_{\nu}^{r}(p) \triangleq \left\{ \c^r: \left\lvert \hat{p}(\c\vert \c^r) - p(\c) \right\rvert \leq \nu p(\c), \textnormal{ for all } \c \in \mathcal{C} \right\}.
\]
\end{definition}

We will use the following result called the \emph{Typical Averaging Lemma} for typical sets\cite[Section 2.4]{el2011network}.
\begin{lemma}
Let $\c^r \in \mathcal{T}_{\nu}^{r}(p)$. Then, for any non-negative function $g:\mathcal{C} \to \mathbb{R}^+$,
\[
(1-\nu) \EXP\left[g(\C)\right] \leq \frac{1}{r} \sum_{\round=1}^r g\left(\c^{\round}\right) \leq (1+\nu) \EXP\left[g(\C)\right],
\]
where $\C$ is a random variable with distribution $p$.
\label{lm:typical_averaging_lemma}
\end{lemma}

Next, due to ergodicity of stationary  Markov chains, we have the following proposition which is proved at the end of this section.
\begin{prop}
Let Assumption \ref{ass:markov} be true. Then, there exist positive constants $\kappa_1$ and $\kappa_2$ such that, for every $\nu>0$, and $r \in \mathbb{N}$,
\[
P\left( \exists r' \geq r \text{ such that } \C^{r'} \not{\in} \mathcal{T}_{\nu}^{r'}(\mu) \right) \leq   \kappa_1 \exp \left(-\kappa_2\nu^2 r\right).
\]
\label{prop:ergodicity}
\end{prop}

 Lemma \ref{lm:typical_averaging_lemma} will be used to argue that if two network-state sequences have similar types (they belong to $\mathcal{T}_{\nu}^{r}(\mu)$), then a state dependent stationary policy $\pifunc$ will have a similar expected wall clock to converge under both sequences. Then, Proposition \ref{prop:ergodicity} will be used to argue that one observes a typical network state sequence with high probability. We proceed to prove this formally in the following proof of Lemma \ref{lm:stationary_q}.

\begin{proof}[Proof of Lemma \ref{lm:stationary_q}]
Denote $h_{\varepsilon}^{\min}$ and $h_{\varepsilon}^{\max}$ as the minimum and maximum of the bounded function $h_{\varepsilon}(\cdot)$. Then, Assumption \ref{ass:q_suff_stat} implies that the number of rounds needed to converge to an error tolerance $\varepsilon$ under any sequence of compression parameters is bounded between $h_{\varepsilon}^{\min}$ and $h_{\varepsilon}^{\max}$.

For a positive $\nu$, let $\varepsilon$ be small enough such that $h_{\varepsilon}^{\min} > 2(1+\nu)/\nu$. First, we will consider network state sequences which are typical with repsect to $\nu$. Specifically, we consider a sequence $\left(\c^{\round}\right)_{\round}$ such that $\c^r$ belongs to  $\mathcal{T}_{\nu}^{r}(p)$ for every $r > h_{\varepsilon}^{\min}$. 

Due to Proposition \ref{prop:delay_convexity}, there exists a state-dependent stationary policy that optimizes the wall clock time to reach error tolerance $\varepsilon$ with respect to the sequence $(\c^{\round})_{\round}$. Let $\pifunc'$ represent this policy, and $r_{\varepsilon}^{\pifunc'}$ be the minimum number of rounds taken to converge by $\pifunc'$. Then, the wall clock time for $\pifunc'$ can be lower bounded as,
\begin{align}
    \sum_{\round = 1}^{r_{\varepsilon}^{\pifunc'}} d\left(\tau, \pifunc'\left(\c^{\round}\right), \c^{\round}\right)
        &=  \left(\frac{1}{r_{\varepsilon}^{\pifunc'}} \sum_{\round = 1}^{r_{\varepsilon}^{\pifunc'}} d\left(\tau, \pifunc'\left(\c^{\round}\right), \c^{\round}\right)\right)r_{\varepsilon}^{\pifunc'}, \nonumber \\
        &\stackrel{(a)}{\geq} (1-\nu)   \EXP\left[ d(\tau, \pifunc'(\C), \C) \right] r_{\varepsilon}^{\pifunc'}, \nonumber \\
        &\stackrel{(b)}{\geq} (1-\nu)   \EXP\left[ d(\tau, \pifunc'(\C), \C) \right] \left( \frac{1}{r_{\varepsilon}^{\pifunc'}} \sum_{\round=1}^{r_{\varepsilon}^{\pifunc'}} \norm{\boldsymbol{h}_\varepsilon \left( \pifunc\left(\c^{\round}\right)\right)}  \right), \nonumber\\
        &\stackrel{(c)}{\geq} (1-\nu)^2 \EXP\left[ d(\tau, \pifunc'(\C), \C) \right] \EXP\left[ \norm{\boldsymbol{h}_{\varepsilon}\left(\pifunc'(\C)\right)} \right], \nonumber\\
        &\stackrel{(d)}{\geq} (1-\nu)^2 \EXP\left[ d(\tau, \pifunc^*(\C), \C) \right] \EXP\left[ \norm{\boldsymbol{h}_{\varepsilon}\left(\pifunc^*(\C)\right)} \right]. \label{eq:lower_bound}
\end{align}
(a) and (c) follow from Lemma \ref{lm:typical_averaging_lemma}, (b) follows from Assumption \ref{ass:q_suff_stat}, and (d) follows by the following definition,
\[
\pifunc^* = \underset{\pifunc \in \Pifunc}{\text{arg min}} \quad \EXP\left[ d(\tau, \pifunc(\C)), \C) \right] \EXP\left[ \norm{\boldsymbol{h}_{\varepsilon}\left(\pifunc(\C)\right)} \right].
\]

Performing a similar calculation for $\pifunc^*$,
\begin{align}
    \sum_{\round = 1}^{r_{\varepsilon}^{\pifunc^*}} d\left(\tau, \pifunc^*\left(\c^{\round}\right), \c^{\round}\right)
        &=  \left(\frac{1}{r_{\varepsilon}^{\pifunc^*}} \sum_{\round = 1}^{r_{\varepsilon}^{\pifunc^*}} d\left(\tau, \pifunc^*\left(\c^{\round}\right), \c^{\round}\right)\right)r_{\varepsilon}^{\pifunc^*}, \nonumber \\
        &\stackrel{(a)}{\leq} (1+\nu)   \EXP\left[ d(\tau, \pifunc^*(\C), \C) \right] r_{\varepsilon}^{\pifunc^*}, \nonumber \\
        &\stackrel{(b)}{\leq} (1+\nu)^2   \EXP\left[ d(\tau, \pifunc^*(\C), \C) \right] \left( \frac{1}{r_{\varepsilon}^{\pifunc^*}} \sum_{\round=1}^{r_{\varepsilon}^{\pifunc^*}} \norm{\boldsymbol{h}_\varepsilon \left( \pifunc^*\left(\c^{\round}\right)\right)}  \right), \nonumber \\
        &\stackrel{(c)}{\leq} (1+\nu)^3 \EXP\left[ d(\tau, \pifunc^*(\C), \C) \right] \EXP\left[ \norm{\boldsymbol{h}_{\varepsilon}\left(\pifunc^*(\C)\right)} \right]. \label{eq:upper_bound}
\end{align}
(a) and (c) follow from Lemma \ref{lm:typical_averaging_lemma} and (b) follows from Proposition \ref{prop:r_ub} proved at the end of this section. 

The expected wall clock time to reach error tolerance $\varepsilon$ under the state-dependent stationary policy $\pifunc^*$ can be upper bounded as,
\begin{align}
    \EXP\left[T_{\varepsilon}^{\pifunc^*}\right] &\stackrel{(a)}{\leq} 
    P\left(
        \C^{R_{\varepsilon}^{\pifunc^*}} \in \mathcal{T}_{\nu}^{R_{\varepsilon}^{\pifunc^*}}(\mu) 
    \right) 
    (1+\nu)^3 \EXP\left[ d(\tau, \pifunc^*(\C), \C) \right] \EXP\left[ \norm{\boldsymbol{h}_{\varepsilon}\left(\pifunc^*(\C)\right)} \right] 
        + P\left(
        \C^{R_{\varepsilon}^{\pifunc^*}} \not{\in} \mathcal{T}_{\nu}^{R_{\varepsilon}^{\pifunc^*}}(\mu) 
        \right) 
        h_{\varepsilon}^{\max}d^{\max}, \nonumber \\
    &\stackrel{(b)}{\leq} (1+\nu)^3 \EXP\left[ d(\tau, \pifunc^*(\C), \C) \right] \EXP\left[ \norm{\boldsymbol{h}_{\varepsilon}\left(\pifunc^*(\C)\right)} \right] + P\left(\exists r' \geq h_{\varepsilon}^{\min}, \; \C^{r'} \not{\in} \mathcal{T}_{\nu}^{r'}(\mu) \right) h_{\varepsilon}^{\max}d^{\max}, \nonumber \\
    &\stackrel{(c)}{\leq} (1+\nu)^3 \EXP\left[ d(\tau, \pifunc^*(\C), \C) \right] \EXP\left[ \norm{\boldsymbol{h}_{\varepsilon}\left(\pifunc^*(\C)\right)} \right] +  \kappa_1\exp\left( -\kappa_2\nu^2 h_{\varepsilon}^{\min} \right) h_{\varepsilon}^{\max}d^{\max}, \label{eq:et_ub}
\end{align}
where, (a) follows by using \eqref{eq:et_ub} for typical sequences, and upper bounding $R^{\pifunc^*}_{\varepsilon}$ by $h_{\varepsilon}^{\max}$ and the round duration by $d^{\max}$ for non-typical sequences. (b) follows by upper bounding $P\left(\C^{R_{\varepsilon}^{\pifunc^*}} \in \mathcal{T}_{\nu}^{R_{\varepsilon}^{\pifunc^*}}(\mu) \right)$ by 1, and upper bounding $P\left(\C^{R_{\varepsilon}^{\pifunc^*}} \not{\in} \mathcal{T}_{\nu}^{R_{\varepsilon}^{\pifunc^*}}(\mu) \right)$ by $P\left(\exists r' \geq h_{\varepsilon}^{\min}, \; \C^{r'} \not{\in} \mathcal{T}_{\nu}^{r'}(\mu) \right)$ because, almost surely, $R^{\pifunc^{*}}_{\varepsilon} \geq h_{\varepsilon}^{\min}$. (c) follows from Proposition \ref{prop:ergodicity}.

 Let $T_{\varepsilon}^*$ be the random variable representing the wall-clock time to reach error tolerance $\varepsilon$ when one uses the optimal sample-path sequence dependent policy on random network-state sequence $(\C^{\round})_{\round}$.  Then, denoting $R_{\varepsilon}^*$ as the corresponding random variable denoting the number of rounds needed to reach error tolerance $\varepsilon$, we have,
\begin{align}
    \EXP\left[T_{\varepsilon}^{*}\right] &\stackrel{(a)}{\geq} P\left(\C^{R_{\varepsilon}^{*}} \in \mathcal{T}_{\nu}^{R_{\varepsilon}^{*}}(\mu) \right) (1-\nu)^2 \EXP\left[ d(\tau, \pifunc^*(\C), \C) \right] \EXP\left[ \norm{\boldsymbol{h}_{\varepsilon}\left(\pifunc^*(\C)\right)} \right], \nonumber \\
    &\stackrel{(b)}{\geq} P\left(\forall r' \geq h_{\varepsilon}^{\min}, \; \C^{r'} \in \mathcal{T}_{\nu}^{r'}(\mu) \right) (1-\nu)^2 \EXP\left[ d(\tau, \pifunc^*(\C), \C) \right] \EXP\left[ \norm{\boldsymbol{h}_{\varepsilon}\left(\pifunc^*(\C)\right)} \right], \nonumber \\
    &\stackrel{(c)}{\geq} (1-\nu)^2\left(1 - \kappa_1\exp\left(-\kappa_2\nu^2 h_{\varepsilon}^{\min}\right)\right) \EXP\left[ d(\tau, \pifunc^*(\C), \C) \right] \EXP\left[ \norm{\boldsymbol{h}_{\varepsilon}\left(\pifunc^*(\C)\right)} \right]. \label{eq:et_lb}
\end{align}
(a) follows due to \eqref{eq:lower_bound}, and (b) follows since, almost surely, $R_{\varepsilon}^{*} \geq h_{\varepsilon}^{\min}$. (c) follows from Proposition \ref{prop:ergodicity}.

Due to Assumption \ref{ass:asymptotic_convergence}, both $h_{\varepsilon}^{\min}$ and $h_{\varepsilon}^{\max}$ are $\Theta(1/poly(\varepsilon))$. And, $\nu$ can be made as small as desired. Therefore, from \eqref{eq:et_ub} and \eqref{eq:et_lb}, 
\[
\EXP\left[T_{\varepsilon}^{\pifunc^*}\right] \to \EXP\left[T_{\varepsilon}^*\right] \quad \text{as} \quad \varepsilon \to 0.
\]
\end{proof}

\begin{prop}
Under Assumptions \ref{ass:q_suff_stat} and \ref{ass:asymptotic_convergence}, for $\nu>0$, if $\varepsilon$ is small enough such that $h_{\varepsilon}^{\min} > 2(1+\nu)/\nu$, then, for any state-dependent stationary policy $\pifunc$, the minimum number of rounds $r_{\varepsilon}^{\pifunc}$ to reach an error tolerance $\varepsilon$ is such that,
\[
r_{\varepsilon}^{\pifunc} \leq (1+\nu) \frac{1}{r_{\varepsilon}^{\pifunc}} \sum_{\round=1}^{r_{\varepsilon}^{\pifunc}} \norm{\boldsymbol{h}_{\varepsilon}\left(\q^{\round} \right)}.
\]
\label{prop:r_ub}
\end{prop}
\begin{proof}
Due to Assumption \ref{ass:q_suff_stat}, $r_{\varepsilon}^{\pifunc}$ satisfies,
\[
r_{\varepsilon}^{\pifunc} > \frac{1}{r_{\varepsilon}^{\pifunc}} \sum_{\round=1}^{r_{\varepsilon}^{\pifunc}} \norm{\boldsymbol{h}_{\varepsilon}\left(\q^{\round} \right)}.
\]
Moreover, since $r_{\varepsilon}^{\pifunc}$ is the earliest round at which error tolerance $\varepsilon$ is reached, due to Assumption \ref{ass:q_suff_stat}, round $r_{\varepsilon}^{\pifunc} - 1$ satisfies,
\begin{align}
    r_{\varepsilon}^{\pifunc} - 1 &\leq \frac{1}{r_{\varepsilon}^{\pifunc}-1} \sum_{\round=1}^{r_{\varepsilon}^{\pifunc}-1} \norm{\boldsymbol{h}_{\varepsilon}\left(\q^{\round} \right)}, \nonumber \\
    &\leq \frac{1}{r_{\varepsilon}^{\pifunc}-1} \sum_{\round=1}^{r_{\varepsilon}^{\pifunc}} \norm{\boldsymbol{h}_{\varepsilon}\left(\q^{\round} \right)}, \nonumber \\
    &= \left( \frac{r_{\varepsilon}^{\pifunc}}{r_{\varepsilon}^{\pifunc}-1} \right)\frac{1}{r_{\varepsilon}^{\pifunc}} \sum_{\round=1}^{r_{\varepsilon}^{\pifunc}} \norm{\boldsymbol{h}_{\varepsilon}\left(\q^{\round} \right)}, \nonumber \\
    \implies r_{\varepsilon}^{\pifunc} &\leq \left( \frac{r_{\varepsilon}^{\pifunc}}{r_{\varepsilon}^{\pifunc}-1} \right)^2 \frac{1}{r_{\varepsilon}^{\pifunc}} \sum_{\round=1}^{r_{\varepsilon}^{\pifunc}} \norm{\boldsymbol{h}_{\varepsilon}\left(\q^{\round} \right)}. \label{eq:r_ub_1}
\end{align}
Now, we upper bound $r_{\varepsilon}^{\pifunc}/(r_{\varepsilon}^{\pifunc}-1)$,
\begin{align}
    \frac{r_{\varepsilon}^{\pifunc}}{r_{\varepsilon}^{\pifunc}-1} &= 1 + \frac{1}{r_{\varepsilon}^{\pifunc}-1}, \nonumber \\
    &\stackrel{(a)}{<} 1 + \frac{\nu}{2+\nu}, \nonumber \\
    &= \frac{1+\nu}{1+\nu/2}, \nonumber \\
    &\stackrel{(b)}{\leq} \sqrt{1+\nu}. \label{eq:r_ub_2}
\end{align}
(a) follows by rearranging the assumption $r_{\varepsilon}^{\pifunc} > 2(1+\nu)/\nu$ to obtain $r_{\varepsilon}^{\pifunc}- 1 > (2+\nu)/\nu$. (b) follows since $1+\nu/2 \geq \sqrt{1+\nu}$ for $\nu > 0$. Substituting \eqref{eq:r_ub_2} in \eqref{eq:r_ub_1} completes the result.
\end{proof}

In order to prove Proposition \ref{prop:ergodicity}, we use Theorem 3.1 from \cite{chung2012chernoff} which we re-state below for clarity.
\begin{theorem}
Let Assumption \ref{ass:markov} hold. Define $r_{\text{mix}}$ as the $1/8$ mixing time\footnote{Denoting $M$ as the transition-matrix of the Markov chain, and $\mu$ as its stationary distribution, $r_{\text{mix}} \triangleq \max_{\psi} \left\{r: \norm{M^r\psi - \mu}_{TV} \leq 1/8 \right\}$, where $\norm{\cdot}_{TV}$ denotes the TV-norm. Due to Theorem 4.9 of \cite{levin2017markov}, $r_{\text{mix}}$ is finite for an aperiodic and irreducible Markov chain.} of the Markov chain $(\C^{\round})_{\round}$ and $f: \mathcal{C} \to [0,1]$ be a function. Let, $\mu_{f} \triangleq \EXP\left[ f(\C^{(1)}) \right]$. Then, there exists a constant $\kappa_{\c}$ such that for every $0\leq \nu \leq 1$ and $r \in \mathbb{N}$,
\[
P\left( \left\lvert \frac{1}{r} \sum_{\round = 1}^{r} f(\C^{\round}) - \mu_{f}  \right\rvert \geq \nu \mu_{f} \right) \leq \kappa_{\c} \exp\left(-\frac{\nu^2\mu_f r}{72r_{\text{mix}}}\right).
\]
\label{th:markov_chain_chernoff}
\end{theorem}

\begin{proof}[Proof of Proposition \ref{prop:ergodicity}]
For some $\c \in \mathcal{C}$, define $f(\c') = \mathbbm{1}(\c'=\c)$. Then, due to Theorem \ref{th:markov_chain_chernoff}, there exist a constant $\kappa_{\c}$ such that,
\begin{equation}
P\left( \left\lvert \frac{1}{r'} \sum_{\round = 1}^{r'} \mathbbm{1}(\C^{\round}=\c) - \mu(\c)  \right\rvert \geq \nu \mu(\c) \right) \leq \kappa_{\c} \exp\left(-\frac{\nu^2\mu(\c) r'}{72r_{\text{mix}}}\right).
\label{eq:prop_ergod_1}
\end{equation}
Denote, $\kappa \triangleq \sum_{\c \in \mathcal{C}} \kappa_{\c}$ and $\mu_{\min} = \min_{\c \in \mathcal{C}} \mu(\c)$. Since the Markov chain is irreducible, $\mu_{\min} > 0$. Then, using \eqref{eq:prop_ergod_1} and taking a union bound over all $\c \in \mathcal{C}$, we obtain,
\[
P\left(\C^{r'} \not{\in} \mathcal{T}_{\nu}^{r'}(\mu) \right) \leq \kappa \exp\left(-\frac{\nu^2\mu_{\min} r'}{72r_{\text{mix}}}\right).
\]
Define $\kappa_2 \triangleq \mu_{\min}/(72r_{\textnormal{mix}})$. Taking a further union bound over all $r' \geq r$, 
\[
P\left(\exists r' \geq r \text{ such that } \C^{r'} \not{\in} \mathcal{T}_{\nu}^{r'}(\mu) \right) \leq \frac{\kappa}{1 - \exp\left( -\kappa_2\nu^2 \right)}  \exp\left( -\kappa_2\nu^2r \right)
\]
Defining $\kappa_1 \triangleq \kappa/(1-\exp(-\kappa_2\nu^2))$ completes the proof.
\end{proof}

\newpage
\section{Proof of Lemma \ref{lm:exp_time}}
\label{app:lemma_2_proof}

Here we prove Lemma \ref{lm:exp_time} which states that the expected wall clock is approximately equal to the product of the expected number of rounds and the expected round duration asymptotically.

The proof is very similar to the proof of Lemma \ref{lm:stationary_q} in Appendix \ref{app:proof_stationary_q}. As such, we will use the notation introduced in Appendix \ref{app:proof_stationary_q}.

Denote $h^{\min}_{\varepsilon}$ and $h^{\max}_{\varepsilon}$ as the minimum and maximum of the bounded function $h_{\varepsilon}(\cdot)$. And, let $d^{\min}$ and $d^{\max}$ denote the minimum and maximum of the positive, bounded function $d(\cdot, \cdot, \cdot)$. 

Similar to the proof of Lemma \ref{lm:stationary_q}, we consider network state sequences which are typical. In order to define the parameters for the typical set, for any given $\delta > 0$, we choose a small enough $\varepsilon_{th}>0$ and $\delta' > 0$ such that,
\begin{enumerate}
    \item $\varepsilon_{th}$ is small enough such that $h^{\min}_{\varepsilon_{th}} > 2(1+\delta')/\delta'$.
    \item $\delta'>0$ is such that for all $0<\varepsilon \leq \varepsilon_{th}$,
    \[
    (1-\delta) < (1-\delta')^2\left(1 - \kappa_1\exp\left(-\kappa_2({\delta'}^2 h^{\min}_{\varepsilon})\right)\right),
    \]
    and,
    \[
    \max \left\{ (1+\delta')^3, 1+\frac{\kappa_1\exp\left( -\kappa_2{\delta'}^2 h^{\min}_{\varepsilon} \right) h^{\max}_{\varepsilon}d^{\max}}{h^{\min}_{\varepsilon}d^{\min}} \right\}< (1+\delta/2).
    \]
    Such a choice of $\delta'$ is possible because $h^{\min}_{\varepsilon} = \Theta(1/poly(\varepsilon))$, and $\exp(-\kappa_2{\delta'}^2h_{\varepsilon}^{\min})h^{\max}_{\varepsilon}/h^{\min}_{\varepsilon}= \exp(-\Omega({\delta'}^{2}/poly(\varepsilon)))$.
\end{enumerate}

We consider a sequence $\left(\c^{\round}\right)_{\round}$ such that $\c^r \in \mathcal{T}_{\delta'}^{r}(\mu)$ for every $r \geq h^{\min}_{\varepsilon_{th}}$. 

For a policy $\pifunc$, let $r^{\pifunc}_{\varepsilon}$ denote the minimum number of rounds needed to converge to error tolerance $\varepsilon < \varepsilon_{th}$ given network state sequence $\left(\c^{\round}\right)_{\round}$. Then, the wall clock time for $\pifunc$ can be lower bounded as,
\begin{align}
    \sum_{\round = 1}^{r_{\varepsilon}^{\pifunc}} d\left(\tau, \pifunc\left(\c^{\round}\right), \c^{\round}\right)
        &=  \left(\frac{1}{r_{\varepsilon}^{\pifunc}} \sum_{\round = 1}^{r_{\varepsilon}^{\pifunc}} d\left(\tau, \pifunc\left(\c^{\round}\right), \c^{\round}\right)\right)r_{\varepsilon}^{\pifunc}, \nonumber \\
        &\stackrel{(a)}{\geq} (1-\delta')   \EXP\left[ d(\tau, \pifunc(\C), \C) \right] r_{\varepsilon}^{\pifunc}, \nonumber \\
        &\stackrel{(b)}{\geq} (1-\delta')   \EXP\left[ d(\tau, \pifunc(\C), \C) \right] \left( \frac{1}{r_{\varepsilon}^{\pifunc}} \sum_{\round=1}^{r_{\varepsilon}^{\pifunc}} \norm{\boldsymbol{h}_\varepsilon \left( \pifunc\left(\c^{\round}\right)\right)}  \right), \nonumber\\
        &\stackrel{(c)}{\geq} (1-\delta')^2 \EXP\left[ d(\tau, \pifunc(\C), \C) \right] \EXP\left[ \norm{\boldsymbol{h}_{\varepsilon}\left(\pifunc(\C)\right)} \right], \label{eq:lm2_lower_bound}
\end{align}
(a) and (c) follow from Lemma \ref{lm:typical_averaging_lemma} since $r^{\pifunc}_{\varepsilon} > h^{\min}_{\varepsilon_{th}}$, and (b) follows from Assumption \ref{ass:q_suff_stat}.

Performing a similar calculation for the upper bound,
\begin{align}
    \sum_{\round = 1}^{r_{\varepsilon}^{\pifunc}} d\left(\tau, \pifunc \left(\c^{\round}\right), \c^{\round}\right)
        &=  \left(\frac{1}{r_{\varepsilon}^{\pifunc}} \sum_{\round = 1}^{r_{\varepsilon}^{\pifunc}} d\left(\tau, \pifunc\left(\c^{\round}\right), \c^{\round}\right)\right)r_{\varepsilon}^{\pifunc}, \nonumber \\
        &\stackrel{(a)}{\leq} (1+\delta')   \EXP\left[ d(\tau, \pifunc(\C), \C) \right] r_{\varepsilon}^{\pifunc}, \nonumber \\
        &\stackrel{(b)}{\leq} (1+\delta')^2   \EXP\left[ d(\tau, \pifunc(\C), \C) \right] \left( \frac{1}{r_{\varepsilon}^{\pifunc}} \sum_{\round=1}^{r_{\varepsilon}^{\pifunc}} \norm{\boldsymbol{h}_\varepsilon \left( \pifunc\left(\c^{\round}\right)\right)}  \right), \nonumber \\
        &\stackrel{(c)}{\leq} (1+\delta')^3 \EXP\left[ d(\tau, \pifunc(\C), \C) \right] \EXP\left[ \norm{\boldsymbol{h}_{\varepsilon}\left(\pifunc(\C)\right)} \right] . \label{eq:lm_2_upper_bound}
\end{align}
(a) and (c) follow from Lemma \ref{lm:typical_averaging_lemma}, and (b) follows from Proposition \ref{prop:r_ub}.

Let $R^{\pifunc}_{\varepsilon}$ be the random variable denoting the minimum number number of rounds needed to converge to an error tolerance $\varepsilon<\varepsilon_{th}$ when policy $\pifunc$ is used on the sequence $(\C^{\round})_{\round}$. The expected wall clock times of $\pifunc$ can be lower bounded as,
\begin{align}
    \EXP\left[T_{\varepsilon}^{\pifunc}\right] &\geq P\left(\C^{R_{\varepsilon}^{\pifunc}} \in \mathcal{T}_{\delta'}^{R_{\varepsilon}^{\pifunc}}(\mu) \right) (1-\delta')^2 \EXP\left[ d(\tau, \pifunc(\C)), \C) \right] \EXP\left[ \norm{\boldsymbol{h}_{\varepsilon}\left(\pifunc(\C)\right)} \right], \nonumber \\
    &\stackrel{(a)}{\geq} P\left(\forall r' \geq h_{\varepsilon}^{\min}, \; \C^{r'} \in \mathcal{T}_{\nu}^{r'}(\mu) \right)(1-\delta')^2 \EXP\left[ d(\tau, \pifunc(\C), \C) \right] \EXP\left[ \norm{\boldsymbol{h}_{\varepsilon}\left(\pifunc(\C)\right)} \right] \nonumber \\
    &\stackrel{(b)}{\geq} (1-\delta')^2\left(1 - \kappa_1\exp\left(-\kappa_2{\delta'}^2 h^{\min}_{\varepsilon}\right)\right) \EXP\left[ d(\tau, \pifunc(\C), \C) \right] \EXP\left[ \norm{\boldsymbol{h}_{\varepsilon}\left(\pifunc(\C)\right)} \right] \nonumber \\
    &\stackrel{(c)}{\geq} (1-\delta) \EXP\left[ d(\tau, \pifunc(\C), \C) \right] \EXP\left[ \norm{\boldsymbol{h}_{\varepsilon}\left(\pifunc(\C)\right)} \right]. \label{eq:lm2_et_lb}
\end{align}
(a) follows since $R^{\pifunc}_{\varepsilon}\geq h^{\min}_{\varepsilon}$ almost surely, (b) follows from Proposition \ref{prop:ergodicity}  and (c) follows from the choice of $\delta'$.

The expected wall clock times of $\pifunc$ can be upper bounded as,
\begin{align}
    \EXP\left[T_{\varepsilon}^{\pifunc}\right] 
        &\stackrel{(a)}{\leq} P\left(
            \C^{R_{\varepsilon}^{\pifunc}} \in \mathcal{T}_{\delta'}^{R_{\varepsilon}^{\pifunc}}(\mu) 
            \right) 
            (1+\delta')^3 \EXP\left[ d(\tau, \pifunc(\C), \C) \right] \EXP\left[ \norm{\boldsymbol{h}_{\varepsilon}\left(\pifunc(\C)\right)} \right] 
            + P\left(
                \C^{R_{\varepsilon}^{\pifunc}} \not{\in} \mathcal{T}_{\delta'}^{R_{\varepsilon}^{\pifunc}}(\mu) 
            \right) h^{\max}_{\varepsilon}d^{\max}, \nonumber \\
            &\stackrel{(b)}{\leq}  
            (1+\delta')^3 \EXP\left[ d(\tau, \pifunc(\C), \C) \right] \EXP\left[ \norm{\boldsymbol{h}_{\varepsilon}\left(\pifunc(\C)\right)} \right] 
            + P\left(
                \exists r' \geq h_{\varepsilon}^{\min} \C^{r'} \not{\in} \mathcal{T}_{\delta'}^{r'}(\mu)
            \right) h^{\max}_{\varepsilon}d^{\max}, \nonumber \\
        &\stackrel{(c)}{\leq} (1+\delta')^3 \EXP\left[ d(\tau, \pifunc(\C), \C) \right] \EXP\left[ \norm{\boldsymbol{h}_{\varepsilon}\left(\pifunc(\C)\right)} \right] 
        +  \kappa_1\exp\left( -\kappa_2{\delta'}^2 h^{\min}_{\varepsilon} \right) h^{\max}_{\varepsilon}d^{\max}, \nonumber \\
    &\stackrel{(d)}{\leq} (1+\delta) \EXP\left[ d(\tau, \pifunc(\C), \C) \right] \EXP\left[ \norm{\boldsymbol{h}_{\varepsilon}\left(\pifunc(\C)\right)} \right]. \label{eq:lm2_et_ub}
\end{align}
(a) follows by using \eqref{eq:lm_2_upper_bound} to upper bound the wall clock time for typical sequences, and using the worst-case upper bound $h_{\varepsilon}^{\max}d^{\max}$ for non-typical sequences. (b) follows because, almost surely, $R_{\varepsilon}^{\pifunc} \geq h_{\varepsilon}^{\min}$. (c) follows from Proposition \ref{prop:ergodicity}.   (d) follows from the choice of $\delta'$. 

\eqref{eq:lm2_et_lb} and \eqref{eq:lm2_et_ub}, jointly, conclude the proof.

\newpage
\section{Proof of Proposition \ref{prop:unique_stationary}}
\label{app:theorem_1_proof}

In this section, we prove Proposition \ref{prop:unique_stationary} which states that the Fluid-Frank-Wolfe (FFW) process has a unique stationary point in the set $\conv(V_{\varepsilon})$. It further states that the stationary point lies in the set $V_{\varepsilon}$. In order to demonstrate the main arguments of the proof, we will first consider the case with a single client ($m=1$) and a single network state, $\mathcal{C} = \{c\}$. Later, in Subsection B, we will generalize these arguments to complete the proof of Proposition \ref{prop:unique_stationary}.

\subsection{Warmup: 1 Client and 1 Network State (1C1NS) Case}
\label{app:1_user_1_state}
Recall that the set $V_{\varepsilon}$ is the set of pairs of achievable expected number of rounds $\hat{r}_{\varepsilon}$ and expected round duration $\hat{d}$ of state-dependent stationary policy. Here, we observe that $\conv(V_{\varepsilon})$ may be interpreted as the corresponding feasibility set for (possibly random) stationary policies. We start by making this observation precise. In the 1C1NS case, a stationary policy may be represented by a one-dimensional random variable $\Pi$ that denotes the possibly randomly selected compression parameter. The space of possible stationary policies is denoted by the set of distributions $\mathcal{Q}_1$,
\[
\mathcal{Q}_1 = \{f_\Pi: \text{ $f_\Pi$ is a distribution over }[0, q_{\max}] \}.
\]

Under a policy corresponding to $\Pi$ and a small enough error tolerance $\varepsilon$, due to Lemma \ref{lm:exp_time}, the expression for the expected wall clock time for stationary policies is given by,
\[
\EXP[T_{\varepsilon}^{\Pi}] \: \approx \: \hat{t}_{\varepsilon}^{\Pi} \: = \: \EXP\left[ h_\varepsilon(\Pi) \right] \EXP\left[ d(\tau, \Pi, c) \right],
\]
where $T_{\varepsilon}^\Pi$ is the wall clock time to reach error tolerance $\varepsilon$ under compression parameter $\Pi$. Then, the feasible set $\conv(V_{\varepsilon})$ for stationary policies may be written for this case as,
\[
\conv(V_{\varepsilon}) = \left\{ (\hat{r}_\varepsilon, \hat{d}): \exists f_{\Pi} \in \mathcal{Q}_1 \text{ s.t., } \Pi \sim f_{\Pi} \text{ satisfies }  \hat{r}_{\varepsilon} = \EXP\left[ h_{\varepsilon}(\Pi) \right], \quad \hat{d} = \EXP\left[ d(\tau, \Pi, c) \right] \right\}.
\]

Here, stationary policies may be separated into two categories,
\begin{itemize} 
    \item deterministic policies: here, $\Pi$ is a constant.
    \item stochastic policies: here, $\Pi$ is non-deterministic random variable.
\end{itemize}

Our aim is to prove that there exists a unique fixed-point of the FFW update in $\conv(V_{\varepsilon})$. As we will see, this will prove Proposition \ref{prop:unique_stationary} for the special case of 1 client and 1 network state. Before delving into the proof of this statement, we state a couple of properties of $\conv(V_{\varepsilon})$ which are useful in proving the existence of a unique fixed-point of the FFW update. 

\begin{prop}
For any $h_{\varepsilon}^{\min}\leq h \leq h_{\varepsilon}^{\max}$, there exists a deterministic policy that minimizes $\hat{t}_{\varepsilon}^{\Pi}$ given a constraint $\EXP[h_{\varepsilon}(\Pi)] = h$.
\label{prop:constr_det_pol_opt}
\end{prop}
\begin{proof}
As a contradiction, assume that there is no deterministic policy that minimizes the expected wall clock time given a constraint $\EXP\left[h_{\varepsilon}(\Pi)\right] = h$. Let $\Pi^*$ be a stochastic policy that minimizes the expected wall clock time with the constraint $\EXP[h_{\varepsilon}(\Pi)]=h$. Then, consider an alternate policy with deterministic compression parameter $\pi$ chosen as, $h_{\varepsilon}(\pi) = \EXP\left[ h_{\varepsilon}(\Pi^*) \right]$. Such a $\pi$ exists due to the Intermediate Value Theorem since $h_{\varepsilon}(\cdot)$ is a continuous function. In this case, by the strict convexity of the duration function assumed in Assumption \ref{ass:comm_delay}, we have,
\[
\EXP\left[ d(\tau, \pi, c) \right] < \EXP\left[d\left( \tau, \Pi^*, c \right) \right].
\]
Therefore, the relation between their expected wall clock times is,
\[
\hat{t}_{\varepsilon}^{\pi} = \EXP\left[h_{\varepsilon}(\pi)\right]\EXP\left[ d(\tau, \pi, c) \right] < \EXP\left[ h_{\varepsilon}(\Pi^*) \right]\EXP\left[d\left( \tau, \Pi^*, c \right) \right] = \hat{t}_{\varepsilon}^{\Pi^*}.
\]
This is a contradiction to the assumption that a stochastic policy minimizes the expected wall clock time given the constraint.
\end{proof}

\begin{figure}[h]
    \centering
    \includegraphics[width=.4\textwidth]{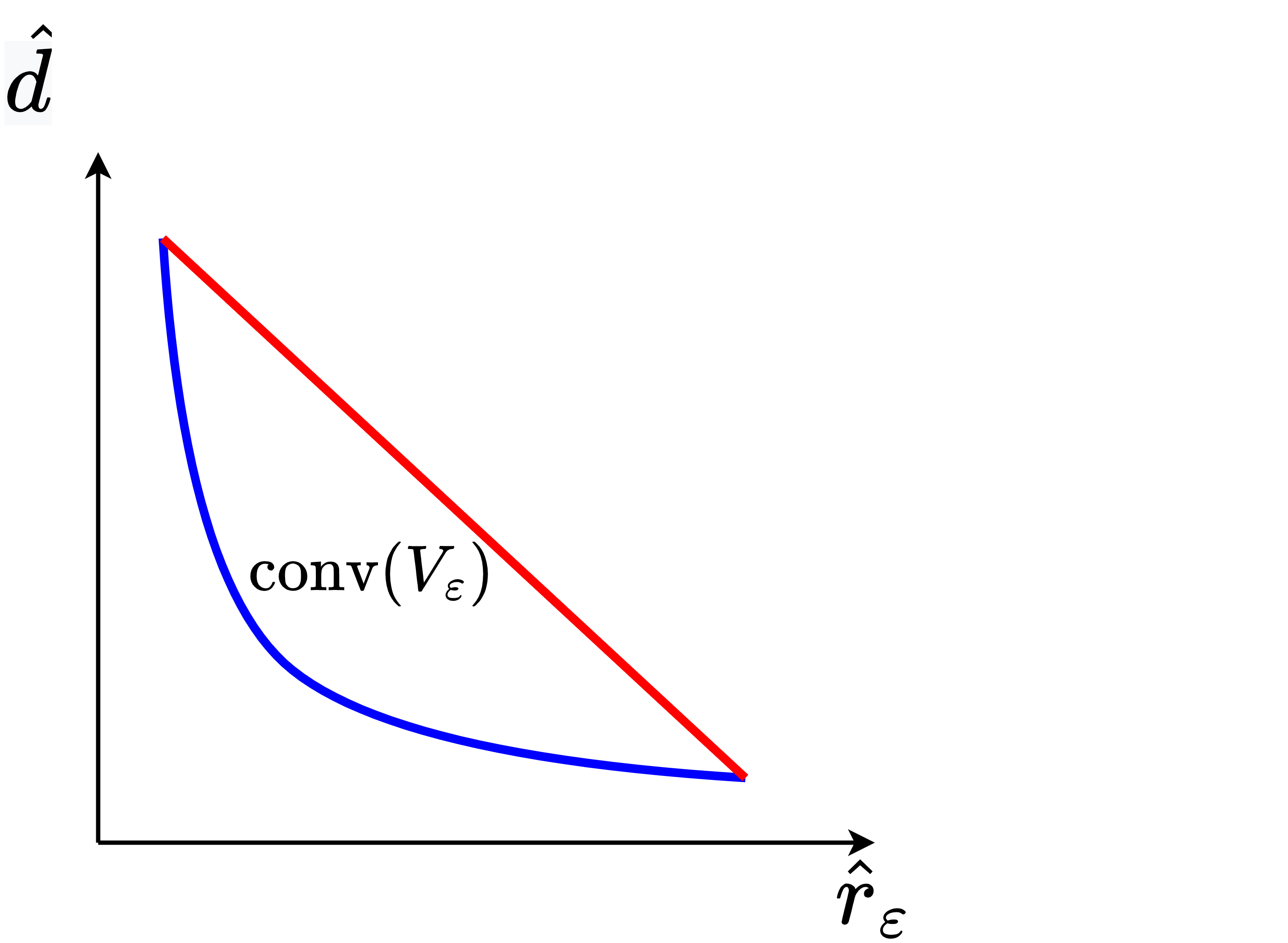}
    \caption{Feasibility set for stationary policies $\conv(V_{\varepsilon})$. In the 1C1NS case, a point $(\hat{r}_{\varepsilon}, \hat{d}) \in \conv(V_{\varepsilon})$ corresponds to a policy $\Pi$ such that, $\hat{r}_{\varepsilon}= \EXP[h_{\varepsilon}(\Pi)]$ and $\hat{d} = \EXP[d(\tau, \Pi, \c)]$. The function $\bar{d}(\hat{r}_{\varepsilon})$ is represented by the blue curve.}
    \label{fig:feasibility_set}
\end{figure}

For notational brevity, denote,
\[
\bar{d}(\hat{r}_{\varepsilon}) = d(\tau, h^{-1}_{\varepsilon}(\hat{r}_{\varepsilon}), c).
\]
Recall that we may denote points $(\hat{r}_{\varepsilon}, \hat{d}) \in \conv(V_{\varepsilon})$ by a two-dimensional vector $\x = (\hat{r}_{\varepsilon} \; \hat{d})^{\top}$. Also, recall the function $H(\x)=x_1x_2$. The following proposition states several equivalent ways of describing a point $\x$ in the set $V_{\varepsilon}$.

\begin{prop}
The following statements are equivalent
\begin{enumerate}
    \item[I.] $\x \in \conv(V_{\varepsilon})$ is of the form $(\hat{r}_{\varepsilon}, \bar{d}(\hat{r}_{\varepsilon}))$ for some $h_{\varepsilon}^{\min} \leq \hat{r}_{\varepsilon} \leq h_{\varepsilon}^{\max}$.
    \item[II.]  $\x \in \conv(V_{\varepsilon})$ is such that $\alpha \x \not{\in} \conv(V_{\varepsilon})$ for any $0<\alpha<1$.
    \item[III.] $\x = (\hat{r}_{\varepsilon}, \hat{d}) \in \conv(V_{\varepsilon})$ is such that, $\hat{d} = \min \{d': (\hat{r}_{\varepsilon}, d') \in \conv(V_{\varepsilon}) \}$. 
    \item[IV.] $\x = (\hat{r}_{\varepsilon}, \hat{d}) \in \conv(V_{\varepsilon})$ is such that, $\hat{r}_{\varepsilon}\hat{d} = \min \{\hat{r}_{\varepsilon}d': (\hat{r}_{\varepsilon}, d') \in \conv(V_{\varepsilon}) \}$.
    \item[V.] $\x$ is an extreme point\footnote{A point $\x \in \conv(V_{\varepsilon})$ is called an \emph{extreme point} if it cannot be written as the convex combination of two other points in $\conv(V_{\varepsilon})$.} of $\conv(V_{\varepsilon})$.
\end{enumerate}
\label{prop:1c1ns_equiv}
\end{prop}
\begin{proof}[Proof Sketch]
The equivalences may be inferred from the structure of the feasibility set $\conv(V_{\varepsilon})$ as shown in Fig. \ref{fig:feasibility_set}. We briefly describe the arguments required to prove the equivalences. I $\iff$ III because $(\hat{r}_{\varepsilon}, \bar{d}_{\varepsilon})$ corresponds to a deterministic policy by definition, and, due to part b, a deterministic policy minimizes the wall clock time $\hat{r}_\varepsilon\hat{d}$ amongst the set of policies $\{f_{\Pi} \in \mathcal{Q}_1: \text{ s.t., } \Pi \sim f_{\Pi} \text{ satisfies } \EXP[h_{\varepsilon}(\Pi)] = \hat{r}_{\varepsilon} \}$. I, III $\iff$ II because $\bar{d}(\hat{r}_{\varepsilon})$ is a strictly decreasing function. I $\iff$ V because $\bar{d}$ is a strictly convex function. III and IV are trivially equivalent.
\end{proof}
From here on, we will call points of the form described in Proposition \ref{prop:1c1ns_equiv} as extreme points, and all other points in $\conv(V_{\varepsilon})$ as non-extreme points. Note that in the 1C1NS case, the set of extreme points is equal to the set $V_{\varepsilon}$. However, we refrain from using this fact here because in the general case of multiple clients and multiple network states, this is no longer true.

Next, we prove the existence of a unique fixed point of the FFW update in two steps. First, we show that a non-extreme point cannot be a fixed point of the FFW update. 
\begin{prop}
If a point $\x \in \conv(V_{\varepsilon})$ is not an extreme point of $\conv(V_{\varepsilon})$, then it is not a fixed-point of the FFW update.
\label{prop:non_ext_pt}
\end{prop}
\begin{proof}
Due Proposition \ref{prop:1c1ns_equiv}, $\x$ being a non-extreme point implies that there exists a constant $0 < \alpha < 1$ such that $\alpha\x \in \conv(V_{\varepsilon})$.

Observe that $\nabla H(\x) = \left(
x_2 \; 
x_1
\right)^\top$. Since, $h_{\varepsilon}()$ and $d()$ are non-negative functions, we have that $\nabla H(\x)$ has non-negative entries for any $\x \in V_{\varepsilon}$.

Recall that $\x$ is a fixed-point of the FFW update if and only if $\x = \underset{\x' \in \conv(V_{\varepsilon}) }{\text{arg min}} \; \nabla H(\x)^\top \x'$. Due to the elementwise non-negativity of $\nabla H$, $\nabla H(\x)^\top (\alpha \x) < \nabla H(\x)^\top \x$. Since $\alpha \x \in \conv(V_{\varepsilon})$, $\x$ is not a fixed point of the FFW-update.
\end{proof}

Due to Proposition \ref{prop:non_ext_pt}, we focus on only extreme points in the next step. Before proving the existence of a unique fixed point of the FFW update amongst the set of extreme points, we state a result about an equivalent description of a fixed point. For this purpose, define $\hat{t}(\hat{r}_{\varepsilon}) = \hat{r}_{\varepsilon}\bar{d}(\hat{r}_{\varepsilon})$.
\begin{prop}
Under Assumption \ref{ass:unique_stationary}, an extreme-point $\x = \left( \hat{r}_{\varepsilon}, \bar{d}(\hat{r}_{\varepsilon}) \right)^\top$ with $h_{\varepsilon}^{\min} < \hat{r}_{\varepsilon} < \hat{h}_{\varepsilon}^{\max}$ is a fixed point for the FFW-update if and only if $\hat{t}'(\hat{r}_{\varepsilon}) = 0$.
\label{prop:1c1ns_1}
\end{prop}
\begin{proof}
First, as a contradiction, assume that there exists an $\hat{r}_{\varepsilon}$ with $\hat{t}'(\hat{r}_{\varepsilon}) = 0$, but $\x = \left( \hat{r}_{\varepsilon} \; \bar{d}(\hat{r}_{\varepsilon}) \right)^\top$ is not a fixed point of the FFW update. Recall that $\x$ is a fixed point of the FFW update if $\x = \underset{\x' \in \conv(V_{\varepsilon}) }{\text{arg min}} \; \nabla H(\x)^\top \x'$.  Therefore, this implies that there exists a different point $\z = (\hat{r}' \; \bar{d}(\hat{r}'))^\top$ in $\conv(V_{\varepsilon})$ such that,
\[
\nabla H(\x)^\top (\z - \x) \leq 0.
\]
Refer to Fig. \ref{fig:theorem_1_proof_motivation} for an illustration of the following argument. Due to strict convexity of the curve $\bar{d}(\cdot)$, there exists a point $\w = (\hat{r} \; \bar{d}(\hat{r}'))^\top$ with $\hat{r}$ being in between $\hat{r}_{\varepsilon}$ and $\hat{r}'$ such that $\nabla H(\x)^\top (\w - \x) = -\xi < 0$.  Then, for any $0<\theta<1$, denote $\w_{\theta}^{c} = (1-\theta)\x + \theta \w$ and $\w_\theta = \left((1-\theta)\hat{r}_\varepsilon + \theta \hat{r}, \bar{d}((1-\theta)\hat{r}_\varepsilon + \theta \hat{r}) \right)$. It is easily verified that $\nabla H(\x)^\top (\w_{\theta}^{c}-\x) = -\xi \theta$. And, due the the convexity of $\bar{d}(\cdot)$, $\nabla H(\x)^\top (\w_{\theta}-\x) \leq \nabla H(\x)^\top (\w_{\theta}^{c}-\x)$. That is, 
\begin{equation}
\nabla H(\x)^\top (\w_\theta - \x) \leq - \xi \theta,
\label{eq:1c1ns_1}
\end{equation}
\begin{figure}
    \centering
    \includegraphics[width=.45\textwidth]{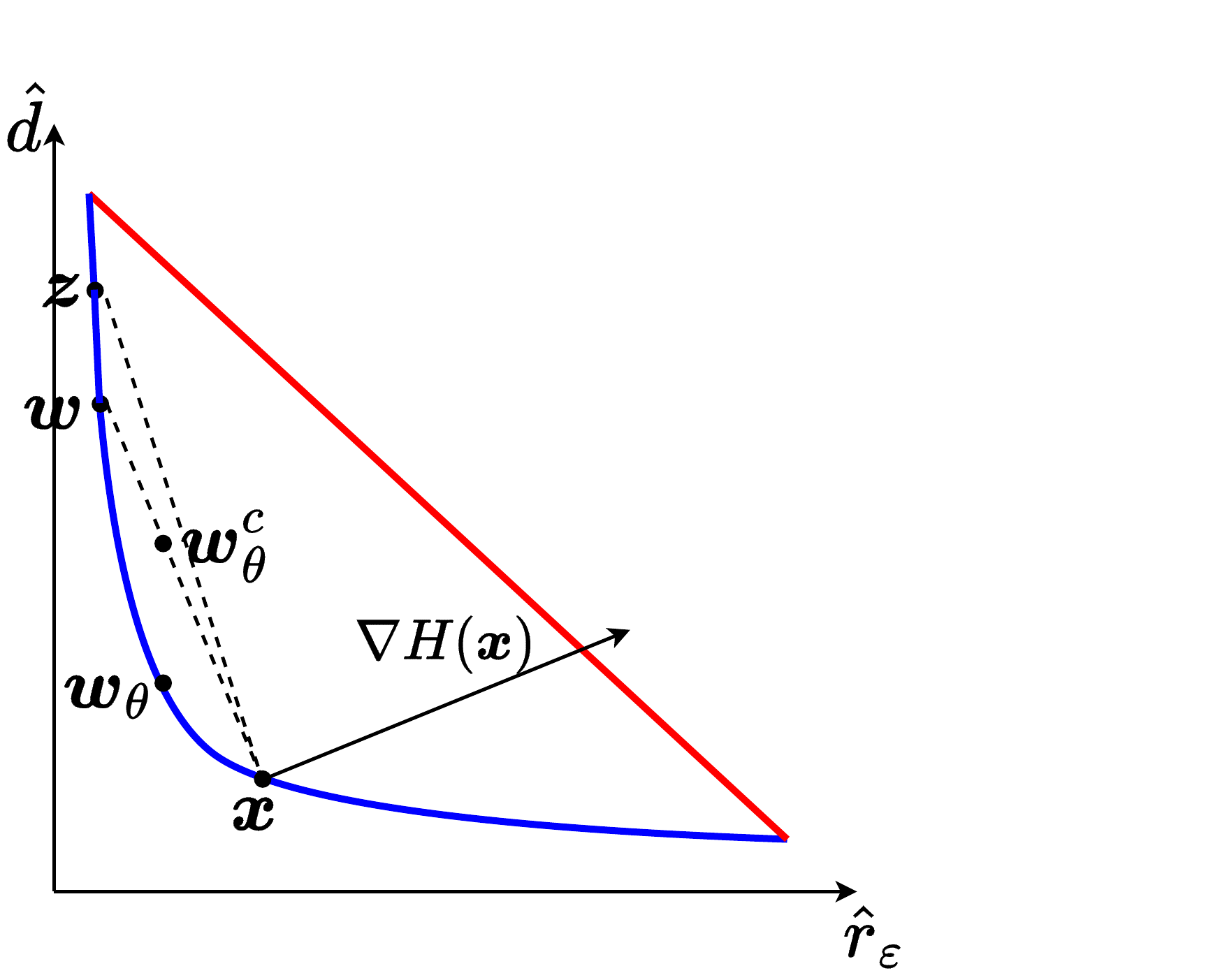}
    \caption{Illustration of the construction for a part of the proof of Proposition \ref{prop:1c1ns_1} and \ref{prop:mcmns_1}. Specifically, proof of the statement, $\hat{t}'(\hat{r}_{\varepsilon})=0$ implies a fixed point for the FW-update.}
    \label{fig:theorem_1_proof_motivation}
\end{figure}
By a Taylor Series expansion of $H$, we have $\hat{t}((1-\theta)\hat{r}_{\varepsilon}+\theta \hat{r} = \hat{t}(\hat{r}_\varepsilon)+ \nabla H(\x)^\top (\w_\theta - \x) + O(\theta^2)$. Therefore, for small enough $\theta$, \eqref{eq:1c1ns_1} implies that $\hat{t}((1-\theta)\hat{r}_{\varepsilon}+\theta \hat{r}) < \hat{t}(\hat{r}_{\varepsilon})$. This is a contradiction to Assumption \ref{ass:unique_stationary} which states that $\hat{t}''(\hat{r}_{\varepsilon})>0$.

Conversely, consider for contradiction that $\x = \left( \hat{r}_{\varepsilon}, \bar{d}(\hat{r}_{\varepsilon}) \right)^\top$ is a fixed point of the FFW update, but $\hat{t}'(\hat{r}_{\varepsilon}) \neq 0$. Define $\hat{r}^\delta = \hat{r}_{\varepsilon}+\delta$, and $\z^{\delta} = (\hat{r}^\delta, \bar{d}(\hat{r}^\delta))^\top$. Then, for small enough $\delta>0$, since $\hat{t}'(\hat{r}_{\varepsilon})\neq 0$, we have that 
\begin{equation}
    \begin{split}
        \hat{t}(\hat{r}^\delta) &< \hat{t}(\hat{r}_{\varepsilon}) - \Omega(\delta), \\
        &\text{(OR)} \\
        \hat{t}(\hat{r}^{-\delta}) &< \hat{t}(\hat{r}_{\varepsilon}) - \Omega(\delta).
    \end{split}
    \label{eq:1c1ns_2}
\end{equation}
Again, by Taylor series expansion, 
\begin{align}
    H(\z^{\delta}) &= H(\x) + \nabla H(\x)(\z^{\delta} - \x) + O(\delta^2), \nonumber \\
    \iff \hat{t}(\hat{r}^\delta) &= \hat{t}(\hat{r}_{\varepsilon}) + \nabla H(\x)(\z^{\delta} - \x) + O(\delta^2).
    \label{eq:1c1ns_3}
\end{align}
Comparing \eqref{eq:1c1ns_2} and \eqref{eq:1c1ns_3}, we have that $\nabla H(\x)(\z^{\delta} - \x)<0$ or $\nabla H(\x)(\z^{-\delta} - \x)<0$. This is a contradiction to the premise that $\x$ is a fixed point for the FFW update.
\end{proof}

\begin{remark}
$\x = \left(h_{\varepsilon}^{\min}, \bar{d}(h_{\varepsilon}^{\min})\right)$ is a fixed point of the FFW update if $t'(h_{\varepsilon}^{\min}) \geq 0$. Similarly, $\x = \left(h_{\varepsilon}^{\max}, \bar{d}(h_{\varepsilon}^{\max})\right)$ is a fixed point of the FFW update if $t'(h_{\varepsilon}^{\max}) \leq 0$. The proof of these statements is very similar to that of Proposition \ref{prop:1c1ns_equiv} We don't prove these statements for the 1C1NS case for succinctness. But, we prove it rigorously in the multiple client and multiple network state case. 
\label{rm:endpoints_fp}
\end{remark}

\begin{prop}
For the 1C1NS case, there exists a unique fixed point $\x$ of the FFW update in $\conv(V_{\varepsilon})$.
\end{prop}
\begin{proof}
Due to Proposition \ref{prop:non_ext_pt}, a non-extreme point of $\conv(V_{\varepsilon})$ cannot be a fixed point of the FFW update.

Proposition \ref{prop:1c1ns_1} and Remark \ref{rm:endpoints_fp} jointly characterize the fixed point of the FFW update amongst the set of extreme points in terms of $\hat{t}(\cdot)$. In particular, they imply that a point $\x=(\hat{r}_{\varepsilon}, \bar{d}(\hat{r}_{\varepsilon}))$ is a fixed-point if and only if $\hat{r}_{\varepsilon}$ is a local-minimum of $\hat{t}(\cdot)$ in the domain $h_{\varepsilon}^{\min} \leq \hat{r}_{\varepsilon} \leq h_{\varepsilon}^{\max}$ (since $\hat{t}(\cdot)$ is quasiconvex, it does not have any local maximum in the interior of its domain). Since $\hat{t}$ is a strictly quasiconvex function (Assumption \ref{ass:unique_stationary}) on a bounded domain, it has a unique local minimum. Therefore, the FFW update has a unique fixed point among the set of extreme points of $\conv(V_{\varepsilon})$. 
\end{proof}

\subsection{Multiple Clients and Multiple Network States (MCMNS) Case}
This section contains the complete proof of Proposition \ref{prop:unique_stationary}. The argument is a generalization of what we saw in the previous subsection to the case with multiple clients and multiple network states. We start by introducing notation to describe policies in this general setting.


We may denote the policy $\pifunc$ as a function, $\pifunc:\mathcal{C} \to [0,q_{\max}]^{ m}$, or as a vector $\pivec$ of dimension $m\lvert \mathcal{C} \rvert$. In specific, enumerating the elements of $\mathcal{C}$ as, $\mathcal{C} = \{ \c_1, \c_2, \dots, \c_{\lvert \mathcal{C}\rvert} \}$, the vector $\pivec$ is represented as,
\[
\pivec = \left(\pifunc_1(\c_1), \dots, \pifunc_m(\c_1), \pifunc_1(\c_2), \dots, \pifunc_m(\c_2), \cdots\cdots\cdots, \pifunc_1(\c_{\lvert \mathcal{C} \rvert}), \dots, \pifunc_m(\c_{\lvert \mathcal{C} \rvert})  \right)^\top,
\]
where $\pifunc_i(\c)$ indicates the $i$\textsuperscript{th} entry of $m$-dimensional vector $\pifunc(\c)$.

Next, we explain how to define the quantities, $\EXP\left[\norm{\h_{\varepsilon}(\pifunc(\C))} \right]$ and $\EXP\left[ d\left(\tau, \pifunc(\C), \C \right) \right]$ in terms of the vector representation of $\pifunc$. Denote $\boldsymbol{e}^{i}$ as an $m\lvert \mathcal{C} \rvert$ dimensional vector with,
\[
e^{i}_j = \begin{cases}
1 & \text{, if } (i-1)\lvert \mathcal{C} \rvert \leq j < i \lvert \mathcal{C} \rvert, \\
0 & \text{, otherwise.}
\end{cases}
\]
Recall that $\mu$ denotes the stationary distribution of the Markov chain of network states. Then, for a deterministic policy $\pifunc$, we may write,
\begin{align}
    \tilde{r}_{\varepsilon}(\pivec) &\triangleq \EXP\left[ \norm{\h_{\varepsilon}(\pifunc(\C))}  \right] = \sum_{i=1}^{\lvert \mathcal{C} \rvert} \mu(\c_i)\norm{\h_\varepsilon(\pivec) \odot \boldsymbol{e}^{i}}, \label{eq:tilde_r} \\
    \tilde{d}(\pivec) &\triangleq \EXP\left[ d(\tau, \pifunc(\C), \C)  \right] = \sum_{i=1}^{\lvert \mathcal{C} \rvert} \mu(\c_i) d\left(\tau, \pivec_{(i-1)\lvert \mathcal{C} \rvert}^{i\lvert \mathcal{C}\rvert - 1}, \c_i\right), \label{eq:tilde_d}
\end{align}
where $\odot$ denotes the elementwise product of two vectors, and $\pivec_{j}^{k}$ denotes the vector $\left( \pivec_{j}, \pivec_{j+1}, \cdots, \pivec_{k} \right)^\top$.

Similar to the 1C1NS case, $\conv(V_{\varepsilon})$ may be interpreted as the set of achievable pairs of expected rounds and expected round duration of (possibly, stochastic) stationary policies. Therefore, a stationary policy may be represented by a random vector, $\Pivec$. Then, we will denote the set of all (deterministic and stochastic) stationary policies as,
\[
\mathcal{Q}_{m\lvert \mathcal{C} \rvert} = \left\{ f_{\Pivec}: \text{ $f_{\Pivec}$ is a $m\lvert \mathcal{C} \rvert$ dimensional distribution over $[0, q_{\max}]^{m\lvert\mathcal{C}\rvert}$} \right\}.
\]
The feasible set $\conv(V_{\varepsilon})$ may be defined as,
\[
\conv(V_{\varepsilon}) = \left\{ (\hat{r}_\varepsilon, \hat{d}): \exists f_{\Pivec} \in \mathcal{Q}_{m\lvert \mathcal{C} \rvert} \text{ s.t., }\Pivec \sim f_{\Pivec} \text{ satisfies } \hat{r}_{\varepsilon} = \EXP\left[ \tilde{r}_{\varepsilon}(\Pivec) \right], \quad \hat{d} = \EXP\left[ \tilde{d}(\Pivec) \right] \right\}.
\]

Next, we prove an analogous result of Proposition \ref{prop:constr_det_pol_opt}.
\begin{prop}
For any $h_{\varepsilon}^{\min} \leq h \leq h_{\varepsilon}^{\max}$, there exists a deterministic policy $\pivec$ that minimizes $\EXP\left[\tilde{d}(\Pivec)\right]$ over the constrained domain $\left\{\Pivec \in \mathcal{Q}_{m\lvert\mathcal{C}\rvert}: \EXP\left[ \tilde{r}_{\varepsilon}(\Pivec) \right] = h \right\}$. 
\label{prop:mcmns_det_opt}
\end{prop}
\begin{proof}
The proof is similar to that of Proposition \ref{prop:constr_det_pol_opt}. 

As a contradiction, assume that there is no deterministic policy that minimizes the expected wall clock time given a constraint $\EXP\left[\tilde{r}_{\varepsilon}(\Pivec)\right] = h$. Let $\Pivec^*$ be a stochastic policy that minimizes the expected wall clock time with the constraint $\EXP[\tilde{r}_{\varepsilon}(\Pivec)]=h$. Then, consider an alternate policy with deterministic compression parameters $\pivec$ chosen as, $\boldsymbol{h}_{\varepsilon}(\pivec) = \EXP\left[ \boldsymbol{h}_{\varepsilon}(\Pivec^*) \right]$. Such a $\pivec$ exists due to the Intermediate Value Theorem since $h_{\varepsilon}(\cdot)$ is a continuous function. In this case, by the strict convexity of the duration function assumed in Assumption \ref{ass:comm_delay}, we have,
\[
\EXP\left[ \tilde{d}(\pivec) \right] < \EXP\left[\tilde{d}\left( \Pivec^* \right) \right].
\]
By the convexity of the norm operator, $\tilde{r}_{\varepsilon}(\pivec) \leq h$. Notice that $\tilde{r}_{\varepsilon}(\cdot)$ is increasing in every co-ordinate, whereas $\tilde{d}(\cdot)$ is decreasing in every co-ordinate. Therefore, for any $\pivec' \geq \pivec$ (elementwise inequality) with $\tilde{r}_{\varepsilon}(\pivec') = h$, we have, 
\[
\EXP\left[ \tilde{d}(\pivec') \right] \leq \EXP\left[ \tilde{d}(\pivec) \right] < \EXP\left[\tilde{d}\left( \Pivec^* \right) \right].
\]
This is a contradiction to the assumption that a stochastic policy minimizes the expected wall clock time given the constraint.
\end{proof}


At this point in the proof of the 1C1NS case, we defined a function $\bar{d}(\hat{r}_{\varepsilon})$ which was the round duration of the deterministic policy whose number of rounds for convergence was $\hat{r}_{\varepsilon}$. In the MCMNS case, since there could be multiple deterministic policies corresponding to a rounds for convergence $\hat{r}_{\varepsilon}$, we define $\bar{d}(\hat{r}_{\varepsilon})$ with respect to the policy that minimizes the round duration.
\begin{equation}
\bar{d}(r_{\varepsilon}) = \min_{\pivec: \tilde{r}_{\varepsilon}(\pivec) = r_{\varepsilon}} \tilde{d}(\pivec).
\label{eq:extreme_curve}
\end{equation}

In the rest of the proof, we need to use the fact that $\bar{d}(\hat{r}_{\varepsilon})$ is strictly convex, and that $\hat{t}(\hat{r}_{\varepsilon}) \triangleq \hat{r}_{\varepsilon}\bar{d}(\hat{r}_{\varepsilon})$ is strictly quasiconvex. In the 1C1NS case, these facts were a direct consequence of Assumptions \ref{ass:comm_delay} and \ref{ass:unique_stationary} because $\bar{d}(\hat{r}_{\varepsilon})$ was simply $d(\tau, h_{\varepsilon}^{-1}(\hat{r}_{\varepsilon}), \c)$. For the MCMNS case we prove these results in the following two propositions.

\begin{prop}
Under Assumptions \ref{ass:q_suff_stat} and \ref{ass:comm_delay}, $\bar{d}(r_{\varepsilon})$ is decreasing and strictly convex in $r_{\varepsilon}$.
\label{prop:opt_delay_convexity}
\end{prop}
\begin{proof}
From Assumption \ref{ass:q_suff_stat}, recall that since $h_{\varepsilon}(\cdot)$ is a strictly increasing, continuous function, it has an inverse $h_{\varepsilon}^{-1}(\cdot)$. Denote by, $\boldsymbol{h}_{\varepsilon}^{-1}: \left[ h_{\varepsilon}^{\min}, h_{\varepsilon}^{\max} \right]^{m\lvert \mathcal{C}\rvert} \to \left[ 0, q_{\max} \right]^{m\lvert \mathcal{C}\rvert}$, the function that outputs a vector obtained by applying $h_{\varepsilon}^{-1}(\cdot)$ elementwise to the input vector.

$\bar{d}(r_{\varepsilon})$ from \eqref{eq:extreme_curve} may be redefined as,
\[
\bar{d}(r_{\varepsilon}) = \min_{\substack{{\r:  \r \in \left[ h_{\varepsilon}^{\min}, h_{\varepsilon}^{\max} \right]^{m\lvert \mathcal{C}\rvert}} \\ {\tilde{r}_{\varepsilon}\left( \boldsymbol{h}_{\varepsilon}^{-1}(\r) \right) = r_{\varepsilon}}}} \tilde{d}\left( \boldsymbol{h}_{\varepsilon}^{-1}(\r) \right).
\]
Due to Assumption \ref{ass:q_suff_stat}, $\tilde{r}_{\varepsilon}\left( \boldsymbol{h}_{\varepsilon}^{-1}(\r) \right)$ is an increasing function in every element of $\r$, and, due to Assumption \ref{ass:comm_delay}, $\tilde{d}\left( \boldsymbol{h}_{\varepsilon}^{-1}(\r) \right)$ is a decreasing function in every element of $\r$. Therefore, we can further reformulate $\bar{d}(r_{\varepsilon})$ as,
\begin{equation}
\bar{d}(r_{\varepsilon}) = \min_{\substack{{\r:  \r \in \left[ h_{\varepsilon}^{\min}, h_{\varepsilon}^{\max} \right]^{m\lvert \mathcal{C}\rvert}} \\ {\tilde{r}_{\varepsilon}\left( \boldsymbol{h}_{\varepsilon}^{-1}(\r) \right) \leq r_{\varepsilon}}}} \tilde{d}\left( \boldsymbol{h}_{\varepsilon}^{-1}(\r) \right).
\label{eq:extreme_curve_2}
\end{equation}
 $\bar{d}(r_{\varepsilon})$ is decreasing because the feasibility set in the minimization problem in \eqref{eq:extreme_curve_2} is an monotonically increasing set with increasing $\hat{r}_{\varepsilon}$.

Consider two points, $r_{\varepsilon, 1}, r_{\varepsilon, 2} \in \left[ h_{\varepsilon}^{\min}, h_{\varepsilon}^{\max} \right]$, and let $\r_{\varepsilon,1}$ and $\r_{\varepsilon,2}$ be their corresponding minimizers according to \eqref{eq:extreme_curve_2},
\begin{equation}
    \begin{split}
        \r_{\varepsilon,1} &\triangleq \underset{\substack{{\r:  \r \in \left[ h_{\varepsilon}^{\min}, h_{\varepsilon}^{\max} \right]^{m\lvert \mathcal{C}\rvert}} \\ {\tilde{r}_{\varepsilon}\left( \boldsymbol{h}_{\varepsilon}^{-1}(\r) \right) \leq r_{\varepsilon, 1}}}}{\text{arg min}} \tilde{d}\left( \boldsymbol{h}_{\varepsilon}^{-1}(\r) \right), \\
    \r_{\varepsilon, 2} &\triangleq \underset{\substack{{\r:  \r \in \left[ h_{\varepsilon}^{\min}, h_{\varepsilon}^{\max} \right]^{m\lvert \mathcal{C}\rvert}} \\ {\tilde{r}_{\varepsilon}\left( \boldsymbol{h}_{\varepsilon}^{-1}(\r) \right) \leq r_{\varepsilon, 2}}}}{\text{arg min}} \tilde{d}\left( \boldsymbol{h}_{\varepsilon}^{-1}(\r) \right).
    \end{split}
    \label{eq:th_1_pf_r_argmin}
\end{equation}

Let $0 < \theta < 1$, $r_{\varepsilon, \theta} = \theta r_{\varepsilon, 1} + (1-\theta) r_{\varepsilon, 2}$ and $\r_{\varepsilon, \theta} = \theta \r_{\varepsilon, 1} + (1-\theta) \r_{\varepsilon, 2}$. By the strict convexity of $\tilde{d}\left( \boldsymbol{h}_{\varepsilon}^{-1}(\cdot) \right)$ as considered in Assumption \ref{ass:comm_delay}, 
\begin{equation}
    \tilde{d}\left( \boldsymbol{h}_{\varepsilon}^{-1}(\r_{\varepsilon, \theta}) \right) < \theta \tilde{d}\left( \boldsymbol{h}_{\varepsilon}^{-1}(\r_{\varepsilon, 1}) \right) + (1-\theta) \tilde{d}\left( \boldsymbol{h}_{\varepsilon}^{-1}(\r_{\varepsilon, 2}) \right).
    \label{eq:th_1_pf_st_convexity}
\end{equation}

From the definition of $\tilde{r}_{\varepsilon}(\cdot)$ in \eqref{eq:tilde_r},
\begin{align}
    \tilde{r}\left(\boldsymbol{h}_{\varepsilon}^{-1}\left(\r_{\varepsilon, \theta} \right) \right) &= \sum_{i=1}^{\lvert \mathcal{C} \rvert} \mu(\c_i)\norm{\r_{\varepsilon, \theta} \odot \boldsymbol{e}^{(i)}}, \nonumber \\
    &\stackrel{(a)}{\leq} \theta \sum_{i=1}^{\lvert \mathcal{C} \rvert} \mu(\c_i)\norm{\r_{\varepsilon,1} \odot \boldsymbol{e}^{(i)}} + (1-\theta) \sum_{i=1}^{\lvert \mathcal{C} \rvert} \mu(\c_i)\norm{\r_{\varepsilon,2} \odot \boldsymbol{e}^{(i)}}, \nonumber \\
    &\stackrel{(b)}{=} \theta \tilde{r}_{\varepsilon}\left(\boldsymbol{h}_{\varepsilon}^{-1}\left(\r_{\varepsilon,1} \right) \right) + (1-\theta) \tilde{r}_{\varepsilon}\left(\boldsymbol{h}_{\varepsilon}^{-1}\left(\r_{\varepsilon,2} \right) \right), \nonumber \\
    &\stackrel{(c)}{\leq} \theta r_{\varepsilon, 1} + (1-\theta) r_{\varepsilon, 2} , \nonumber \\
    &\stackrel{(d)}{=} r_{\varepsilon, \theta}. \label{eq:th_1_pf_feasibility_1}
\end{align}
(a) follows from the convexity of the norm operator, and (b), (c) and (d) follow by definition.

Due to \eqref{eq:th_1_pf_feasibility_1}, $\r_{\varepsilon,\theta}$ is a feasible point in the constraint set of the minimization problem in \eqref{eq:extreme_curve_2} evaluated at $r_{\varepsilon, \theta}$. Therefore,
\begin{align*}
    \bar{d}(r_{\varepsilon, \theta}) &\leq \tilde{d}\left( \boldsymbol{h}_{\varepsilon}^{-1}(\r_{\theta}) \right), \\
    &\stackrel{(a)}{<} \theta \tilde{d}\left( \boldsymbol{h}_{\varepsilon}^{-1}(\r_1) \right) + (1-\theta) \tilde{d}\left( \boldsymbol{h}_{\varepsilon}^{-1}(\r_2) \right), \\
    &\stackrel{(b)}{=} \theta \bar{d}(r_{\varepsilon, 1 }) + (1-\theta) \bar{d}(r_{\varepsilon, 2}).
\end{align*}
(a) follows from \eqref{eq:th_1_pf_st_convexity}, and (b) follows by definition. Since this is true for any $r_{\varepsilon, 1}, r_{\varepsilon, 2} \in \left[ h_{\varepsilon}^{\min}, h_{\varepsilon}^{\max} \right]$, and $0 < \theta < 1$, $\bar{d}(\cdot)$ is strictly convex.
\end{proof}
However, unlike the 1 client and 1 network state case, $\bar{d}(r_{\varepsilon})$ may not be differentiable. But, since it is convex, it has left-derivative and right-derivative functions denoted as $\bar{d}'_L(r_{\varepsilon})$ and $\bar{d}'_R(r_{\varepsilon})$ respectively. 
\begin{align*}
    \bar{d}'_L(r_{\varepsilon}) &= \begin{cases}\lim_{\delta \downarrow 0} \frac{\bar{d}(r_{\varepsilon})-\bar{d}(r_{\varepsilon}-\delta)}{\delta}, & r_\varepsilon > h_\min(\varepsilon)  \\
    -\infty, & r_\varepsilon = h_\min(\varepsilon),
    \end{cases} \\
    \bar{d}'_R(r_{\varepsilon}) &= \begin{cases}
    \lim_{\delta \downarrow 0} \frac{\bar{d}(r_{\varepsilon}+\delta)-\bar{d}(r_{\varepsilon})}{\delta}, & r_\varepsilon < h_\max(\varepsilon)  \\
    0, & r_\varepsilon = h_\max(\varepsilon).
    \end{cases}
\end{align*}

Similar to the case of 1 client and 1 network state, given a constraint, $\tilde{r}_{\varepsilon}(\pivec) = r_{\varepsilon}$, the optimal expected wall clock time may be expressed as $\hat{t}(r_{\varepsilon}) = r_{\varepsilon}\bar{d}(r_{\varepsilon})$. Since $\bar{d}(r_{\varepsilon})$ has left and right derivatives everywhere, so does $\hat{t}(r_{\varepsilon})$. Denote them by $\hat{t}'_L(r_{\varepsilon})$ and $\hat{t}'_R(r_{\varepsilon})$ respectively. 

\begin{prop}
$\hat{t}(r_{\varepsilon})$ is strictly quasiconvex. That is, for any $r_{\varepsilon, 1}, r_{\varepsilon, 2} \in \left[ h_{\varepsilon}^{\min}, h_{\varepsilon}^{\max} \right]$, and $0< \theta < 1$,
\[
\hat{t}(\theta r_{\varepsilon, 1} + (1-\theta)r_{\varepsilon, 2}) < \max \{\hat{t}(r_{\varepsilon,1}), \hat{t}(r_{\varepsilon,2}) \}.
\]
\label{prop:mcmns_quasiconv}
\end{prop}
\begin{proof}
Here, we reuse the definitions introduced in the proof of Proposition \ref{prop:opt_delay_convexity}.

Consider two points  $r_{\varepsilon,1}, r_{\varepsilon,2} \in \left[ h_{\varepsilon}^{\min}, h_{\varepsilon}^{\max} \right]$.  Let $\r_{\varepsilon, 1}$ and $\r_{\varepsilon, 2}$ be defined as in \eqref{eq:th_1_pf_r_argmin}.
Consider, $0<\theta<1$. Since $\tilde{r}_{\varepsilon}\left(\boldsymbol{h}_{\varepsilon}^{-1}(\r)\right)$ is continuous in $\r$, by the Intermediate Value Theorem, there exists a $0<\delta<1$ such that $\r^{\delta} = \delta \r_{\varepsilon,1} + (1-\delta) \r_{\varepsilon, 2}$ has $\tilde{r}_{\varepsilon}(\r^{\delta}) = \theta r_{\varepsilon, 1} + (1-\theta)r_{\varepsilon, 2}$. 

From the strict quasiconvexity of the wall clock time as in Assumption \ref{ass:unique_stationary}, we have,
\[
\tilde{r}_{\varepsilon}\left( \boldsymbol{h}_{\varepsilon}^{-1}(\r^{\delta}) \right)\tilde{d}\left(\boldsymbol{h}_{\varepsilon}^{-1}(\r^{\delta}) \right) < \max\left\{ \hat{t}(r_{\varepsilon, 1}), \hat{t}(r_{\varepsilon, 2}) \right\}.
\]
By definition,
\[
\hat{t}(\theta r_{\varepsilon, 1} + (1-\theta) r_{\varepsilon, 2}) \leq \tilde{r}_{\varepsilon}\left( \boldsymbol{h}_{\varepsilon}^{-1}(\r^{\delta}) \right)\hat{d}\left(\boldsymbol{h}_{\varepsilon}^{-1}(\r^{\delta}) \right).
\]
Therefore,
\[
\hat{t}(\theta r_{\varepsilon, 1} + (1-\theta) r_{\varepsilon, 2}) < \max\left\{ \hat{t}(r_{\varepsilon, 1}), \hat{t}(r_{\varepsilon, 2}) \right\}.
\]
\end{proof}

At this point in the proof of the 1C1NS case, we stated some equivalent descriptions of a point $\x \in V_{\varepsilon}$. The description is the same for the MCMNS case as well, which we restate here for clarity.
\begin{prop}
The following statements are equivalent
\begin{enumerate}
    \item[I.] $\x \in \conv(V_{\varepsilon})$ is of the form $(\hat{r}_{\varepsilon}, \bar{d}(\hat{r}_{\varepsilon}))$ for some $h_{\varepsilon}^{\min} \leq \hat{r}_{\varepsilon} \leq h_{\varepsilon}^{\max}$.
    \item[II.]  $\x \in \conv(V_{\varepsilon})$ is such that $\alpha \x \not{\in} \conv(V_{\varepsilon})$ for any $0<\alpha<1$.
    \item[III.] $\x = (\hat{r}_{\varepsilon}, \hat{d}) \in \conv(V_{\varepsilon})$ is such that, $\hat{d} = \min \{d': (\hat{r}_{\varepsilon}, d') \in \conv(V_{\varepsilon}) \}$. 
    \item[IV.] $\x = (\hat{r}_{\varepsilon}, \hat{d}) \in \conv(V_{\varepsilon})$ is such that, $\hat{r}_{\varepsilon}\hat{d} = \min \{\hat{r}_{\varepsilon}d': (\hat{r}_{\varepsilon}, d') \in \conv(V_{\varepsilon}) \}$.
    \item[V.] $\x$ is an extreme point of $\conv(V_{\varepsilon})$.
\end{enumerate}
\label{prop:mcmns_equiv}
\end{prop}
\begin{proof}[Proof Sketch]
The only difference in the proof from that of Proposition \ref{prop:1c1ns_equiv} is the equivalence I$\iff$III. Here, I$\iff$III follows from the definition of $\bar{d}(\cdot)$ in \eqref{eq:extreme_curve}.  
\end{proof}

At this point in the 1C1NS case, we showed that a non-extreme point of $\conv(V_{\varepsilon})$ cannot be a fixed point of the FFW update. The result is the same in the MCMNS case. We restate the result for clarity, but skip the proof as it is the same as that of Proposition \ref{prop:non_ext_pt}.
\begin{prop}
If a point $\x \in \conv(V_{\varepsilon})$ is not an extreme point of $\conv(V_{\varepsilon})$, then it is not a fixed-point of the FFW update.
\label{prop:mcmns_non_ext_pt}
\end{prop}

Next, similar to the 1C1NS case, we show necessary and sufficient condition for an extreme-point of $\conv(V_{\varepsilon})$ to be a fixed-point of the FFW update.
\begin{prop}
A point $\x = \left( r_{\varepsilon}, \bar{d}(r_{\varepsilon}) \right)^\top$ is a fixed point for the FFW update if and only if $\hat{t}'_L(r_{\varepsilon}) \leq 0$ and $\hat{t}'_R(r_{\varepsilon}) \geq 0$.
\label{prop:mcmns_1}
\end{prop}
\begin{proof}
The proof is very similar to the proof of Proposition \ref{prop:1c1ns_1}, but with some extra care because, here, $\hat{t}(\cdot)$ may not be differentiable.

Consider a point $\x = \left( r_{\varepsilon}, \bar{d}(r_{\varepsilon}) \right)^\top$ which is not a fixed point of the FFW update. This implies that there exists another point, $\z = (r, \bar{d}(r))^\top$, such that,
\[
\nabla H(\x)^\top (\z - \x) \leq 0.
\]
Refer to Fig. \ref{fig:theorem_1_proof_motivation} for an illustration of the following argument. Due to strict convexity of the curve $\bar{d}(\cdot)$(Proposition \ref{prop:opt_delay_convexity}), there exists a point $\w = (\hat{r'} \; \bar{d}(\hat{r}'))^\top$ with $\hat{r}$ being in between $\hat{r}_{\varepsilon}$ and $\hat{r}'$ such that $\nabla H(\x)^\top (\w - \x) = -\xi < 0$.  Then, for any $0<\theta<1$, consider $\w_{\theta}^{c} = (1-\theta)\x + \theta \w$ and $\w_\theta = \left((1-\theta)\hat{r}_\varepsilon + \theta \hat{r}, \bar{d}((1-\theta)\hat{r}_\varepsilon + \theta \hat{r}) \right)$. It is easily verified that $\nabla H(\x)^\top (\w_{\theta}^{c}-\x) = -\xi \theta$. And, due the the convexity of $\bar{d}(\cdot)$, $\nabla H(\x)^\top (\w_{\theta}-\x) \leq \nabla H(\x)^\top (\w_{\theta}^{c}-\x)$. That is, 
\begin{equation}
\nabla H(\x)^\top (\w_\theta - \x) \leq - \xi \theta,
\label{eq:mcmns_1}
\end{equation}
for some positive constant $\xi$.

By a Taylor Series expansion of $H$, we have $\hat{t}((1-\theta)\hat{r}_{\varepsilon}+\theta \hat{r}) = \hat{t}(\hat{r}_\varepsilon)+ \nabla H(\x)^\top (\w_\theta - \x) + O(\theta^2)$. Therefore, due to \eqref{eq:mcmns_1}, for small enough $\theta$, we have $\hat{t}((1-\theta)\hat{r}_{\varepsilon}+\theta \hat{r}) < \hat{t}(\hat{r}_{\varepsilon}) - \xi' \theta$, where $\xi'$ is some positive constant. This, in turn, implies that either $\hat{t}'_L(\hat{r}_{\varepsilon})> 0$ or $\hat{t}'_R(\hat{r}_{\varepsilon}) < 0$.

Conversely, consider for contradiction that $\x = \left( r_{\varepsilon}, \bar{d}(r_{\varepsilon}) \right)^\top$ is a fixed point of the FFW update, but $\hat{t}'_L(r_{\varepsilon}) >0$ or $\hat{t}'_R(r_{\varepsilon}) <0$. Define $r^\delta = r_{\varepsilon}+\delta$, and $\z^{\delta} = (r^\delta, \bar{d}(r^\delta))^\top$. Then, for small enough $\delta>0$, we have that 
\begin{equation}
    \begin{split}
        \hat{t}(r^\delta) &< \hat{t}(r_{\varepsilon}) - \Omega(\delta), \\
        &\text{(OR)} \\
        \hat{t}(r^{-\delta}) &< \hat{t}(r_{\varepsilon}) - \Omega(\delta).
    \end{split}
    \label{eq:mcmns_2}
\end{equation}
Again, by Taylor series expansion, 
\begin{equation}
    \hat{t}(r^\delta) = \hat{t}(r_{\varepsilon}) + \nabla H(\x)(\z^{\delta} - \x) + O(\delta^2).
    \label{eq:mcmns_3}
\end{equation}
Comparing \eqref{eq:mcmns_2} and \eqref{eq:mcmns_3}, we have that $\nabla H(\x)(\z^{\delta} - \x)<0$ or $\nabla H(\x)(\z^{-\delta} - \x)<0$. This is a contradiction to the premise that $\x$ is a fixed point for the FFW update.
\end{proof}

Next, similar to the 1C1NS case, we use the equivalence derived in Proposition \ref{prop:mcmns_equiv} to prove that the fixed point of the FFW update is unique.
\begin{prop}
There exists a unique point $\x = \left( r, \bar{d}(r) \right)^\top$ with $h_{\varepsilon}^{\min}\leq r \leq h_{\varepsilon}^{\max}$  such that $\hat{t}'_L(r) \leq 0$ and $\hat{t}'_R(r) \geq 0$.
\label{prop:mcmns_t_local_eq_global}
\end{prop}
\begin{proof}
 Denote,
\[
\r(r_{\varepsilon}) = \underset{\substack{{\r:  \r \in \left[ h_{\varepsilon}^{\min}, h_{\varepsilon}^{\max} \right]^{m\lvert \mathcal{C}\rvert}} \\ {\tilde{r}_{\varepsilon}\left( \boldsymbol{h}_{\varepsilon}^{-1}(\r) \right) = r_{\varepsilon}}}}{\text{arg min}} \tilde{d}\left(\boldsymbol{h}_{\varepsilon}^{-1}(\r) \right).
\]
Also, denote,
\[
\tilde{t}_{\varepsilon}(\r) = \tilde{r}_{\varepsilon}(\boldsymbol{h}_{\varepsilon}^{-1}(\r)) \tilde{d}(\boldsymbol{h}_{\varepsilon}^{-1}(\r))
\]
We use the following claim which is proved at the end of this section.
\begin{claim}
If $\hat{t}'_L(r_{\varepsilon}) \leq 0$ and $\hat{t}'_R(r_{\varepsilon}) \geq 0$, then,
\[
\nabla_{\r} \left(\tilde{t}_{\varepsilon}(\r(r_{\varepsilon})) \right) = 0,
\]
or, the left and right derivatives of $\r(\cdot)$ evaluated at $r_{\varepsilon}$, $\r'_{L}(r_{\varepsilon})$ and $\r'_{R}(r_{\varepsilon})$, exist, and,
\begin{align*}
    \r'_{L}(r_{\varepsilon})^\top \nabla_{\r} \left(\tilde{t}_{\varepsilon}(\r(r_{\varepsilon})) \right) &\leq 0, \\
    (AND) \qquad &\quad \\
    \r'_{R}(r_{\varepsilon})^\top \nabla_{\r} \left(\tilde{t}_{\varepsilon}(\r(r_{\varepsilon})) \right) &\geq 0.
\end{align*}
\label{claim:tL=0}
\end{claim}

In either case, Assumption \ref{ass:unique_stationary} ensures that for all small enough $\delta > 0$, we have $\hat{t}(r_{\varepsilon}-\delta) > \hat{t}(r_{\varepsilon})$ and $\hat{t}(r_{\varepsilon}+\delta) > \hat{t}(r_{\varepsilon})$. That is, if $\hat{t}'_L(r_{\varepsilon}) \leq 0$, then for all small enough $\delta > 0$, $\hat{t}(r_{\varepsilon}-\delta) > \hat{t}(r_{\varepsilon})$ (whenever $r_{\varepsilon}-\delta$ is in $[h_{\varepsilon}^{\min}, h_{\varepsilon}^{\max}]$). Similarly, if $\hat{t}'_R(r_{\varepsilon}) \geq 0$, then for small enough $\delta>0$, $\hat{t}(r_{\varepsilon}+\delta)> t(r_{\varepsilon})$ (whenever $r_{\varepsilon}+\delta$ is in $[h_{\varepsilon}^{\min}, h_{\varepsilon}^{\max}]$). Therefore, $r_{\varepsilon}$ is a strict local minimum of of $\hat{t}$.

Since $\hat{t}(\cdot)$ is strictly quasiconvex (Proposition \ref{prop:mcmns_quasiconv}), there can be at most one point $r_{\varepsilon}$ which is a strict local minimum of $\hat{t}(\cdot)$.

Moreover, since $\hat{t}$ is a continuous function over a closed and bounded set, it attains a local minimum over its domain. And, $\hat{t}'_L(r_{\varepsilon}) \leq 0$ and $\hat{t}'_R(r_{\varepsilon}) \geq 0$ is necessary condition for a local minimum. Therefore, there exists at least one point $r_{\varepsilon}$ in the domain such that $\hat{t}'_L(r_{\varepsilon}) \leq 0$ and $\hat{t}'_R(r_{\varepsilon}) \geq 0$.
\end{proof}

\begin{proof}[Proof of Claim \ref{claim:tL=0}]
Recall the notation, $\tilde{t}_{\varepsilon}(\r) = \tilde{r}_{\varepsilon}(\boldsymbol{h}_{\varepsilon}^{-1}(\r)) \tilde{d}(\boldsymbol{h}_{\varepsilon}^{-1}(\r))$.
\paragraph{Case 1}: First, consider the case that $\r(r_{\varepsilon})$ is in the interior of the domain $[h_{\varepsilon}^{\min}, h_{\varepsilon}^{\max}]^{m \lvert \mathcal{C}\rvert}$. In this case we will show that, $\nabla \tilde{t}_{\varepsilon}(\r(r_{\varepsilon})) = 0$. As a contradiction, assume that $\nabla \tilde{t}_{\varepsilon}(\r(r_{\varepsilon})) \neq 0$. 

Let $L(r_{\varepsilon}) \triangleq \{\r: \tilde{r}(\r) = r_{\varepsilon} \}$. Since, $\r(r_{\varepsilon})$ is the minimizer of $\tilde{t}_{\varepsilon}(\r)$ over the set $L(r_{\varepsilon})$, $\nabla \tilde{t}_{\varepsilon}(\r(r_{\varepsilon}))$ has to be normal to $L(r_{\varepsilon})$ at the point $\r(r_{\varepsilon})$. Let $\boldsymbol{n}$ denote the normal to the set $L$ at point $r_{\varepsilon}$. Then, one of the following is true,
\begin{align*}
    \tilde{t}_{\varepsilon}(\r(r_{\varepsilon}) + \delta \boldsymbol{n}) &= \tilde{t}_{\varepsilon}(\r(r_{\varepsilon})) - \Theta(\delta), \\
    &(OR) \\
    \tilde{t}_{\varepsilon}(\r(r_{\varepsilon}) - \delta \boldsymbol{n}) &= \tilde{t}_{\varepsilon}(\r(r_{\varepsilon})) - \Theta(\delta).
\end{align*}
Then, by the definition of $t(\cdot)$, one of the following is true,
\begin{align*}
    t(r_{\varepsilon} + \delta) &= t(r_{\varepsilon}) - \Omega(\delta), \\
    &(OR) \\
    t(r_{\varepsilon} - \delta ) &= t(r_{\varepsilon}) - \Omega(\delta).
\end{align*}
This implies that either $t'_L(r_{\varepsilon}) > 0$ or $t'_R(r_{\varepsilon}) > 0$. This is a contradiction. Therefore, $\nabla \tilde{t}_{\varepsilon}(\r(r_{\varepsilon})) = 0$, for all $\r(r_{\varepsilon})$ in the interior of $[h_{\varepsilon}^{\min}, h_{\varepsilon}^{\max}]^{m \lvert \mathcal{C}\rvert}$.

\paragraph{Case 2: } Consider the case that $\r(r_{\varepsilon})$ is on the boundary of $\left[h_{\varepsilon}^{\min}, h_{\varepsilon}^{\max} \right]^{m \lvert \mathcal{C} \rvert}$ and $\nabla \tilde{t}_{\varepsilon}(\r(r_{\varepsilon})) \neq 0$. 

Here, the left derivative exists and its direction is expressed as,
\[
\r'_{L}(r_{\varepsilon}) \propto \underset{\boldsymbol{s}: \r(r_{\varepsilon})+\boldsymbol{s} \in \left[ h_{\varepsilon}^{\min}, h_{\varepsilon}^{\max} \right]^{m\lvert \mathcal{C} \rvert}}{\text{arg max}} \quad \frac{\boldsymbol{s}^\top \nabla \tilde{t}_{\varepsilon}(\r(r_{\varepsilon}))}{\norm{\boldsymbol{s}}_2}.
\]
Since, $t'_{L}(r_{\varepsilon}) \leq 0$, we obtain $\r'_L(r_{\varepsilon})^\top \nabla \tilde{t}_{\varepsilon}(\r(r_{\varepsilon})) \leq 0$. 

A similar argument holds for the right derivative as well.
\end{proof}

Finally, we summarize how the results in this section prove Proposition \ref{prop:unique_stationary}.

\begin{proof}[Proof Summary of Proposition \ref{prop:unique_stationary}]
Due to Proposition \ref{prop:mcmns_non_ext_pt}, a non-extreme point of $\conv(V_{\varepsilon})$ is not a fixed-point of the FFW update. Further, Proposition \ref{prop:mcmns_equiv} shows that an equivalent condition for an extreme point $\x = (\hat{r}_{\varepsilon}, \bar{d}(\hat{r}_{\varepsilon}))$ of $\conv(V_{\varepsilon})$ being a fixed point of the FFW update is that $t'_{L}(\hat{r}_{\varepsilon}) \leq 0$ and $t'_{R}(\hat{r}_{\varepsilon}) \geq 0$. Then, in Proposition \ref{prop:mcmns_t_local_eq_global} we showed that there is a unique point which satisfies $t'_{L}(\hat{r}_{\varepsilon}) \leq 0$ and $t'_{R}(\hat{r}_{\varepsilon}) \geq 0$. Therefore, this proves that there is a unique fixed point $\x^*$ for the FFW update in $\conv(V_{\varepsilon})$.

Further, since $\x^*$ is an extreme point of $\conv(V_{\varepsilon})$, and $\conv(V_{\varepsilon})$ is the convex hull of the set $V_{\varepsilon}$, $\x^*$ lies in the set $V_{\varepsilon}$.

Finally, to prove that $\x^*$ is the minimizer of $H(\cdot)$ over the set $V_{\varepsilon}$, we make the following observation which is a consequence of IV in Proposition \ref{prop:mcmns_equiv},
\[
\min_{\x \in V_{\varepsilon}} H(\x) = \min_{r \in [h_{\varepsilon}^{\min}, h_{\varepsilon}^{\max}]} \hat{t}(r).
\]
Recall from Proposition \ref{prop:mcmns_equiv} that $\x^* = (r^*, \bar{d}(r^*))$ is such that $t'_{L}(r^*) \leq 0$ and $t'_R(r^*) \geq 0$. Further, since $\hat{t}$ is strictly quasiconvex (Proposition \ref{prop:mcmns_quasiconv}), $r^*$ is the global minimizer of $\hat{t}$. Therefore, $\x^*$ is the minimizer of $H(\cdot)$ over the set $V_{\varepsilon}$.
\end{proof}

\newpage
\section{Proof of Theorem \ref{th:fedcom_informal}}
\label{app:fedcom_proof}

In this section, we prove Theorem \ref{th:fedcom_informal} which bounds the convergence of the FedCOM-V Algorithm shown in Algorithm \ref{algo:fedcom}. Moreover, we explicitly state the choice of local learning rate schedule $\left( \eta^n \right)$, and global learning rate schedule $\left( \gamma^n \right)$ required to achieve the convergence rate stated in Theorem \ref{th:fedcom_informal}. 

In the rest of the section, we violate our convention by sometimes using lower case letters instead of capital letters to denote random variables/vectors in order to stay consistent with the notation used in FL literature. 

We start by defining sigma-algebras and filtrations associated with the probability space.
\begin{remark}
 Let $\sigma(X)$ denote the sigma-algebra generated by the random variable $X$. Let $\sigma(X_1, \ldots, X_n)$ denote the sigma-algebra generated by the set of random variables $X_1$ to $X_n$. Similarly, let $\sigma(\mathcal{F}_1, \mathcal{F}_2, \dots, \mathcal{F}_n)$ denote the smallest sigma-algebra containing the sigma-algebras $\mathcal{F}_1$ to $\mathcal{F}_n$.

First, we describe the sigma-algebra associated with the network state process. Recalling that $C^n$ denotes the network state at round $n$, denote,
\[
\mathcal{F}^{C} \triangleq \sigma\left(  C^n: n \geq 1  \right).
\]
We remark that although the compression parameters $(\q^n)_{n}$ may be random vectors dependent across various rounds $n$, their randomness only depends on the network states under the NAC-FL policy as well as under other baseline policies we consider in this paper. More precisely, $\q^n$ is measurable in $\mathcal{F}^{C}$ for all rounds $n$, and is therefore not dependent on the updates of the FedCOM-V algorithm. Moreover, the aim in this proof is to study the convergence of the FedCOM-V algorithm for arbitrary choices of $(\q^n)_n$. So, in this section, all expectations will be conditioned on $\mathcal{F}^C$, and therefore, $q_j^n$'s will be treated as arbitrary, but known, constants in this Section.

Next, we describe the sigma-algebras associated with the stochastic gradients. Recall that at round $n$, by local-step $b$, client $j$ has sampled mini-batches $\left( \mathcal{Z}_{j}^{a,n} \right)_{a=1}^{b}$ to compute the stochastic-gradients. So, we denote the associated sigma-algebra across all clients as,
\[
\mathcal{D}_{b,n} \triangleq \sigma\left( \mathcal{Z}_{j}^{a,n}: j \in [m], a \in [b]  \right), \quad b \in [\tau_{n}], \\
\]

The sigma-algebra associated with the compressors used in round $n$ is denoted as,
\[
\mathcal{Q}_n \triangleq \sigma\left( \mathcal{Q}(\cdot, q_j^n): j \in [m]  \right).
\]

Finally, we describe the filtration for the entire system. Since, in this Section, we condition all events on the knowledge of the network states, the filtration is initialized as,
\[
\mathcal{F}_{0} = \F^C.
\]
At the $b$\textsuperscript{th} local step of round $n$, the filtration includes the knowledge of all previous rounds and the stochastic-gradients up to the $b$\textsuperscript{th} local step,
\[
\F_{b,n} = \sigma\left(\F_{n-1}, \mathcal{D}_{b,n} \right), \quad, b\in[\tau_n], n \geq 1,
\]
And, finally, at the end of round $n$, the filtration includes the knowledge of all previous rounds, the stochastic gradients and compressors of round $n$,
\[
\F_{n} = \sigma\left( \F_{n-1}, \mathcal{D}_{\tau_n, n},  \mathcal{Q}_n \right), \quad n \geq 1.
\]
\end{remark}

Next we recap the FedCOM-V Algorithm and introduce further notation used in the proof.
\begin{figure}[h]
    \centering
    \includegraphics[width=\textwidth]{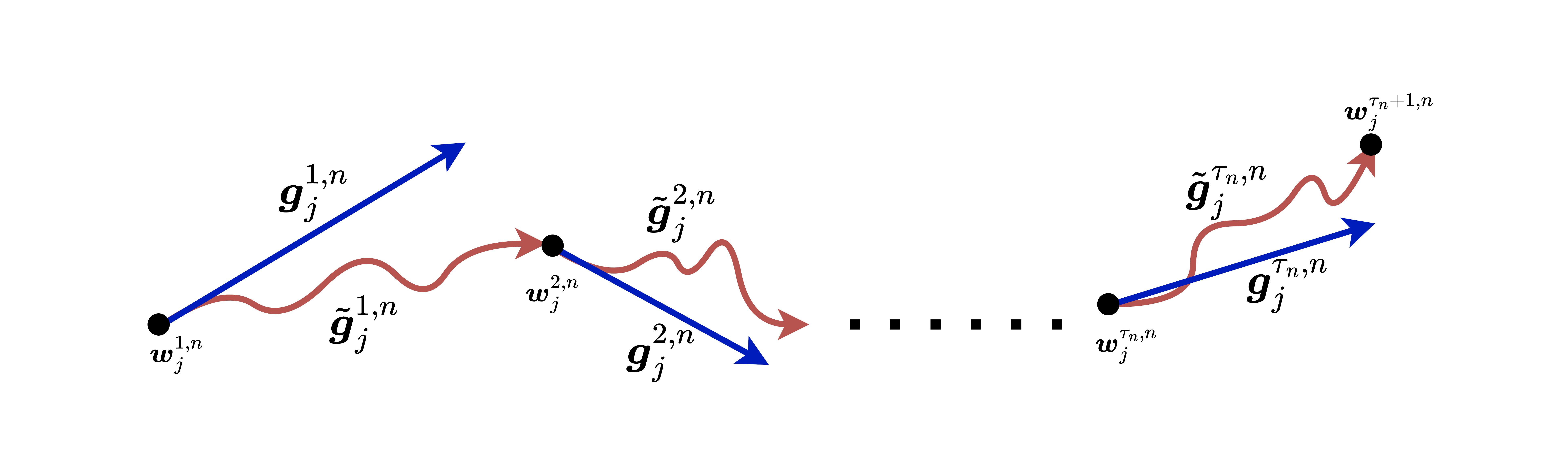}
    \caption{Illustration of the local steps at a client.}
    \label{fig:stochastic_evolution}
\end{figure}

At the start of round $n$, client $j$ recieves the global model $\w^n$ from the server, which it initializes as $\w_j^{1,n}$. Then, at local-step $b$, it samples a mini-batch $\mathcal{Z}_j^{b,n}$ and computes the stochastic gradient, $\tilde{\g}_j^{b,n} \triangleq \nabla f(\w_j^{b,n}, \mathcal{Z}_j^{b,n})$, while performing the local model update as, $\w_j^{b+1,n} = \w_j^{b,n} - \eta_n \tilde{\g}_j^{b,n}$. At this point, we remark that there are two sources of randomness involved in the evaluation of a stochastic gradient at a local step $b$ of round $n$. One is from the model $\w_j^{b,n}$ at which the gradient is evaluated, which is itself obtained by stochastic gradient and compressed aggregation updates of previous rounds and local steps (i.e., $\w_j^{b,n}$ is measurable in $ \mathcal{F}_{b-1,n}$). The second source of randomness is from the mini-batch, $\mathcal{Z}_j^{b,n}$, used to compute the stochastic gradient $\tilde{\g}_j^{b,n}$.  So, $\tilde{\g}_j^{b,n}$ is measurable in $ \mathcal{F}_{b,n}$.

Additionally, for the analysis, we will define the ``true-gradient'' at the local model $\w_j^{b,n}$ as, $\g_{j}^{b,n} \triangleq \nabla f(\w_j^{b,\round})$. Observe that the true gradient at local step $b$ of round $n$ is itself a random vector, as it is evaluated at $\w_j^{b,\round}$. But, it is independent of the mini-batch, $\mathcal{Z}_j^{b,n}$, sampled at that step. Therefore, $\g_j^{b,n}$ is measurable in $ \mathcal{F}_{b-1,n}$. Refer to Figure \ref{fig:stochastic_evolution} for an illustration of this process. 


After $\tau_n$ local computations, the client computes its ``pre-compressed'' update, $\tilde{\g}_{j}^{n} \triangleq \sum_{b=1}^{\tau_n} \tilde{\g}_{j}^{b,n}$, which can also be expressed as, 
$\tilde{\g}_j^n = (\w^n - \w_j^{\tau_n+1,n})/\eta_n$. Next, the client sends the compressed message, $\tilde{\g}_{Qj}^n = \mathcal{Q}(\tilde{\g}_j^n, q_j^n)$, to the server.

The server aggregates the compressed messages received from the clients as, $\tilde{\g}_Q^n = 1/m \sum_{j=1}^m \tilde{\g}_{Qj}^n$, and performs the update $\w^{n+1} = \w^n - \eta_n\gamma_n \tilde{\g}_Q^n$.

For the purpose of analysis, define $\tilde{\g}^n \triangleq 1/m \sum_{j=1}^m \tilde{\g}_{j}^n$, which may be interpreted as the message aggregated at the server had the clients not used any compression. Further, define the ``true-gradient'' analogies of $\tilde{\g}_j^n$ and $\tilde{\g}^n$ as, $\g_j^n = \sum_{b=1}^{\tau_n} \g_j^{b,n}$ and $\g^n = 1/m \sum_{j=1}^m \g_j^n$. $\g^n$ and $\g_j^n$ are random variables since their components, $\g_j^{b,n}$'s, are evaluated at models which are obtained by stochastic gradient updates.

\begin{remark}
   The presence of a tilde and the subscript $Q$, such as in $\tilde{\g}_Q^n$, will indicate that vector is both compressed and has stochastic gradient components. The presence of just a tilde, such as in $\tilde{\g}_j^n$, will indicate that vector (or its components) has two sources of randomness: one from the model at which it (or its components) is evaluated, and second from the mini-batch using which it (or its components) is evaluated. The absence of both the tilde and subscript $Q$, such as in $\g_j^n$, will indicate that the vector (or its components) has one source of randomness, which is from the model at which it (or its components) is evaluated. 
\end{remark}



Also, denote the average noise-variance across clients per round as $\bar{q}^{\round}$:
\[
\bar{q}^{\round} = \frac{1}{m} \sum_{j=1}^m q_j^{\round}.
\]

We start by stating some results which will assist in proving Theorem \ref{th:fedcom_informal}. First is a lemma that bounds the distance between the sum of stochastic gradients at a client in a round to the sum of the true gradients across the local steps at the client in the round.

\begin{lemma}
    $\EXP\left[ \norm{\tilde{\g}_j^n - \g_j^n}^2 \Big\vert \mathcal{F}_{n-1} \right] \leq \tau_n \sigma^2$.
    \label{lm:var_g_n}
\end{lemma}
\begin{proof}
    \begin{align*}
        \EXP\left[ \norm{\tilde{\g}_j^n - \g_j^n}^2 \Big\vert \mathcal{F}_{n-1}\right] 
            &\stackrel{(a)}{=} \EXP\left[ \norm{\sum_{b=1}^{\tau_n} \left( \tilde{\g}_j^{b,n} - \g_j^{b,n}\right)}^2 \Bigg\vert \mathcal{F}_{n-1}\right], \\
            &\stackrel{(b)}{=} \EXP\left[ \EXP\left[\norm{\sum_{b=1}^{\tau_n} 
            \left( \tilde{\g}_j^{b,n} - \g_j^{b,n} \right)}^2 \Big\vert \mathcal{F}_{\tau_n-1,n}\right] \Bigg\vert \mathcal{F}_{n-1} \right], \\
            &\stackrel{(c)}{=} \EXP\Bigg[ \EXP\left[\norm{\sum_{b=1}^{\tau_n-1} 
            \left( \tilde{\g}_j^{b,n} - \g_j^{b,n} \right)}^2 \Big\vert \mathcal{F}_{\tau_n-1, n}\right] + \EXP \left[ \norm{\tilde{\g}_j^{\tau_n,n} - \g_j^{\tau_n,n}}^2 \Big\vert \mathcal{F}_{\tau_n-1,n}\right] \\
            &\qquad + 2\EXP\left[ \left(\tilde{\g}_j^{\tau_n,n} - \g_j^{\tau_n,n} \right)^\top \sum_{b=1}^{\tau_n-1} 
            \left( \tilde{\g}_j^{b,n} - \g_j^{b,n} \right) \big\vert \mathcal{F}_{\tau_n-1,n} \right] \Bigg\vert \mathcal{F}_{n-1} \Bigg], \\
            &\stackrel{(d)}{=} \EXP\Bigg[ \norm{\sum_{b=1}^{\tau_n-1} 
            \left( \tilde{\g}_j^{b,n} - \g_j^{b,n} \right)}^2 + \EXP \left[ \norm{\tilde{\g}_j^{\tau_n,n} - \g_j^{\tau_n,n}}^2 \Big\vert \mathcal{F}_{\tau_n-1,n}\right] \\
            &\qquad + 2\EXP\left[ \left(\tilde{\g}_j^{\tau_n,n} - \g_j^{\tau_n,n} \right) \Big\vert \mathcal{F}_{\tau_n-1,n}\right]^\top \sum_{b=1}^{\tau_n-1} 
            \left( \tilde{\g}_j^{b,n} - \g_j^{b,n} \right) \Bigg\vert \mathcal{F}_{n-1} \Bigg], \\
            &\stackrel{(e)}{=} \EXP\left[ \norm{\sum_{b=1}^{\tau_n-1} 
            \left( \tilde{\g}_j^{b,n} - \g_j^{b,n} \right)}^2 + \EXP \left[ \norm{\tilde{\g}_j^{\tau_n,n} - \g_j^{\tau_n,n}}^2 \Big\vert \mathcal{F}_{\tau_n-1,n}\right] \Bigg\vert \mathcal{F}_{n-1}\right], \\
            &\stackrel{(f)}{\leq} \EXP\left[ \norm{\sum_{b=1}^{\tau_n-1} 
            \left( \tilde{\g}_j^{b,n} - \g_j^{b,n} \right)}^2 \Bigg\vert \mathcal{F}_{n-1} \right] + \sigma^2 , \\
            &\vdots \\
            &\leq \tau_n \sigma^2
    \end{align*}
    (a) follows by definition. (b) follows by law of iterated expectations. (c) follows by partially expanding the summation. (d) follows because all random vectors except $\tilde{\g}_j^{\tau_n,n}$ are measurable in $\mathcal{F}_{\tau_n-1,n}$. (e) follows because $\tilde{\g}_j^{\tau_n,n}$ is an unbiased estimate of $\g_j^{\tau_n,n}$ (Assumption \ref{ass:sg_variance}). (f) follows by Assumption \ref{ass:sg_variance} which bounds the variance of the stochastic gradients. Repeating steps (b)-(f) ($\tau_n-1$) more times by taking internal conditional expectations w.r.t., $\mathcal{F}_{\tau_n-2,n}$, ..., $\mathcal{F}_{1,n}$, $\mathcal{F}_{n-1}$ gives us the result.
\end{proof}

Next is a lemma comparing the expected norm of the (compressed and stochastic) gradient received by the server to the true gradients at all the local steps.

\begin{lemma}
Under Assumption \ref{ass:sg_variance} and \ref{ass:quantization},
\[
    \EXP \left[ \norm{\tilde{\g}_Q^{\round}}^2 \Big\vert \mathcal{F}_{n-1}\right]
    \leq \frac{2\tau_n}{m} \left(\frac{q_\max}{m}+1\right)\sum_{j=1}^m \sum_{b=1}^{\tau_n} \EXP\left[\norm{\g_j^{b,\round}}^2 \Big\vert \mathcal{F}_{n-1}  \right] + \left(\bar{q}^{\round} + 1 \right) \frac{2\tau_n \sigma^2}{m}.
\]
\label{lm:fedcom_1}
\end{lemma}
\begin{proof}
\begin{align}
    \EXP&\left[\norm{\tilde{\g}_Q^{\round}}^2 \Big\vert \mathcal{F}_{n-1}\right] \nonumber \\
    &\stackrel{(a)}{=} \EXP\left[ \EXP \left[ \norm{\frac{1}{m} \sum_{j=1}^m \tilde{\g}_{Q,j}^{\round}}^2 \Big\vert \mathcal{F}_{\tau_n,n} \right] \Big\vert \mathcal{F}_{n-1} \right], \nonumber \\
    &\stackrel{(b)}{=} \EXP\left[ \EXP \left[ \norm{\frac{1}{m} \sum_{j=1}^m \tilde{\g}_{Q,j}^{\round} - \frac{1}{m} \sum_{j=1}^m \EXP\left[ \tilde{\g}_{Q,j}^{\round} \Big\vert \mathcal{F}_{\tau_n,n} \right]}^2 \Big\vert \mathcal{F}_{n-1} \right] + \norm{\EXP\left[ \frac{1}{m} \sum_{j=1}^m \tilde{\g}_{Q,j}^{\round} \Big\vert \mathcal{F}_{\tau_n,n} \right]}^2 \Big\vert \mathcal{F}_{n-1}\right], \nonumber \\
    &\stackrel{(c)}{=} \EXP \left[ \frac{1}{m^2} \sum_{j=1}^m \left(\EXP \left[\norm{\tilde{\g}_{Q,j}^{\round}-\tilde{\g}_j^{\round}}^2 \Big\vert \mathcal{F}_{\tau_n,n} \right] \right) + \norm{\frac{1}{m} \sum_{j=1}^m \tilde{\g}_j^{\round}}^2 \Big\vert \mathcal{F}_{n-1} \right], \nonumber \\
    &\stackrel{(d)}{\leq} \EXP \left[ \sum_{j=1}^m \frac{q_j^{\round}}{m^2} \norm{\tilde{\g}_j^{\round}}^2 + \norm{\tilde{\g}^{\round}}^2 \Big\vert \mathcal{F}_{n-1} \right].
    \label{eq:fedcom_1_1}
\end{align}
(a) follows by Law of Iterated Expectations. (b) follows by the identity, $\EXP[\norm{\X}^2] = \EXP[\norm{\X-\EXP[\X]}^2]+\norm{\EXP[\X]}^2$ (this is the $\EXP X^2 = \text{var}(X) + (\EXP X)^2$ identity applied to vectors). The first summation term of (c) is obtained by expanding out the first squared norm from (b) and observing that $\tilde{\g}^{\round}_{Q,j}-\tilde{\g}^{\round}_{j}$ is a zero mean random vector, and independent across different $j$'s. The second term in (c) follows from the linearity of expectation and the unbiased property of the compressor (Assumption \ref{ass:quantization}). (d) follows from Assumption \ref{ass:quantization}, which bounds the noise introduced by the compressor.

Let's bound $\EXP\left[ \norm{\tilde{\g}^{\round}}^2 \Big\vert \mathcal{F}_{n-1} \right]$,
\begin{align}
    \EXP\left[ \norm{\tilde{\g}^{\round}}^2 \Big\vert \mathcal{F}_{n-1}\right] 
    &\stackrel{(a)}{\leq} 2\EXP \left[ \norm{\tilde{\g}^{\round}- \g^{\round}}^2 \Big\vert \mathcal{F}_{n-1} \right] + 2\EXP\left[ \norm{\g^{\round}}^2 \mathcal{F}_{n-1}\right], \nonumber \\
    &\stackrel{(b)}{=} 2\EXP \left[ \norm{\frac{1}{m}\sum_{j=1}^m \left(\tilde{\g}_j^{\round}- \g_j^{\round}\right)}^2 \Big\vert \mathcal{F}_{n-1} \right] + 2\EXP \left[ \norm{\frac{1}{m}\sum_{j=1}^m \g_{j}^{\round}}^2 \Big\vert \mathcal{F}_{n-1} \right], \nonumber \\
    &\stackrel{(c)}{=} \frac{2}{m^2}\sum_{j=1}^m \EXP \left[ \norm{ \tilde{\g}_j^{\round}- \g_j^{\round}}^2 \Big\vert \mathcal{F}_{n-1}\right] + 2\EXP \left[ \norm{\frac{1}{m}\sum_{j=1}^m \g_{j}^{\round}}^2 \mathcal{F}_{n-1}\right], \nonumber \\
    &\stackrel{(d)}{\leq} \frac{2}{m^2}\sum_{j=1}^m \EXP \left[ \norm{ \tilde{\g}_j^{\round}- \g_j^{\round}}^2 \Big\vert \mathcal{F}_{n-1}\right] + \frac{2}{m}\sum_{j=1}^m \EXP \left[  \norm{\g_{j}^{\round}}^2 \Big\vert \mathcal{F}_{n-1} \right], \nonumber \\
    &\stackrel{(e)}{\leq} \frac{2\tau_n\sigma^2}{m} + \frac{2}{m} \sum_{j=1}^m \EXP \left[\norm{\g_j^{\round}}^2 \Big\vert \mathcal{F}_{n-1} \right]. \label{eq:fedcom_1_2}
\end{align}
Above, (a) follows from the identity $\norm{\x}^2 \leq 2\norm{\x-\y}^2+2\norm{\y}^2$ (to prove the identity, observe that $\norm{\cdot}^2$ is a convex function, and use Jensen's inequality, $\norm{\x/2}^2 \leq 1/2 \norm{\x-\y}^2 + 1/2 \norm{\y}^2$), and (b) follows by definition. (c) is true because, given $\mathcal{F}_{n-1}$, $(\tilde{\g}_{j}^{n} - \g_{j}^{n})$ is independent of $(\tilde{\g}_{k}^{n}-\g_{k}^{n})$ for $k \neq j$. (d) follows from Jensen's Inequality, and (e) is true by Lemma \ref{lm:var_g_n}.

Performing a similar calculation again,
\begin{align}
    \EXP\left[ \norm{\tilde{\g}_j^{\round}}^2 \Big\vert \mathcal{F}_{n-1} \right] 
    &\leq 2\EXP \left[ \norm{\tilde{\g}_j^{\round} - \g_j^{\round}}^2 \Big\vert \mathcal{F}_{n-1} \right] + 2\EXP \left[ \norm{\g_j^{\round}}^2 \Big\vert \mathcal{F}_{n-1} \right], \nonumber \\
    &\leq 2\tau_n \sigma^2 + 2\EXP \left[ \norm{\g_j^{\round}}^2 \Big\vert \mathcal{F}_{n-1} \right].
    \label{eq:fedcom_1_3}
\end{align}
Using Jensen's inequality, we get,
\begin{equation}
    \norm{\g_j^{\round}}^2 \leq \tau_n \sum_{b=1}^{\tau_n} \norm{\g_j^{b,\round}}^2.
    \label{eq:fedcom_1_4}
\end{equation}
Substituting \eqref{eq:fedcom_1_2}, \eqref{eq:fedcom_1_3} and \eqref{eq:fedcom_1_4} in \eqref{eq:fedcom_1_1}, we get the result,
\begin{equation}
    \EXP\left[\norm{\tilde{\g}_Q^{\round}}^2 \Big\vert \mathcal{F}_{n-1} \right] \leq \frac{2\tau_n}{m} \sum_{j=1}^m \sum_{b=1}^{\tau_n} \left( \frac{q_j^{\round}}{m} + 1 \right) \EXP\left[ \norm{\g_j^{b,\round}}^2 \Big\vert \mathcal{F}_{n-1} \right] + \left(\sum_{j=1}^m\frac{q_j^{\round}}{m} + 1 \right) \frac{2\tau_n \sigma^2}{m}.
\end{equation}
\end{proof}

The following lemma bounds the inner product between the true gradient evaluated at the global model at the start of a round and the approximate gradient received by the server. 
\begin{lemma}
Under Assumption \ref{ass:L-smooth}, the FedCOM-V updates follow,
\begin{equation*}
-\EXP\left[ \left\langle \nabla f(\w^{\round}),
    \tilde{\g}_Q^{\round} \right\rangle \Big\vert \mathcal{F}_{n-1} \right]
   \leq \frac{1}{2m} \sum_{j=1}^m \sum_{b=1}^{\tau_n} \Big( -\norm{\nabla f(\w^{\round})}^2 - \EXP\left[ \norm{\g_j^{b,\round}}^2 \Big\vert \mathcal{F}_{n-1}\right]
    + L^2 \EXP \left[ \norm{\w^{\round} - \w_j^{b,\round}}^2 \Big\vert \mathcal{F}_{n-1} \right] \Big).
\end{equation*}
\label{lm:fedcom_2}
\end{lemma}
\begin{proof}
    \begin{align*}
        -\EXP\left[ \left\langle \nabla f(\w^{\round}),
            \tilde{\g}_Q^{\round} \right\rangle \Big\vert \mathcal{F}_{n-1} \right]
            &\stackrel{(a)}{=} -\EXP\left[ \left\langle \nabla f(\w^{\round}),
            \EXP\left[\tilde{\g}_Q^{\round} \Big\vert \mathcal{F}_{\tau_n,n} \right] \right\rangle \Big\vert \mathcal{F}_{n-1} \right], \\
            &\stackrel{(b)}{=} -\EXP\left[ \left\langle \nabla f(\w^{\round}),
            \tilde{\g}^{\round} \right\rangle \Big\vert \mathcal{F}_{n-1} \right], \\
            &\stackrel{(c)}{=} -\EXP\left[ \left\langle \nabla
                f(\w^{\round}), \frac{1}{m} \sum_{j=1}^m \sum_{b=1}^{\tau_n} \tilde{\g}_{j}^{b,n} \right\rangle \Big\vert \mathcal{F}_{n-1} \right], \\
            &\stackrel{(d)}{=} - \EXP \left[ \frac{1}{m} \sum_{j=1}^m \sum_{b=1}^{\tau_n} \left\langle \nabla
                f(\w^{\round}),\EXP\left[\tilde{\g}_{j}^{b,n} \Big\vert \mathcal{F}_{b-1,n} \right]\right\rangle \Big\vert \mathcal{F}_{n-1} \right] , \\
            &\stackrel{(e)}{=} -  \EXP \left[ \frac{1}{m} \sum_{j=1}^m \sum_{b=1}^{\tau_n} \left\langle \nabla
                f(\w^{\round}),\g_{j}^{b,n} \right\rangle \Big\vert \mathcal{F}_{n-1} \right] , \\
            &\stackrel{(f)}{=}  \EXP \left[ \frac{1}{2m} \sum_{j=1}^m \sum_{b=1}^{\tau_n} \left( - \norm{\nabla
                f(\w^{\round})}^2 - \norm{\g_{j}^{b,n}}^2 + \norm{\nabla
                f(\w^{\round}) - \g_{j}^{b,n}}^2 \right) \Big\vert \mathcal{F}_{n-1} \right], \\
            &\stackrel{(g)}{\leq} \EXP \left[ \frac{1}{2m} \sum_{j=1}^m \sum_{b=1}^{\tau_n} \left( - \norm{\nabla
                f(\w^{\round})}^2 - \norm{\g_{j}^{b,n}}^2 + L^2 \norm{\w^{\round} - \w_{j}^{b,n}}^2
                \right) \Big\vert \mathcal{F}_{n-1} \right], \\
    \end{align*}
    (a) follows by Law of Iterated Expectations since $\nabla{f}(\w^{\round})$ is measurable in $\mathcal{F}_{\tau_n,n}$. (b) follows since $\tilde{\g}_Q^n$ is an unbiased estimate of $\tilde{\g}^n$ by Assumption \ref{ass:quantization}. (c) follows from the definition of $\tilde{\g}^{\round}$. (d) follows from the Law of Iterated Expectations. (e) follows from the unbiased property of the stochastic gradients as stated in Assumption \ref{ass:sg_variance}. (f)
    follows from the relation $2\langle \x, \y \rangle = \norm{\x}^2 + \norm{\y}^2 -
    \norm{\x-\y}^2$. (g) follows from $L$-smoothness of Assumption \ref{ass:L-smooth} since $\g_j^{b,n} = \nabla f(\w_j^{b,n})$.
\end{proof}

The following Lemma bounds the distance between the global model at the start of a round to the local model at a local step of a client.
\begin{lemma}
Under Assumption \ref{ass:sg_variance}, FedCOM-V updates follow,
\[
\EXP\left[ \norm{\w^{\round} - \w_j^{b,\round}}^2 \Big\vert \mathcal{F}_{n-1} \right] 
        \leq 2\eta^2\tau_n \sum_{b=1}^{\tau_n} \EXP\left[ \norm{\g_j^{b,\round}}^2 \Big\vert \mathcal{F}_{n-1}  \right]  + 2\eta_\round^2\tau_n\sigma^2.
\]
\label{lm:fedcom_3}
\end{lemma}
\begin{proof}
    \begin{align*}
    \EXP\left[ \norm{\w^{\round} - \w_j^{b,\round}}^2 \Big\vert \mathcal{F}_{n-1} \right]
        &\stackrel{(a)}{=} \EXP
            \left[ \norm{\eta_n \sum_{a=1}^{b-1} \tilde{\g}_{j}^{a, \round}}^2
            \Big\vert \mathcal{F}_{n-1} \right], \\
        &\stackrel{(b)}{\leq} 2\eta_n^2 \EXP \left[ \norm{ \sum_{a=1}^{b-1} \left(\tilde{\g}_{j}^{a, \round} - \g_j^{a,\round}\right)}^2 \Big\vert \mathcal{F}_{n-1} \right] + 2\eta_n^2 \EXP \left[ \norm{ \sum_{a=1}^{b-1} \g_j^{a,\round} }^2 \Big\vert \mathcal{F}_{n-1} \right]
            , \\
        &\stackrel{(c)}{\leq} 2\eta_n^2(b-1)\sigma^2 + 2\eta_n^2 \EXP\left[ \norm{ \sum_{a=1}^{b-1} \g_j^{a,\round} }^2 \Big\vert \mathcal{F}_{n-1} \right]
            , \\
        &\stackrel{(d)}{\leq} 2\eta_n^2(b-1)\sigma^2 + 2\eta_n^2 (b-1) \sum_{a=1}^{b-1} \EXP \left[ \norm{\g_j^{a,\round} }^2 \Big\vert \mathcal{F}_{n-1} \right]
            , \\
        &\leq 2\eta_n^2\tau_n\sigma^2 + 2\eta_n^2 \tau_n \sum_{a=1}^{\tau_n} \EXP\left[ \norm{\g_j^{a,\round} }^2 \Big\vert \mathcal{F}_{n-1} \right].
    \end{align*}
    (a) follows from the local update rule of Algorithm \ref{algo:fedcom}. (b) follows from the identity, $\norm{\x}^2 \leq 2\norm{\x-\y}^2+2\norm{\y}^2$. (c) follows from a similar calculation as in the proof of Lemma \ref{lm:var_g_n}. (d) follows from
    Jensen's inequality.
\end{proof}

We prove Theorem \ref{th:fedcom_informal} in two steps. First, in Theorem \ref{th:fedcom_nonconvex_2}, we bound the convergence rate of FedCOM-V for a general choice of learning rates and local computations. Then, in Theorem \ref{th:fedcom_nonconvex_formal}, we prove a more explicit form of Theorem \ref{th:fedcom_informal} for a specific choice of learning rates and local computations.

\begin{theorem}
Under Assumptions \ref{ass:L-smooth} to \ref{ass:quantization}, if the local learning rates $\left(\eta_{\round}\right)$, local computations $\left( \tau_{\round} \right)$ and global learning rates $\left( \gamma_{\round} \right)$ satisfy,
\[
1 \geq 2\tau_{\round}^2L^2\eta_{\round}^2 + 2\left(\frac{q_{\max}}{m} + 1 \right) \eta_{\round}\gamma_{\round} L\tau_{\round}, \quad \forall \round,
\]
then, the FedCOM-V updates satisfy,
\[
\frac{\sum_{\round=1}^{r} \eta_{\round}\tau_{\round}\gamma_{\round} \EXP\left[\norm{\nabla f(\w^{\round})}^2 \Big\vert \mathcal{F}^C \right]}{\sum_{\round=0}^{r-1} \eta_{\round}\tau_{\round}\gamma_{\round}} 
\leq 
    \frac{2\left(f(\w^{(0)})-f(\w^*)\right)}{ \sum_{\round=0}^{r-1} \eta_{\round}\tau_{\round}\gamma_{\round}} 
    + \frac{2L\sigma^2}{m} \frac{\sum_{\round=0}^{r-1} \eta_{\round}^2\tau_{\round}\gamma_{\round}^2 (\bar{q}^{\round}+1)}{\sum_{\round=0}^{r-1} \eta_{\round}\tau_{\round}\gamma_{\round}} 
    + 2L^2\sigma^2\frac{\sum_{\round=0}^{r-1} \eta_{\round}^3\tau_{\round}^2\gamma_{\round}}{\sum_{\round=0}^{r-1} \eta_{\round}\tau_{\round}\gamma_{\round}}.
\]
\label{th:fedcom_nonconvex_2}
\end{theorem}
\begin{proof}
Recall the global update rule, $\w^{\round+1} = \w^{\round} - \eta_{\round}\gamma_{\round} \tilde{\g}_Q^{\round}$. From the $L$-smoothness of $f(\cdot)$, we can write,
\[
f(\w^{\round+1}) - f(\w^{\round}) 
    \leq 
    -\eta_{\round}\gamma_{\round} \langle \nabla f(\w^{\round}), \tilde{\g}_Q^{\round} \rangle 
    + \frac{\eta_{\round}^2\gamma_{\round}^2L}{2} \norm{\tilde{\g}_Q^{\round}}^2.
\]
Now, we bound the conditional expectation,
\begin{align}
    \EXP \left[ f(\w^{\round+1}) - f(\w^{\round}) \big\vert \mathcal{F}_{n-1} \right] 
    &\leq 
        -\eta_{\round}\gamma_{\round} \EXP \left[ \langle \nabla f(\w^{\round}), \tilde{\g}_Q^{\round} \rangle \Big\vert \mathcal{F}_{n-1} \right]  
        + \frac{\eta_{\round}^2\gamma_{\round}^2L}{2} \EXP \left[ \norm{\tilde{\g}_Q^{\round}}^2 \Big\vert \mathcal{F}_{n-1} \right], \nonumber \\
    &\stackrel{(a)}{\leq} \frac{\eta_{\round}\gamma_{\round}}{2m} \sum_{j=1}^m \sum_{b=1}^{\tau_n} \Big( -\norm{\nabla f(\w^{\round})}^2 - \EXP\left[\norm{\g_j^{b,\round}}^2 \Big\vert \mathcal{F}_{n-1} \right]  + L^2\EXP\left[\norm{\w^{\round} - \w_j^{b,\round}}^2 \Big\vert \mathcal{F}_{n-1} \right] \Big) \nonumber\\
   &\quad+ \frac{2\tau_{\round} L\eta_{\round}^2\gamma_{\round}^2}{2m} \left(\frac{q_\max}{m}+1 \right) \sum_{j=1}^m \sum_{b=1}^{\tau_{\round}} \EXP\left[\norm{\g_j^{b,\round}}^2 \Big\vert \mathcal{F}_{n-1} \right] + \left(\bar{q}^{\round} + 1 \right) \frac{2\tau_{\round} L\eta_{\round}^2\gamma_{\round}^2\sigma^2}{2m}, \\
   &\stackrel{(b)}{\leq} \frac{\eta_{\round}\gamma_{\round}}{2m} \sum_{j=1}^m \sum_{b=1}^{\tau_{\round}} \Big(-\norm{\nabla f(\w^{\round})}^2 - \EXP\left[\norm{\g_j^{b,\round}}^2 \Big\vert \mathcal{F}_{n-1} \right] 
   + 2L^2\eta_{\round}^2\tau_{\round}\sum_{b'=1}^{\tau_{\round}} \EXP\left[\norm{\g_j^{b',\round}}^2 \Big\vert \mathcal{F}_{n-1} \right]
   \nonumber \\  &\quad
    + 2L^2\eta_{\round}^2\tau_{\round}\sigma^2 \Big) 
   + \frac{\tau_{\round} L\eta_{\round}^2\gamma_{\round}^2}{m} \left(\frac{q_\max}{m}+1\right) \sum_{j=1}^m \sum_{b=1}^{\tau_{\round}} \EXP\left[\norm{\g_j^{b,\round}}^2 \Big\vert \mathcal{F}_{n-1}\right] + \left(\bar{q}^{\round} + 1 \right) \frac{\tau_{\round} L\eta_{\round}^2\gamma_{\round}^2\sigma^2}{m}, \\
   &\stackrel{(c)}{=} -\frac{\eta_{\round}\gamma_{\round}\tau_{\round}}{2} \norm{\nabla f(\w^{\round})}^2 + \frac{L\tau_{\round}\eta_{\round}^2\gamma_{\round}}{m}(mL\tau_{\round}\eta_{\round} + \gamma_{\round}(\bar{q}^{\round}+1))\sigma^2 \nonumber \\
   &\quad - \frac{\eta_{\round}\gamma_{\round}}{2m}\left(1 - 2L^2\eta_{\round}^2\tau_{\round}^2 -2L\tau_{\round}\eta_{\round}\gamma_{\round}\left(\frac{q_{\max}}{m} + 1 \right) \right)  \sum_{j=1}^m\sum_{b=1}^{\tau_{\round}} \EXP\left[\norm{\g_j^{b,\round}}^2 \Big\vert \mathcal{F}_{n-1} \right] , \nonumber \\
   &\stackrel{(d)}{\leq}  -\frac{\eta_{\round}\gamma_{\round}\tau_{\round}}{2} \norm{\nabla f(\w^{\round})}^2 + \frac{L\tau_{\round}\eta_{\round}^2\gamma_{\round}}{m}(mL\tau_{\round}\eta_{\round} + \gamma_{\round}(\bar{q}^{\round}+1))\sigma^2, \label{eq:fedcom_4_2}
\end{align}
where, (a) is obtained by using Lemmas \ref{lm:fedcom_1} and \ref{lm:fedcom_2}, and (b) is obtained by using Lemma \ref{lm:fedcom_3}. (c) is a rearrangement of terms, and (d) follows from the premise of the theorem,
\[
1 \geq 2\tau_{\round}^2L^2\eta_{\round}^2 + 2\left(\frac{q_{\max}}{m} + 1 \right) \eta_{\round}\gamma_{\round} L\tau_{\round}.
\]

Taking an expectation, rearranging terms and summing up equation \eqref{eq:fedcom_4_2} for all the rounds $r$, we have a telescopic cancellation to get,
\begin{align}
\frac{\sum_{\round=1}^{r} \eta_{\round}\tau_{\round}\gamma_{\round} \EXP\left[\norm{\nabla f(\w^{\round})}^2 \Big\vert \mathcal{F}^C \right]}{\sum_{\round=0}^{r-1} \eta_{\round}\tau_{\round}\gamma_{\round}} 
&=
    \frac{2\left(f(\w^{(0)})-f(\w^{r})\right)}{ \sum_{\round=0}^{r-1} \eta_{\round}\tau_{\round}\gamma_{\round}} 
    + \frac{2L\sigma^2}{m} \frac{\sum_{\round=0}^{r-1} \eta_{\round}^2\tau_{\round}\gamma_{\round}^2 (\bar{q}^{\round}+1)}{\sum_{\round=0}^{r-1} \eta_{\round}\tau_{\round}\gamma_{\round}} 
    + 2L^2\sigma^2\frac{\sum_{\round=0}^{r-1} \eta_{\round}^3\tau_{\round}^2\gamma_{\round}}{\sum_{\round=0}^{r-1} \eta_{\round}\tau_{\round}\gamma_{\round}}, \nonumber \\
&\leq 
    \frac{2\left(f(\w^{(0)})-f(\w^*)\right)}{ \sum_{\round=0}^{r-1} \eta_{\round}\tau_{\round}\gamma_{\round}} 
    + \frac{2L\sigma^2}{m} \frac{\sum_{\round=0}^{r-1} \eta_{\round}^2\tau_{\round}\gamma_{\round}^2 (\bar{q}^{\round}+1)}{\sum_{\round=0}^{r-1} \eta_{\round}\tau_{\round}\gamma_{\round}} 
    + 2L^2\sigma^2\frac{\sum_{\round=0}^{r-1} \eta_{\round}^3\tau_{\round}^2\gamma_{\round}}{\sum_{\round=0}^{r-1} \eta_{\round}\tau_{\round}\gamma_{\round}}.
\end{align}
\end{proof}

\begin{theorem}
If we choose the learning rates for round $\round$, $\eta_{\round}$ and $\gamma_{\round}$, and the number of local computations $\tau_{\round}$ as,
\[
\eta_{\round} = \frac{c_{\eta}}{L\round}, \quad \gamma_{\round} = \frac{c_{\gamma}}{\sqrt{\bar{q}^{\round}+1}}, \quad \tau_{\round} = \frac{\round}{2c_{\eta}}, 
\]
where, 
\[
c_{\eta} = 2\left(\frac{L\Delta_f \sqrt{m}}{\sigma} \left(\frac{q_{\max}}{m} +1 \right)\right)^2, \quad c_{\gamma} = \frac{1}{2\left(\frac{q_{\max}}{m}+1\right)},
\]
where, $\Delta_f \triangleq \sqrt{2(f(\w^0)-f(\w^*))/L}$, and, if FedCOM-V is run for $r$ communication rounds such that,
\begin{equation}
\frac{r}{1+\log r} \geq \max\left\{\left(\frac{q_\max}{m} + 1\right)^2\frac{4L^2\Delta_f^2m\sqrt{q_\max + 1}}{ \varepsilon}, \left(\frac{q_\max}{m} + 1 \right) \frac{12L^2\Delta_f^2 \sigma}{\varepsilon} \frac{\sum_{\round=1}^r \sqrt{\bar{q}^{\round}+1}}{r} \right\},
\label{eq:fedcom_round_bound}
\end{equation}
then we have,
\begin{equation}
\frac{\sum_{\round=1}^r \eta_{\round}\gamma_{\round}\tau_{\round} \EXP\left[\norm{f(\w^{\round})}^2 \Big\vert \mathcal{F}^C \right]}{\sum_{\round=1}^r \eta_{\round}\tau_{\round}\gamma_{\round}} \leq \varepsilon.
\label{eq:th_fedcom_result}
\end{equation}
\label{th:fedcom_nonconvex_formal}
\end{theorem}
\begin{proof}[Proof of Theorem \ref{th:fedcom_nonconvex_formal}]
This proof will use the result of Theorem \ref{th:fedcom_nonconvex_2}. Therefore, we first check that the premise of Theorem \ref{th:fedcom_nonconvex_2} is satisfied. Due to the choice of $\eta_{\round}$, $\tau_{\round}$ and $\gamma_{\round}$,
\begin{align*}
    2\tau_{\round}^2L^2\eta_{\round}^2 + 2\left(\frac{q_{\max}}{m} + 1 \right) \eta_{\round}\gamma_{\round} L\tau_{\round} 
    &= 1/2 + \left(\frac{q_{\max}}{m}+1\right)c_{\gamma}, \\
    &= 1.
\end{align*}

Therefore, the following result of Theorem \ref{th:fedcom_nonconvex_2} holds true.
\[
\frac{\sum_{\round=1}^r \eta_{\round}\tau_{\round}\gamma_{\round} \EXP\left[\norm{\nabla f(\w^{\round})}^2 \Big\vert \mathcal{F}^C \right]}{\sum_{\round=1}^r \eta_{\round}\tau_{\round}\gamma_{\round}} 
    \leq 
    \underbrace{\frac{2\left(f(\w^{0})-f(\w^*)\right)}{ \sum_{\round=1}^r \eta_{\round}\tau_{\round}\gamma_{\round}} + \frac{2L\sigma^2}{m} \frac{\sum_{\round=1}^r \eta_{\round}^2\tau_{\round}\gamma_{\round}^2 (\bar{q}^{\round}+1)}{\sum_{\round=1}^r \eta_{\round}\tau_{\round}\gamma_{\round}}}_{\text{\rom{1}}} + \underbrace{2L^2\sigma^2\frac{\sum_{\round=1}^r \eta_{\round}^3\tau_{\round}^2\gamma_{\round}}{\sum_{\round=1}^r \eta_{\round}\tau_{\round}\gamma_{\round}}}_{\text{\rom{2}}}.
\]

Consider \rom{1},
\begin{align}
    \text{\rom{1}} &\stackrel{(a)}{=} \frac{4L(f(\w^{(0)})-f(\w^*))}{ c_\gamma \sum_{\round=1}^r 1/\sqrt{\bar{q}^{\round}+1}} + \frac{2\sigma^2 c_\eta c_\gamma}{m} \frac{\sum_{\round = 1}^r 1/\round}{\sum_{\round=1}^r 1/\sqrt{\bar{q}^{\round}+1}}, \nonumber \\
    &\stackrel{(b)}{\leq}   \frac{4L(f(\w^{(0)})-f(\w^*))}{ c_\gamma \sum_{\round=1}^r 1/\sqrt{\bar{q}^{\round}+1}} + \frac{2\sigma^2 c_\eta c_\gamma}{m} \frac{1+\log r}{\sum_{\round=1}^r 1/\sqrt{\bar{q}^{\round}+1}}, \nonumber \\
    &\stackrel{(c)}{\leq} \frac{2(1+\log r)}{\sum_{\round=1}^r 1/\sqrt{\bar{q}^{\round}+1}} \left( \frac{L^2 \Delta_f^2}{c_\gamma} + \frac{\sigma^2 c_\eta c_\gamma}{m} \right), \nonumber \\
     &\stackrel{(d)}{=} \frac{1+\log r}{\sum_{\round=1}^r 1/\sqrt{\bar{q}^{\round}+1}}6L^2\Delta_f^2 \left(\frac{q_\max}{m} + 1 \right), \nonumber \\
     &\stackrel{(e)}{\leq} 6L^2\Delta_f^2 \left(\frac{q_\max}{m} + 1 \right) \frac{1+\log r}{r}  \frac{\sum_{\round=1}^r \sqrt{\bar{q}^{\round}+1}}{r}.
\end{align}
(a) is obtained by substituting the expressions for $\eta_n$, $\gamma_n$ and $\tau_n$. (b) is obtained by the bound, $\sum_{n=1}^r 1/n \leq 1+\log r$. (c) is a rearrangement of terms and the bound $1 \leq 1 + \log r$. (d) is obtained by substituting the expressions for $c_\eta$ and $c_\gamma$. (e) is obtained using the result that the harmonic mean is smaller than the arithmetic mean.

Consider \rom{2},
\begin{align}
    \text{\rom{2}}
    &= 2L^2\sigma^2\frac{\sum_{\round=1}^r \eta_{\round}^3\tau_{\round}^2\gamma_{\round}}{\sum_{\round=1}^{r} \eta_{\round}\tau_{\round}\gamma_{\round}}, \nonumber \\
    &= \sigma^2 c_\eta \frac{\sum_{\round=1}^{r} 1/(\round\sqrt{\bar{q}^{\round}+1})}{\sum_{\round=1}^r,  1/\sqrt{\bar{q}^{\round}+1}}, \nonumber \\
    &\leq \sigma^2 c_\eta \sqrt{q_\max+1} \frac{\sum_{\round=1}^r 1/\round}{r}, \nonumber \\
    &\leq \sigma^2 c_\eta \sqrt{q_\max+1} \frac{1+\log r}{r}, \nonumber \\
    &= 2\left(\frac{q_\max}{m} + 1\right)^2 L^2\Delta_f^2m\sqrt{q_\max + 1} \frac{1+\log r}{r}.
\end{align}

Therefore, if we chose the total number of communication rounds $r$ such that,
\begin{equation}
6L^2\Delta_f^2 \left(\frac{q_\max}{m} + 1 \right) \frac{1+\log r}{r}  \frac{\sum_{\round=1}^r \sqrt{\bar{q}^{\round}+1}}{r} \leq \frac{\varepsilon}{2},
\label{eq:fedcom_round_bound_1}
\end{equation}
and,
\begin{equation}
2\left(\frac{q_\max}{m} + 1\right)^2 L^2\Delta_f^2m\sqrt{q_\max + 1} \frac{1+\log r}{r} \leq \frac{\varepsilon}{2},
\label{eq:fedcom_round_bound_2}
\end{equation}
then \eqref{eq:th_fedcom_result} is satisfied. \eqref{eq:fedcom_round_bound_1} and \eqref{eq:fedcom_round_bound_2} are simply a restatement of \eqref{eq:fedcom_round_bound}. This completes the proof.
\end{proof}

If the compression parameters $\left( \Q^n \right)_n$ formed a stationary process with a stationary distribution according to a random variable $\Q$, then $\sum_{n=1}^r \sqrt{\bar{Q}^n+1}/r \to \EXP[\sqrt{\bar{Q}+1}]$. Therefore, Theorem \ref{th:fedcom_nonconvex_formal} proves Theorem \ref{th:fedcom_informal}.

\end{appendices}
\end{document}